\documentclass[12pts]{article}
\usepackage[utf8]{inputenc}
\usepackage[round]{natbib}
\usepackage[margin=1in]{geometry}
\usepackage{graphicx} 
\usepackage{authblk}
\title{Generalization Properties of Score-matching Diffusion Models for Intrinsically Low-dimensional Data}
\author[1]{Saptarshi Chakraborty\thanks{email: \texttt{saptarsc@umich.edu}}}
\author[4]{Quentin Berthet \thanks{email: \texttt{qerthet@google.com}}}
 \author[2,3,4]{Peter L.~Bartlett\thanks{email: \texttt{peter@berkeley.edu}}}
 \affil[1]{Department of Statistics, University of Michigan}
 \affil[2]{Department of Statistics, University of California, Berkeley}
  \affil[3]{Department of Electrical Engineering and Computer Sciences, UC Berkeley}
 \affil[4]{Google DeepMind}
\date{\vspace{-5ex}}
\date{\today}
\usepackage{titletoc}
\usepackage{amsfonts}
\usepackage{amsmath}
\usepackage{amsthm}
\usepackage{amssymb}
\usepackage{bbm}
\usepackage{enumerate}
\usepackage{thm-restate}
\usepackage{graphicx} 
\usepackage{caption} 
\usepackage{subcaption} 
\newtheorem{defn}{Definition}
\newtheorem{theorem}[defn]{Theorem}
\newtheorem{lemma}[defn]{Lemma}
\newtheorem{cor}[defn]{Corollary}
\newtheorem{ass}{Assumption}

\newtheorem{propos}[defn]{Proposition}
\usepackage[mathscr]{euscript}
\usepackage[italicdiff]{physics}
\usepackage{natbib}
\usepackage{xcolor}
\usepackage{hyperref}
\hypersetup{
    colorlinks=true,
    citecolor = blue,
    linkcolor=blue,
    filecolor=magenta,      
    urlcolor=teal,
    pdftitle={Overleaf Example},
    pdfpagemode=FullScreen,
    }

\def\Real{\mathop{\mathbb{R}}\nolimits}

\def\argmin{\mathop{\rm argmin}}

\newcommand{\diff}{\mathrm{d}}
\newcommand{\prob}{\mathbb{P}}
\newcommand{\one}{\mathbbm{1}}

\newcommand{\bs}{\boldsymbol{s}}

\newcommand{\cB}{ \mathcal{B}}
\newcommand{\cC}{ \mathcal{C}}

\newcommand{\cE}{ \mathcal{E}}
\newcommand{\cF}{ \mathcal{F}}

\newcommand{\cL}{ \mathcal{L}}
\newcommand{\cM}{ \mathcal{M}}
\newcommand{\cN}{ \mathcal{N}}
\newcommand{\cO}{ \mathcal{O}}

\newcommand{\cR}{ \mathcal{R}}
\newcommand{\cS}{ \mathcal{S}}

\newcommand{\cW}{ \mathcal{W}}

\newcommand{\sA}{ \mathscr{A}}
\newcommand{\sB}{ \mathscr{B}}
\newcommand{\sC}{ \mathscr{C}}

\newcommand{\sH}{ \mathscr{H}}

\newcommand{\sM}{ \mathscr{M}}
\newcommand{\sN}{ \mathscr{N}}

\newcommand{\sQ}{ \mathscr{Q}}

\newcommand{\sS}{ \mathscr{S}}
\newcommand{\sT}{ \mathscr{T}}

\newcommand{\sX}{ \mathscr{X}}

\newcommand{\hP}{ \hat{P}}
\newcommand{\hQ}{ \hat{Q}}

\newcommand{\hX}{ \widehat{X}}
\newcommand{\hY}{ \widehat{Y}}

\newcommand{\hmu}{\hat{\mu}}
\newcommand{\E}{\mathbb{E}}

\newcommand{\relu}{\text{ReLU}}

\newcommand{\fH}{\mathbb{H}}

\newcommand{\fL}{\mathbb{L}}
\newcommand{\fM}{\mathbb{M}}
\newcommand{\fN}{\mathbb{N}}

\newcommand{\fW}{\mathbb{W}}

\newcommand{\hx}{\hat{x}}
\newcommand{\hp}{\hat{p}}

\newcommand{\tv}{\operatorname{TV}}
\newcommand{\kl}{\operatorname{KL}}

\date{\vspace{-5ex}}

\begin{document}
\maketitle
\begin{abstract}
Despite the remarkable empirical success of score-based diffusion models, their theoretical guarantees for statistical accuracy remain relatively underdeveloped. Existing analyses often yield pessimistic convergence rates that fail to reflect the intrinsic low-dimensional structure commonly present in real-world data distributions, such as those arising in natural images. In this work, we analyze the statistical convergence of score-based diffusion models for learning an unknown data distribution $\mu$ from finitely many samples. Under mild regularity conditions on both the forward diffusion process and the data distribution, we establish finite-sample error bounds on the learned generative distribution measured in the Wasserstein-$p$ distance. In contrast to prior results, our guarantees hold for arbitrary $p \ge 1$ and require only a finite-moment condition on $\mu$, without compact-support, manifold or smooth density assumptions. To characterize this intrinsic dimension, we introduce the concept of the  $(p,q)$-Wasserstein dimension, which generalizes the classical Wasserstein dimension to distributions with bounded support as well as to heavy-tailed distributions.  In particular, we show that with $n$ independent and identically distributed (i.i.d.) samples from the target distribution $\mu$ with finite $q$-th moment (i.e., $\E_{X \sim \mu}\|X\|^q < \infty$) and appropriately chosen network architectures, hyper-parameters and discretization scheme, the expected Wasserstein-$p$ distance between the learned distribution $\hat{\mu}^{\text{SGM}}$ and true distribution $\mu$ scales roughly as $\E \fW_p \left(\hat{\mu}^{\text{SGM}}, \mu\right) = \widetilde{\mathcal{O}}\left(n^{-1/d^\star_{p,q}(\mu)}\right)$, where $d^\star_{p,q}(\mu)$ denotes the $(p, q)$-Wasserstein dimension of $\mu$. Our analyses demonstrate that diffusion models naturally adapt to the intrinsic geometry of the data and effectively mitigate the curse of dimensionality, in the sense that the convergence exponent depends only on $d^\star_{p,q}(\mu)$ rather than the ambient dimension. Furthermore, our results conceptually bridge the theoretical understanding of diffusion models with that of GANs and the sharp minimax rates established in optimal transport theory.
\end{abstract}
\section{Introduction}
Score-based diffusion models have emerged as a central paradigm in generative modeling \citep{pmlr-v37-sohl-dickstein15,song2019generative, ho2020denoising,songscore}, achieving remarkable success across domains such as image and text generation \citep{dhariwal2021diffusion,austin2021structured,10.5555/3600270.3602913}, text-to-speech synthesis \citep{pmlr-v139-popov21a}, molecular structure modeling \citep{xu2022geodiff,trippe2023diffusion,watson2023novo} and many more. At their core, Denoising Diffusion Probabilistic Models (DDPMs) \citep{ho2020denoising} rely on a two-stage process. In the forward phase, data samples are progressively corrupted by the iterative addition of Gaussian noise, eventually transforming the data distribution into an isotropic Gaussian. This forward noising process is chosen to be simple, tractable, and analytically well-characterized, often modeled through a stochastic differential equation (SDE). The reverse phase then seeks to invert this corruption by learning a sequence of denoising transformations that iteratively recover clean data samples from noise. In practice, the reverse process is parameterized using deep neural networks that approximate the  score function of the intermediate noisy distributions. By chaining together these learned denoising steps, diffusion models are able to generate highly realistic samples that closely approximate the underlying data distribution.

The empirical success of diffusion models has sparked considerable interest in their theoretical foundations. For instance, \citet{chen2023sampling} and \citet{lee2023convergence} analyzed the convergence of the denoising process in terms of Total Variation (TV), assuming access to an $\epsilon$-accurate (in the $\ell_2$ sense) score function. Building on this line of work, recent studies \citep{chen2023improved,li2023towards,benton2024nearly} have examined the ``optimal"  time to run the forward process and developed appropriate partitioning schemes for the denoising procedure under the same assumption of an $\epsilon$-accurate score estimate. Their results show that with an $\epsilon$-approximate score estimate, one can achieve a guarantee of order $O(\operatorname{poly}(D, \log(1/\delta))\epsilon^{-2})$-approximation of the target distribution in terms of the Kullback Leibler (KL) divergence. Here $D$ is the ambient dimension of the data space and $\delta$ is the early stopping time for the reverse process.

There has been a growing body of work in understanding the performance of diffusion models from a learning theory perspective. \citet{oko2023diffusion} demonstrate that under certain smoothness conditions on the true density function, the resulting estimated data distribution achieves an (almost) minimax optimal convergence rate in both total variation and Wasserstein-1 distances, resulting in an error rate of  $\E \fW_1\left(\hat{\mu}^{\text{SGM}}, \mu\right) = \cO\left(n^{-\frac{s+1-\delta}{D+2s}}\right)$ error rate, when $\mu$ is $s$-smooth in the Besov sense (i.e., $\mu \in \mathcal{B}^s_{p,q}$).  In a related work, when $D=1$, i.e., when the data support is $[-1,1]$, \citet{dou2024optimalscorematchingoptimal} derived the minimax estimation rate for the TV and Wasserstein-1 metric and showed that score matching can effectively achieve this rate when the score function is modeled using kernel regression based estimators with appropriately chosen bandwidth. 

Although significant progress has been made in our theoretical understanding of score-based diffusion models, some limitations of the existing results are yet to be addressed. For instance, the generalization bounds frequently suffer from the curse of dimensionality. In practical applications, data distributions tend to have high dimensionality, making the convergence rates that have been proven exceedingly slow. However, high-dimensional data, such as images, texts, and natural languages, often possess latent low-dimensional structures that reduce the complexity of the problem. For example, it is hypothesized that natural images lie on a low-dimensional structure, despite its high-dimensional pixel-wise representation \citep{pope2020intrinsic}. Though in classical statistics there have been various approaches, especially using kernel tricks and Gaussian process regression that achieve a fast rate of convergence that depends only on their low intrinsic dimensionality \citep{bickel2007local,kim2019uniform}, such results are largely unexplored in the context of diffusion models. \citet{chen2023score} derive explicit convergence rates for specific score estimation methods when the data distribution lies on a low-dimensional hyperplane within the ambient space under the Wasserstein-$1$ distance.  Recently, \citet{pmlr-v238-tang24a} showed that when the target data distribution lies on a $d$-dimensional differentiable sub-manifold and has a density with respect to the volume measure of that manifold, diffusion models can achieve an error rate of the form,  $\E \fW_1\left(\hat{\mu}^{\text{SGM}}, \mu\right) = \cO \left(n^{-\frac{A}{B + d}}\right)$, when the score function class is chosen properly. \citet{debortoli2022convergence} analyzed the excess risk in the 1-Wasserstein distance under an $\ell_2$ error assumption on the score estimator. 
Recent advances by \citet{huang2024denoising}, \citet{potaptchik2024linear}, and \citet{Kadkhodaie26a} have established the proximity of estimated measures to target measures in the $\operatorname{KL}$ sense, particularly when the target $\mu$ exhibits manifold or low-dimensional Minkowski support. However, all of the aforementioned works rely on the idealized assumption that the estimated score function is $\epsilon$-close (in the $\ell_2$ sense) to the true score function, which cannot be guaranteed in practice. Furthermore, the analysis by \citet{Kadkhodaie26a} tends to overestimate the intrinsic dimension by utilizing the global covering number of the support while neglecting the impact of the discretization error.

 It is important to note that all of the aforementioned works fail to fully capture the intrinsic low-dimensional structure of the underlying data distribution (see Section~\ref{intrinsic}). In particular, the assumption in \citet{pmlr-v238-tang24a} that the support of the target distribution lies on a compact Riemannian manifold with a smooth, bounded density is rather restrictive and often unrealistic in practice. Similarly, the subspace-support assumption adopted in \citet{chen2023score} provides only a crude approximation of the underlying geometry and overlooks more general forms of low-dimensional structure observed in real-world data. Furthermore, none of the aforementioned approaches address the problem in its full generality or attain the sharp convergence rates for empirical distributions established in the optimal transport literature \citep{weed2019sharp}.

\subsection{Contributions} To address the aforementioned limitations in the existing literature, the main contributions of this paper are summarized below.
\begin{itemize}
\item In order to bridge the gap between the theory and practice of score-based diffusion generative models, in this paper, we develop a framework to establish the statistical convergence rates in the Wasserstein-$p$ metric in terms of the intrinsic dimension of the underlying target probability measure. 
    \item To formalize the notion of intrinsic dimension,  we introduce the $(p,q)$-Wasserstein dimension (see Definition~\ref{new_dim}) that develops the notion of Wasserstein dimension \citep{weed2019sharp} to distributions with an unbounded support but satisfying a finite moment condition. This $(p,q)$-Wasserstein dimension plays an important role in (upper) bounding the convergence rate of $\hmu_n$ (the empirical distribution based on $n$ i.i.d. data samples) and the true probability measure $\mu$ (See Theorem~\ref{moment_thm}). 
    \item Our results, in essence, suggest that when the score network class, the number of Monte Carlo samples used during training, and the discretization scheme are all chosen appropriately (see Theorem~\ref{mainthm}), diffusion-based generative models can achieve near-optimal statistical accuracy. Specifically, if the model is trained on $n$ independent and identically distributed (i.i.d.) samples drawn from the target distribution $\mu$, the expected error of the learned distribution satisfies $\E \fW_p\left(\hmu^{\text{SGM}}, \mu\right) \lesssim n^{-1/d^\star_{p,q}(\mu)} \operatorname{poly-log}(n)$. 
    where $d^\star_{p,q}(\mu)$ denotes the intrinsic $(p, q)$-Wasserstein dimension of $\mu$. This result highlights that, under mild regularity assumptions, diffusion models can adapt to the low-dimensional geometry of the data and achieve convergence rates that scale only with the intrinsic rather than the ambient dimension. Importantly, our analysis yields the sharpest known error bound for diffusion models to date. Notably, our results yield sharper rates for the special case of manifolds and $p=1$ as considered in the recent literature \citep{pmlr-v238-tang24a, oko2023diffusion} even under significantly milder regularity conditions.
    \item When the underlying target measure $\mu$ has a ``regular" support (this includes compact differentiable manifolds, affine subspaces etc.), our results indicate that deep score-based diffusion models can effectively achieve the minimax optimal error rates, albeit with poly-log factors in the number of samples $n$.
\end{itemize}
\subsection{Organization} The remainder of this paper is organized as follows: In Section~\ref{proof_of_concept}, we empirically validate that the sample efficiency of score-matching diffusion models is primarily contingent upon the intrinsic data dimension.  Section~\ref{background} revisits necessary notation and definitions and outlines the problem statement. In Section~\ref{intrinsic}, we revisit the concept of intrinsic dimension and introduce a new notion of intrinsic dimension termed the $(p,q)$-Wasserstein dimension of a measure, comparing it with commonly used metrics. This $(p,q)$-Wasserstein dimension determines the convergence rate of the empirical measure to the population in the Wasserstein-$p$ distance under finite moment conditions only, as shown in Theorem~\ref{moment_thm}. The subsequent focus shifts to theoretical analyses of score-based deep diffusion models in Section~\ref{theo_ana}. We begin by presenting the assumptions in Section~\ref{sec_assumptions}, then stating the main result in Section~\ref{main results} and providing a proof sketch in Section~\ref{pf_main_results}, with detailed proofs available in the appendices. Section~\ref{minimax bounds} demonstrates that diffusion can achieve the minimax optimal rates for estimating distributions, followed by concluding remarks in Section~\ref{conclusions}.

\section{A Proof of Concept Result}\label{proof_of_concept}
\begin{figure}[h!]
    \centering
    \begin{subfigure}[t]{0.3\textwidth}
        \includegraphics[width=\textwidth]{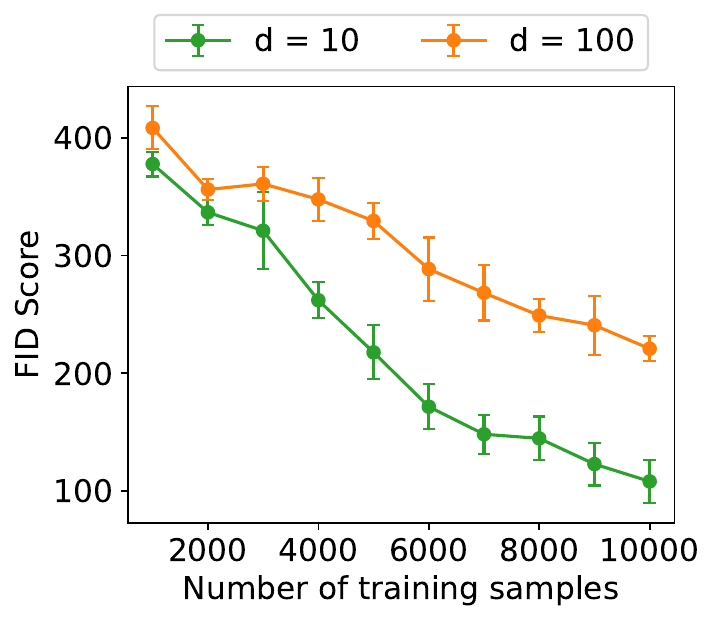}
        \caption{Results on the basenji class}
        \label{fig:sub1}
    \end{subfigure}
    \hfill 
    \begin{subfigure}[t]{0.3\textwidth}
        \includegraphics[width=\textwidth]{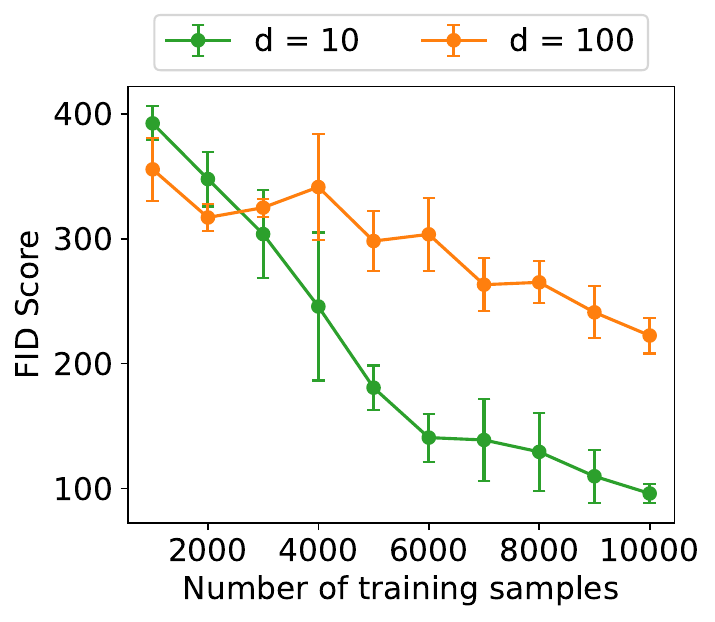}
        \caption{Results on the soap-bubble class}
        \label{fig:sub2}
    \end{subfigure}
    \hfill 
    \begin{subfigure}[t]{0.3\textwidth}
        \includegraphics[width=\textwidth]{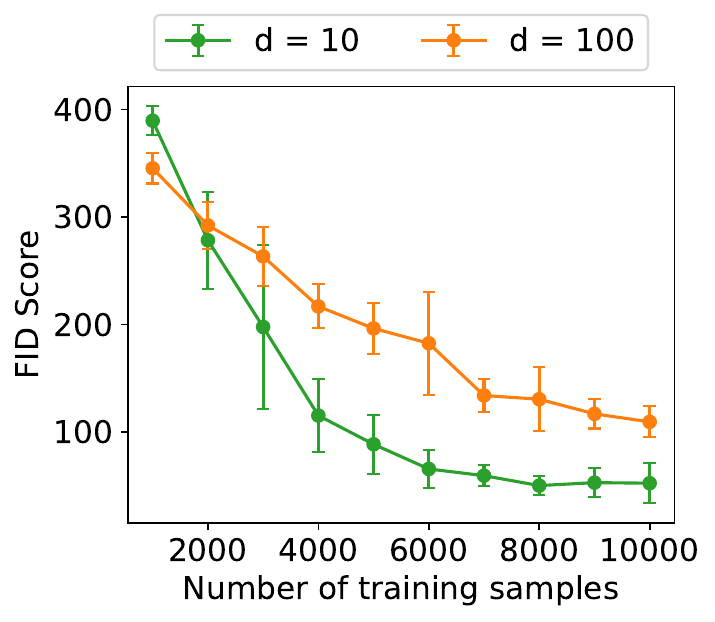}
        \caption{Results on the goldfish class}
        \label{fig:sub3}
    \end{subfigure}
    \caption{Average generalization error (in terms of FID scores) for different values of $n$ for DDPM. The error bars denote the standard deviation out of $10$ replications.}
    \label{fig:main}
\end{figure}
Before turning to the theoretical analysis, we present an experiment illustrating that 
the error rates of denoising diffusion probabilistic models (DDPMs) \citep{ho2020denoising} 
depend primarily on the \emph{intrinsic} dimension of the data. Directly estimating the 
intrinsic dimensionality of natural images is difficult, so we follow the synthetic data 
construction strategy of \citet{pope2020intrinsic} and \citet{chakraborty2024a}. Specifically, we use a pre-trained 
BigGAN \citep{brock2018large} with a 128-dimensional latent space and 
$128 \times 128 \times 3$ outputs, trained on the ImageNet dataset \citep{deng_imagenet}. We conduct three separate experiments, on the \texttt{basenji}, \texttt{soap-bubble} and \texttt{goldfish} image classes in the ImageNet data set. 
Using the BigGan generator, we produce $11{,}000$ images  by 
fixing all but $d$ latent coordinates to zero, thereby constraining the generated data 
to lie on a $d$-dimensional manifold. We consider two intrinsic dimensions, 
$d = 10$ and $d = 100$. All images are then downsampled to $28 \times 28$ using 
bilinear interpolation to reduce computational cost.

We train a DDPM with a standard UNet, closely following the architecture of 
\citet{ho2020denoising}. The model consists of a 3-channel input layer, four resolution 
levels, sinusoidal timestep embeddings, and residual blocks with attention at the 
$14\times14$ scale. We use a linear $\beta$-schedule with $1000$ diffusion timesteps. 
At each iteration, the model receives a clean training image $x_0$, a random timestep 
$t \sim \mathrm{Uniform}\{0,\dots,999\}$, and Gaussian noise $\varepsilon$, and is trained 
to predict the noise component using the standard denoising objective 
$\mathbb{E}\| \varepsilon - \varepsilon_\theta(x_t,t)\|_2^2$. For each intrinsic dimension, we vary the number of training samples in 
$\{1000, 2000, \ldots, 10000\}$ and reserve the final $1000$ images for testing. 
To ensure stable optimization across all sample sizes, the UNet is trained for 
$50$ epochs using Adam with learning rate $10^{-4}$ and batch size $64$. 
All experiments are run on a single NVIDIA GPU. To prevent memory overflow, generated 
samples during evaluation are produced in batches of size $64$, and gradients are 
disabled during the sampling stage.
After training, we generate $256$ samples from the learned reverse diffusion process 
using ancestral (DDPM) sampling, starting from standard Gaussian noise and iteratively 
denoising for all $1000$ timesteps. We compute the Fréchet Inception Distance (FID) 
\citep{heusel2017gans} between the generated samples and an equally sized subset of 
held-out real images. Since FID is sensitive to sampling noise, each configuration is 
repeated 10 times; we report average FID values across these repetitions. All FID 
calculations are performed using the \texttt{pytorch\_fid\_wrapper} library.

The empirical results, presented in Fig.~\ref{fig:main}, clearly validate our 
hypothesis. As the number of training samples increases from $1000$ to $10000$, the 
DDPM trained on data with intrinsic dimension $d = 10$ consistently achieves 
substantially lower FID scores than the model trained on data with intrinsic dimension 
$d = 100$. The gap persists across large sample sizes, and in fact widens slightly as 
more data become available, indicating that lower intrinsic-dimensional structure 
yields faster error decay under diffusion-based learning. This behavior aligns precisely 
with our theoretical predictions: the intrinsic dimension, rather than the ambient pixel 
dimension, governs the sample complexity and achievable error rates of diffusion models. All code pertaining to the experiment is publicly available through \url{https://github.com/saptarshic27/DiffusionIntrinsic}.
\section{Background}\label{background}
Before presenting the main theoretical results, we first introduce the notation and review several preliminary concepts that will be used throughout the paper.

\subsection{Notation}
We use the notation $x \vee y : = \max\{x,y\}$ and $x \wedge y : = \min\{x,y\}$.  
$B_\varrho(x,r)$ denotes the open ball of radius $r$ around $x$, with respect to (w.r.t.) the metric $\varrho$. 
For any probability measure $\gamma$, the support of $\gamma$ is defined as $\text{supp}(\gamma) = \{x: \gamma(B_\varrho(x,r))>0, \text{ for all }r >0\}$.  For any function $f: \cS \to \Real$, and any probability measure $\gamma$ on $\cS$, let $\|f\|_{\fL_p(\gamma)} : = \left(\int_\cS |f(x)|^p d \gamma(x) \right)^{1/p}$, if $0<p< \infty$. Also let, $\|f\|_{\fL_\infty(\gamma)} : = \operatorname{ess\, sup}_{x \in \text{supp}(\gamma)}|f(x)|$. 
We say $A_n \lesssim B_n$ (also written as $A_n = \cO(B_n)$) if there exists $C>0$, independent of $n$, such that $A_n \le C B_n$. Similarly, the notation, ``$\precsim$" (also written as $A_n = \widetilde{\cO}(B_n)$) ignores poly-log factors in  $n$, i.e., $A_n \le B_n \cdot \log^C(en)$, for some $C>0$. We say $A_n \asymp B_n$, if $A_n \lesssim B_n$ and $B_n \lesssim A_n$. For any $k \in \mathbb{N}$, we let $[k] = \{1, \dots, k\}$. For two random variables $X$ and $Y$, we say that $X  \overset{d}{=} Y$, if the random variables have the same distribution. We use bold lowercase letters to denote members of $\mathbb{N}^k$ for $k \in \mathbb{N}$. 
Suppose $Z$ is a random variable with law $\nu$. Then the distribution of $Z\one\{\|Z\|_\infty \le R\}$ is denoted as $\sT_R(\nu)$. $\cN(\theta, \Sigma)$ denotes the Gaussian distribution with mean $\theta$ and covariance matrix $\Sigma$. $\gamma_D$ denotes the standard Gaussian distribution on $\Real^D$, i.e., $\gamma_D \equiv \cN(0, I_D)$. $\sC_1$ denotes the set of all continuously differentiable functions from $\Real$ to $\Real$. 
\begin{defn}[Covering and Packing Numbers] 
    \normalfont 
    For a metric space $(\cS,\varrho)$, the $\epsilon$-covering number w.r.t. $\varrho$ is defined as:
    \(\cN(\epsilon; \cS, \varrho) = \inf\{n \in \mathbb{N}: \exists \, x_1, \dots x_n \text{ such that } \cup_{i=1}^nB_\varrho(x_i, \epsilon) \supseteq \cS\}.\) A minimal $\epsilon$ cover of $\cS$ is denoted as $\cC(\epsilon; \cS, \varrho)$.
    Similarly, the $\epsilon$-packing number is defined as:
    \(\cM(\epsilon; \cS, \varrho) = \sup\{m \in \mathbb{N}: \exists \, x_1, \dots x_m \in \cS \text{ such that } \varrho(x_i, x_j) \ge \epsilon, \text{ for all } i \neq j\}.\)
\end{defn}
\begin{defn}[Neural networks]\normalfont
 Let $L \in \mathbb{N}$ and $ \{N_i\}_{i \in [L]} \in \mathbb{N}$. Then a $L$-layer neural network $f: \Real^d \to \Real^{N_L}$ is defined as,
\begin{equation}
\label{ee1}
f(x) = A_L \circ \sigma_{L-1} \circ A_{L-1} \circ \dots \circ \sigma_1 \circ A_1 (x)    
\end{equation}
Here, $A_i(y) = W_i y + b_i$, with $W_i \in \Real^{N_{i} \times N_{i-1}}$ and $b_i \in \Real^{N_{i-1}}$, with $N_0 = d$. Note that $\sigma_j$ is applied component-wise.  Here, $\{W_i\}_{1 \le i \le L}$ are known as weights, and $\{b_i\}_{1 \le i \le L}$ are known as biases. $\{\sigma_i\}_{1 \le i \le L-1}$ are known as the activation functions. Without loss of generality, one can take $\sigma_\ell(0) = 0, \, \forall \, \ell \in [L-1]$. We define the following quantities:  
(Depth) $\cL(f) : = L$ is known as the depth of the network; (Number of weights) the number of weights of the network $f$ is denoted as  $\cW(f) = \sum_{i=1}^L N_i N_{i-1}$; 
(maximum weight) $\cB(f) = \max_{1 \le j \le L} (\|b_j\|_\infty) \vee \|W_j\|_{\infty}$ to denote the maximum absolute value of the weights and biases.
\begin{align*}
    \cN \cN_{\{\sigma_i\}_{1 \le i \le L-1}} (L, W, B) = \big\{  f \text{ of the form \eqref{ee1}}: \cL(f) \le L , \, \cW(f) \le W, \cB(f) \le B \big\}.
\end{align*}
 If $\sigma_j(x) = x \vee 0$, for all $j=1,\dots, L-1$, we use the notation $\cR \cN (L, W, B)$ to denote $\cN \cN_{\{\sigma_i\}_{1 \le i \le L-1}} (L, W, B)$.
 \end{defn}
\begin{defn}[Sobolev functions]\normalfont
Let $f: \mathcal{S} \to \Real^{d^\prime}$ be a function, where $\mathcal{S} \subseteq \Real^d$. For a multi-index $\bs = (s_1,\dots,s_d)$, let, $\partial^{\bs} f = \frac{\partial^{|\bs|} f}{\partial x_1^{s_1} \ldots \partial x_d^{s_d}}$, where, $|\bs| = \sum_{\ell = 1}^d s_\ell $. We say that a function $f$ is $(\beta,C)$-Sobolev (for $\beta \in \mathbb{N} \cup\{0\}$) if
\[ \|f\|_{\sH^\beta}: = \sum_{j=1}^{d^\prime}\sum_{\bs: 0 \le |\bs| \le \beta} \|\partial^{\bs} f_j\|_\infty \le C, \]
where $f_j$ denotes the $j$-th component of $f$, $j = 1, \ldots, d^\prime$.
\end{defn}
 \begin{defn}[Wasserstein $p$-distance] \normalfont
    Let $(\Omega,\text{dist})$ be a Polish space and let $\mu$ and $\nu$ be two probability measures on the same with finite $p$-moments. Then the $p$-Wasserstein distance between $\mu$ and $\nu$ is defined as:
    \[\fW_p(\mu, \nu) = \left(\inf_{\gamma \in \Gamma(\mu, \nu)} \E_{(X,Y) \sim \gamma} \left(\operatorname{dist}(X,Y)\right)^{p}\right)^{1/p}.\]
    Here $\Gamma(\mu, \nu)$ denotes the set of all measure couples between $\mu$ and $\nu$. In what follows, we take $\text{dist}(\cdot, \cdot)$ to be the $\ell_2$-norm on $\Real^D$, i.e., for $(\Real^D, \ell_2)$,
    \[\fW_p(\mu, \nu) = \left(\inf_{\gamma \in \Gamma(\mu, \nu)} \E_{(X,Y) \sim \gamma} \|X-Y\|_2^{p}\right)^{1/p}.\]
    Throughout this paper, we will assume that $p \ge 1$.
\end{defn}
\subsection{Score Matching Diffusion Models}
\label{sec_diff_intro}
The objective in generative modeling is to approximate an unknown distribution $\mu$ on a data space $\sX$ from $n$ independent and identically distributed (i.i.d.) samples ${X_1, \dots, X_n} \sim \mu$. In most practical settings, the data space $\sX$ is assumed to be a subset of $\Real^D$, where $D$ can be very large. For example, the ImageNet image data set consists of $128 \times 128$ colored images, making $D= 49152$. Classical approaches such as Variational Autoencoders (VAEs) \citep{kingma2014autoencoding} or Generative Adversarial Networks (GANs) \citep{goodfellow2020generative} seek to learn a generative map from a simple latent distribution (e.g., standard Gaussian) to $\sX$ whose pushforward distribution approximates $\mu$.

Score-based diffusion models \citep{ho2020denoising,songscore,song2019generative} adopt a different perspective: instead of directly learning a generative map, they construct a sequence of intermediate distributions that gradually transform the unknown data distribution into a tractable reference distribution, from which it is easy to sample, e.g., the standard Gaussian distribution. This is accomplished by adding Gaussian noise to the data in small increments, thereby diffusing the empirical distribution until, in the limit, it converges to a standard Gaussian distribution on $\Real^D$. The generative process is then learned in reverse: one trains a family of score functions (or equivalently, denoisers) that estimate the gradients of the log-density at each noise level. By combining these learned scores backward along the diffusion trajectory, one can iteratively remove noise from the reference distribution, eventually producing high-quality samples that approximate $\mu$. 

\paragraph{Forward Process} In mathematical terms, the forward process is modeled by a Stochastic Differential Equation (SDE). The simplest amongst these processes is the Ornstein-Uhlenbeck (OU) process \citep{oksendal2003stochastic}. We consider the forward time-rescaled OU process: 
\begin{align}
    \diff  \hX_t = -\beta_t \hX_t \diff  t + \sqrt{2 \beta_t} \diff  W_t, \, \hX_0 \sim \hmu_n, \label{forward}
\end{align}
where $\hmu_n = \frac{1}{n}\sum_{i=1}^n \delta_{X_i}$ denotes the empirical distribution based on the i.i.d. samples $\{X_i\}_{i \in [n]}$. Here $\{W_t\}_{t \ge 0}$ denotes the standard Brownian motion on $\Real^D$. The forward process is interpreted as transforming the empirical data distribution $\hmu_n$ to the Gaussian distribution. From the literature of Markov diffusion, it is well known $\hP_t := \operatorname{Law}(\hX_t)$ approaches the Gaussian distribution $\gamma_D$, exponentially fast, in terms of various divergence measures including KL and TV. Under Assumption~\ref{a1}, i.e., when $t \mapsto \beta_t$ is bounded (above and away from zero) and continuously differentiable, $\hP_t$ admits a density  $\hp_t$ with respect to the Lebesgue measure. This mild regularity condition ensures the stability of the forward process by preventing abrupt changes and guarantees the existence of the backward process. The standard OU process (i.e., when $\beta_t \equiv 1$) satisfies this regularity condition. Further, it is well-known that $\hX_t |\hX_s \sim \cN\left(e^{-\int_s^t \beta_\tau \diff  \tau } \hX_s, (1-e^{- 2\int_s^t \beta_\tau \diff  \tau}) I_D \right)$, making the forward process computationally tractable. See Proposition~\ref{prop_forward_conditional} for a proof of this result.

\paragraph{Reverse Process}
Under mild regularity conditions on $\{\beta\}_{t\ge 0}$ (see Assumption~\ref{a1}), the reverse process $\{Y_t\}_{0 \le t \le T} = \{\hX_{T-t}\}_{0 \le t \le T}$ satisfies the SDE,
\begin{align}
    \diff  Y_t = \beta_{T-t} \left( Y_t + 2 \nabla \log \hp_{T-t}(Y_t)\right) \diff  t + \sqrt{2 \beta_{T-t}} \diff  \Tilde{W}_t, \, Y_0 \sim \hP_T, \label{eq_rp}
\end{align}
where $\{\Tilde{W}_t\}_{t \ge 0}$ is another Brownian motion independent of $\{W_t\}_{t \ge 0}$. This result is proved in Proposition~\ref{reverse_guarantee} in our context.  Since the score function $\nabla \log \hat{p}_{T-t}$ is unknown, one estimates it by minimizing a Mean Squared Error (MSE) loss as described in the next paragraph. The score function estimate for $\nabla \log \hat{p}_t(x)$ is denoted by $\hat{s}(x,t)$. Further, due to computational limitations of exactly implementing the reverse SDE, one discretizes the interval $[0,T]$ for the reverse process as $ t_0 \le t_1 \le \dots \le t_N$. We will take $t_0 = 0$ and $t_N = T-\delta_0$. The reverse process is stopped at time $T-\delta_0$ instead at $T$ to prevent variance explosion near the data support, which is potentially low-dimensional. $\delta_0$ is known as the early stopping time of the reverse process and plays a crucial role in ensuring the stability of the score estimates both from a theoretical \citep{debortoli2022convergence,benton2024nearly,pmlr-v238-tang24a} and methodological viewpoint \citep{songscore,lai2025principlesdiffusionmodels}. We will use the exponential integrator scheme to approximate the reverse process as,
\begin{align}
    \diff  \hat{Y}_t = \beta_{T-t} (\hat{Y}_t + 2 s(\hY_{t_i}, T-t_i)) \diff  t + \sqrt{2 \beta_{T-t}} \diff  \Tilde{W}_t, \ t_i \le t \le t_{i+1} \text{ and } \hY_0 \sim \gamma_D. \label{e22}
\end{align}
Note that one starts the reverse process from $\gamma_D$ and not $\hat{P}_T$ as the former is easy to sample from and the two distributions are close (in the KL sense) when $T$ is large (see Lemma~\ref{kl_bd}). It can be easily shown (see Lemma~\ref{lem_a.1}) that \eqref{e22} is solved by taking
\begin{align}
    \hY_{t_{i+1}} = & \hY_{t_i} + \left(e^{\int_{T-t_{i+1}}^{T-t_i} \beta_\tau \diff  \tau} - 1\right)\left(\hY_{t_i} + 2 s(\hY_{t_i}, T-t_i)\right) + Z_{t_i} \sqrt{e^{2\int_{T-t_{i+1}}^{T-t_i} \beta_\tau \diff  \tau} - 1} ; \nonumber\\
    & \hY_0, Z_{t_0}, \dots, Z_{t_N} \overset{i.i.d.}{\sim} \gamma_D. \label{eq_ito}
\end{align}
For notational simplicity, we define $Q_t = \operatorname{Law}(Y_t)$ and $\hQ_t(s) = \operatorname{Law}(\hY_t)$. We write $\hQ_t$ as a function of $s$ to denote that it is dependent on the choice of $s$, which is estimated as described in the next paragraph. 

\paragraph{Score matching} The score function in diffusion models is estimated by minimizing the so-called MSE loss. In this paper, we will consider the following weighted MSE loss:
\begin{align}
    \hat{s}_n = \argmin_{s \in \sS} \sum_{i=0}^{N-1} h_i \|s(\cdot, T-t_i) - \nabla \log \hp_{T-t_i}(\cdot)\|^2_{\fL_2(\hP_{T-t_i})}, \label{sm_1}
\end{align}
where, $h_i = t_{i+1} - t_{i}.$ Here, $\sS$ is a class of functions from $\Real^D \to \Real^D$. In practice, this function class is realized by neural networks. We take $\sS = \cR\cN(L, W, B)$. Again, as $\nabla \log \hp_{t_i}(\cdot)$ is unknown, one does an equivalent reformulation of the objective as,  
\begin{align}
    \hat{s}_n
    = \argmin_{s \in \sS} \sum_{i=0}^{N-1} h_i \E \left\| s\left(\hX_{t_i} , t_i\right) + \frac{Z_{t_i}}{\sqrt{1-e^{- 2\int_0^{t_i} \beta_\tau \diff  \tau}}}  \right\|^2 , \label{sm1}
\end{align}
where $\left\{Z_{t_i}\right\}_{i \in [N]}$ are i.i.d. standard Gaussian random variables on $\Real^D$. Here, $\hX_{t_i} = e^{-\int_0^{t_i} \beta_\tau \diff  \tau } \hX_0 + \sqrt{1-e^{- 2\int_0^{t_i} \beta_\tau \diff  \tau}} Z_{t_i}$ We refer the reader to Lemma~\ref{stein} for a formal proof of this equivalence. In practice, one typically can not compute the expectation in \eqref{sm1} directly, but estimate it through Monte Carlo sampling as:
\begin{align*}
    X_{0}^{(j)} \sim \operatorname{Unif}\left(\{X_1, \dots , X_n\}\right) , \, Z^{(j)}_{t_i} \sim \gamma_D, \, j = 1, \dots, m_i
\end{align*}
\begin{align}
    \hat{s}^{\mathrm{mc}}_n
    = \argmin_{s \in \sS}   \sum_{i=0}^{N-1}  \sum_{j=1}^m \frac{h_i}{ m_i} \left\| s\left(\hX_{t_i}^{(j)} , t_i\right) + \frac{Z_{t_i}^{(j)}}{\sqrt{1-e^{- 2\int_0^{t_i} \beta_\tau \diff  \tau}}}  \right\|^2 , \label{sm2}
\end{align} 
where, $\hX_{t_i}^{(j)} = e^{-\int_0^{t_i} \beta_\tau \diff  \tau } X_{0}^{(j)} + \sqrt{1-e^{- 2\int_0^{t_i} \beta_\tau \diff  \tau}} Z_{t_i}^{(j)}$. The goal of our theoretical analyses is to bound $\fW_p(\sT_R(\hQ_{t_N}(\hat{s}_n)), \mu)$ and $\fW_p(\sT_R(\hQ_{t_N}(\hat{s}^{\mathrm{mc}}_n)), \mu)$ for appropriate choices of the forward process stopping time ($T$), backward process early stopping time ($\delta_0$), time partition ($\{t_i\}_{i=0}^N$), score function class ($\sS$),  the truncation level $R$, and the number of Monte Carlo samples $\{m_i\}_{i=0}^{N-1}$.

\section{Intrinsic Data Dimension}
\label{intrinsic}
In practice, real-world data is often believed to lie on a lower-dimensional structure embedded within a high-dimensional feature space. To formalize this intuition, researchers have introduced various notions of the \emph{effective dimension} of the underlying probability measure generating the data. Among these, the most widely used approaches rely on characterizing the growth rate, on the logarithmic scale, of the covering number of most of the measure's support.   Let $(\cS, \varrho)$ be a Polish space, and let $\mu$ be a probability measure defined on it. Throughout this paper, we take $\varrho$ to be the $\ell_\infty$-norm. Before proceeding, we recall the notion of an $(\epsilon, \tau)$-cover of a probability measure \citep{posner1967epsilon}, defined as  
\( \sN_\epsilon(\mu, \tau) = \inf\{ \cN(\epsilon; S, \varrho) : \mu(S) \ge 1 - \tau \}, \)  
i.e., $\sN_\epsilon(\mu, \tau)$ is the minimal number of $\epsilon$-balls required to cover a set $S$ with probability at least $1-\tau$.  

Perhaps the most rudimentary measure of dimensionality is the \emph{upper Minkowski dimension} of the support of $\mu$, given by  
\[
\overline{\text{dim}}_M(\mu) = \limsup_{\epsilon \downarrow 0} \frac{\log \cN(\epsilon; \text{supp}(\mu), \ell_\infty)}{\log(1/\epsilon)}.
\]  
This notion depends solely on the covering number of the support and does not require the support to be smoothly embedded in a Euclidean space of smaller dimension. As a result, it captures not only smooth manifolds but also highly irregular sets such as fractals.  The statistical implications of the upper Minkowski dimension have been extensively studied. For instance, \citet{kolmogorov1961} analyzed how covering numbers of various function classes depend on the Minkowski dimension of the support. More recently, \citet{JMLR:v21:20-002} demonstrated that deep learners can exploit this intrinsic low-dimensional structure, which manifests in their convergence rates, while \citet{JMLR:v23:21-0732} and \cite{chakraborty2024a} showed similar adaptivity for Wasserstein GANs (WGANs) and Wasserstein Autoencoders (WAEs), respectively.  

A well-known limitation \citep{JMLR:v26:24-0054}, however, is that the upper Minkowski dimension can be large if the measure spreads over the entire sample space, even if it is highly concentrated in certain regions.  To overcome this difficulty, \citet{weed2019sharp} extended Dudley's \citep{dudley1969speed} notion of entropic dimension to characterize the expected convergence rate of the Wasserstein-$p$ distance between a distribution $\mu$ and its empirical counterpart. They introduced the notions of \emph{upper} and \emph{lower Wasserstein dimensions}, defined as follows:

\begin{defn}[Upper and Lower Wasserstein Dimensions \citep{weed2019sharp}]\label{def1} \normalfont
For any $p > 0$, the $p$-upper dimension of $\mu$ is given by  
\[
d^\ast_p(\mu) \;=\; \inf \left\{ s > 2p : \limsup_{\epsilon \downarrow 0} \frac{\log \sN_{\epsilon}\!\left(\mu, \epsilon^{\tfrac{s p}{s-2 p}}\right)}{\log(1/\epsilon)} \;\leq\; s \right\}.
\]  
The \emph{lower Wasserstein dimension} of $\mu$ is defined as  
\(d_\ast(\mu) \;=\; \lim_{\tau \downarrow 0}\; \liminf_{\epsilon \downarrow 0} \frac{\log \sN_\epsilon(\mu,\tau)}{\log(1/\epsilon)}.\)
\end{defn}
\citet{weed2019sharp} established that, approximately,
\(n^{-1/d_\ast(\mu)} \lesssim \E \fW_p(\hmu_n,\mu) \lesssim n^{-1/d^\ast_p(\mu)}\). However, the result is only applicable when the target measure $\mu$ is supported on a compact set with diameter at most $1$.  To understand the behavior of $\fW_p(\hmu_n,\mu)$ for unbounded measures with a finite moment condition, we propose the $(p, q)$-Wasserstein dimension as follows.
\begin{defn} \label{new_dim}
For any $0< p < q < \infty$, the $(p, q)$-Wasserstein dimension of $\mu$ is defined as:
    \[d^\star_{p, q}(\mu) = \inf \left\{s > 2p: \limsup_{\epsilon \downarrow 0} \frac{\log \sN_\epsilon\left(\mu, \epsilon^{\frac{s  p q }{(q-p)(s - 2 p)}}\right)}{\log(1/\epsilon)} \le s\right\}.\]
\end{defn}

The proposed $(p,q)$-Wasserstein dimension is compared with different notions of intrinsic dimension popular in the literature in Proposition~\ref{relation_dimension}.  Before proceeding with the comparison, we recall the definition of regularity dimensions and packing dimension of a measure \citep{fraser2017upper}.
\begin{defn}[Regularity dimensions]
\normalfont
The upper and lower regularity dimensions of a measure are defined as:
    \begin{align*}
    \overline{\text{dim}}_{\text{reg}}(\mu) = & \inf\bigg\{ s: \exists \, C>0 \text{ such that, } \forall \ 0 < r < R \text{ and } x \in \text{supp}(\mu), \, \frac{\mu(B(x,R))}{\mu(B(x,r))} \le C \left(\frac{R}{r}\right)^s\bigg\},\\
        \underline{\text{dim}}_{\text{reg}}(\mu) = & \sup\bigg\{ s: \exists \, C >0\text{ such that, } \forall \ 0 < r < R  \text{ and } x \in \text{supp}(\mu), \, \frac{\mu(B(x,R))}{\mu(B(x,r))} \ge C \left(\frac{R}{r}\right)^s\bigg\}.
    \end{align*} 
\end{defn}

\begin{defn}[Upper packing dimension] \normalfont The upper packing dimension of a measure $\mu$ is defined as:
    $\overline{\text{dim}}_P(\mu) = \operatorname{ess\, sup} \left\{\limsup\limits_{r \rightarrow 0} \frac{\log \mu(B(x,r))}{\log r}: x \in \text{supp}(\mu)\right\}$.
\end{defn}
The following proposition compares the $(p,q)$-Wasserstein dimension and these other notions of intrinsic dimension used in the literature.
\begin{restatable}{propos}{propone}\label{relation_dimension}
    For any probability measure $\mu$ and $0< p < q < \infty$,
    \begin{enumerate}
        \item[(a)] $d^\ast_p(\mu) \le d^\star_{p,q}(\mu)$,
        \item [(b)] $d^\star_{p,q}(\mu) $ is non-increasing in $q$,
        \item[(c)] $d^\star_{p,q}(\mu) $ is non-decreasing in $p$,
        \item[(d)] For any $0 < p < \overline{\text{dim}}_M(\mu)/2$, $d^\star_{p,q}(\mu) \le \overline{\text{dim}}_M(\mu)$,
        \item[(e)] For any $0 < p < \overline{dim}_P(\mu)/2  $, $d^\star_{p,q}(\mu) \le \overline{dim}_P(\mu) \le \overline{dim}_{\text{reg}}(\mu)$,
        \item[(f)] $\underline{dim}_{\text{reg}}(\mu) \le d^\star_{p,q}(\mu)$.
    \end{enumerate}
\end{restatable}

The $(p, q)$-Wasserstein dimension can be used to upper bound the Wasserstein-$p$ distance between $\hmu_n$ to $\mu$, under a finite moment condition. We can show that if $\sM_q(\mu) < \infty$ and $ 1\le p < q$, $\E\fW_p(\mu, \hmu_n)$ roughly scales as, $\cO\left(n^{-1/d^\star_{p,q}(\mu)}\right)$. To show this, we derive the following theorem that bounds $\E \fW^p_p(\mu, \hmu_n)$, when $n$ is large enough and $p > 0$, with proof appearing in Appendix~\ref{ap_moment_thm}. 

\begin{theorem}\label{moment_thm}
     Suppose that $\sM_q(\mu) < \infty$ and $0 < p < q$. Then for any $d > d^\star_{p, q}(\mu)$, there exists constant $n_0\in \fN$ and $c >0$ that may depend on $d, \mu, p$ and $q$, such that of $n \ge n_0$,
     \[\E \fW^p_p(\mu, \hmu_n) \le c \ n^{-p/d}.\]
 \end{theorem}
 Note that under the stronger assumption that the support is compact, Theorem~1 of \citet{weed2019sharp} implies that for any $d > d_p^\ast(\mu)$,
\(\mathbb{E}\,\fW_p^p(\hat{\mu}_n,\mu) \lesssim n^{-p/d},\) for $ p \ge 1.$ Hence, when the data support is compact, Proposition~\ref{relation_dimension}(a), together with Theorem~1 of \citet{weed2019sharp}, immediately yields Theorem~\ref{moment_thm}. In contrast, when the support is not compact, we impose a stronger condition on the covering numbers of most of the mass of the data distribution. This is captured through the definition of the $(p,q)$-Wasserstein dimension, which allows us to extend the result from the compact-support setting to the more general case of distributions with bounded moments.

\section{Theoretical Analyses}
\label{theo_ana}

  \subsection{Assumptions}\label{sec_assumptions}
  To lay the foundation for our analysis of deep diffusion models, we introduce a set of assumptions that form the basis of our theoretical investigations. These assumptions encompass the underlying data distribution and the ``smoothness" of the process.  For the purpose of the theoretical analysis, we assume that the data are independent and identically distributed from some unknown target distribution $\mu$ on $\Real^D$. This is a standard assumption in the study of generative models \citep{liang2021well, JMLR:v23:21-0732, JMLR:v26:24-0054,pmlr-v238-tang24a,oko2023diffusion} and is stated formally as follows:

\begin{ass}
    \label{a2}
    We assume that $X_1, \dots, X_n$ are independent and identically distributed according to the probability distribution $\mu$, such that $\sM_q (\mu) : = \left(\int \|x\|^q \diff \mu\right)^{1/q} < \infty$ for some $q > 2$.
\end{ass}
It should be noted Assumption~\ref{a2} imposes only a mild integrability requirement on the target measure $\mu$, namely that it possesses a finite $q$-th moment. This condition ensures that $\mu$ is sufficiently well-behaved to admit meaningful error and convergence guarantees, while remaining broad enough to encompass a wide class of distributions encountered in practice, including heavy-tailed distributions. Importantly, compared to existing formulations for both GANs \citep{liang2021well, JMLR:v23:21-0732, JMLR:v26:24-0054} and diffusion models \citep{pmlr-v238-tang24a,oko2023diffusion}, Assumption~\ref{a2} is considerably weaker. In particular, it neither constrains the support of $\mu$ to lie on a (sub)manifold nor requires compactness of the support. Moreover, we do not assume the existence of a density with respect to the Lebesgue measure, nor do we rely on any Poincaré, log-Sobolev inequalities or sub-Gaussian tails. This relaxation of the assumptions significantly broadens the applicability of our theoretical results beyond settings considered in prior work.

To ensure the regularity of both the forward and backward diffusion processes, we assume that the time scaling $\{\beta_t\}$ is upper and lower bounded by positive constants $\overline{\beta}$ and $\underline{\beta}$, respectively. Formally, 
  \begin{ass}\label{a1}
    $0 < \underline{\beta} \le \beta_t \le \overline{\beta} < \infty$, for all $t \ge 0$. Further $t\mapsto \beta_t \in \sC^1$ on $[0,\infty)$.
\end{ass}
Assumption~\ref{a1} imposes a mild regularity condition on the time-scaling sequence $\{\beta_t\}_{t \ge 0}$, ensuring that it varies smoothly over time. This assumption plays several important roles in our analysis. First, it guarantees the stability of the forward diffusion process by preventing abrupt changes in the noise-injection rate, which could otherwise lead to ill-posed behavior or numerical instability. Second, this smoothness condition ensures that the forward process admits a well-defined reverse-time dynamics as shown in Proposition~\ref{reverse_guarantee}; the existence of such a reverse diffusion is central to both the theoretical formulation of score-based generative modeling and its algorithmic implementation. Third, Assumption~\ref{a1} ensures that, as the terminal time $T \to \infty$, the forward process converges to the reference measure $\gamma_D$. Such a condition is standard in diffusion-based analyses and primarily serves to control the rate at which noise is injected into the system.  

 \begin{restatable}{propos}{proptwo}
    \label{reverse_guarantee}
Under Assumption \ref{a1},  the reverse process $\{Y_t\}_{0 \le t \le T} = \{\hX_{T-t}\}_{0 \le t \le T}$ satisfies the SDE of Equation \eqref{eq_rp}.
\end{restatable}
\paragraph{Partition} To approximate the continuous-time dynamics of the backward process, we employ a discrete-time scheme that iteratively updates the particles according to the drift and diffusion coefficients derived from the backward process as in equation \eqref{eq_ito}. We take the following choice of partition. 

  \begin{itemize}
      \item Take $t_0^\prime = \delta_0$.
      \item Define $h_i^\prime = \kappa \min\{t_i^\prime, 1\}$, $i \in \{0, \dots , N-1\}$, where $0 < \kappa \le 1$.
      \item $t_{i+1}^\prime = t_i^\prime + h_i^\prime$, $i \in \{0, \dots , N-1\}$.
      \item Take $t_i = T - t^\prime_{N-i}$.
  \end{itemize}
  This choice of partition is standard in the diffusion literature \citep{chen2023improved, benton2024nearly, huang2024denoising, potaptchik2024linear}, as it facilitates fast convergence of the backward process. Here, $\kappa$ can be interpreted as the maximum step size. The partitioning scheme is designed so that, as the reverse process approaches time $T - \delta_0$, the intervals become progressively finer. The main idea is that coarser early steps set
the global structure, later steps add fine detail, and the sample becomes more
realistic as the noise is removed. This adaptive refinement near the terminal time mitigates numerical instability and improves the accuracy of approximating the score function in regions where the variance of the process might explode due to possible singularities in $\mu$. Consequently, the discretized backward dynamics more faithfully capture the behavior of the underlying continuous-time diffusion. It can be shown that for the above partitioning scheme, the following lemma holds. This lemma plays a crucial role in the subsequent proofs by controlling the discretization error with a suitable choice of $\kappa$. The proof of this result can be found in Appendix~\ref{ap_discretization}.

  \begin{restatable}{lemma}{lemdis}\label{lemdis}
  For the above partitioning scheme, the discretization error is bounded by, 
      \begin{align}
    \sum_{i=0}^{N-1}\int_{t_i}^{t_{i+1}} \beta^2_{T-t} \E \|\nabla \log \hp_{T-t_i}(\hX_{T-t_i}) - \nabla \log \hp_{T-t}(\hX_{T-t})\|^2 \diff  t 
    \lesssim \kappa \left(\log(1/\delta_0) + T \right) . \label{e_dis}
\end{align}
Moreover, the number of points in the partition, $N \le \frac{\log(1/\delta_0)}{\log(1+\kappa)} +  \frac{T}{\kappa} \lesssim \frac{1}{\kappa}\left(\log(1/\delta_0) + T \right)$.
  \end{restatable}
\subsection{Main Result}\label{main results} 
Under Assumptions~\ref{a2} and \ref{a1}, we derive an upper bound on the expected Wasserstein-$p$ distance between the score-matching estimator $\sT_R(\hQ_{T-\delta_0}(\hat{s}))$ and the target distribution $\mu$. The bound scales with the number of available training samples, and its rate exponent depends only on the intrinsic $(p, q)$-Wasserstein dimension of $\mu$, rather than on the ambient data dimension $D$. Thus, for suitable data-dependent choices of the model hyper-parameters, such as the forward process stopping time ($T$), backward process early stopping time ($\delta_0$), time partition ($\{t_i\}_{i=0}^N$) and score function class ($\sS$), the estimator obtained by minimizing the score-matching MSE loss adapts to the low-dimensional structure of the underlying distribution. The main theoretical guarantee is stated in the following theorem.

\begin{restatable}[Error rates for score-matching diffusion models]{theorem}{mainthm}\label{mainthm}
    Suppose that $d > d^\star_{p, q}(\mu) $ and $1 \le p < q$. Assume that Assumptions \ref{a2} and \ref{a1} hold. Then, with the choice of the partition as stated in Section \ref{sec_assumptions}, if, $T \ge \frac{p}{\underline{\beta}}\bigg( \frac{1 + p(q-p)}{dp(q-p)}\log n + \frac{1}{2} \log D + \frac{1}{q-p} \log( \sM_q^q(\mu) + \sM_q^q(\gamma_D)) 
    + \frac{1}{2p} \log(D + \sM_2^2(\mu)) + \frac{q-1}{q-p} \log 2\bigg)$, 
    $R = 2^{\frac{q-p}{q-p}} n^{\frac{1}{d p (q-p)}} \left(  \sM_q^q(\mu) +   \sM_q^q(\gamma_D) \right)^{\frac{1}{q-p}}$, $\delta_0 = n^{-2/d}$, and $\kappa \asymp  n^{-\frac{2 (1+p(q-p))}{d p (q-p)}} $, 
    \begin{equation}
        \E \fW_p\left(\mu, \sT_R\left(\hQ_{T-\delta_0}(\hat{s}_n)\right)\right) \lesssim n^{-1/d} \operatorname{poly-log}(n), \label{s_n}
    \end{equation}
    when $\sS = \cR\cN(L,W,B)$, with, $W \asymp \mathsf{M}_{n,x}^{D} n^{\frac{D+2 + (Dp + D + 2p)}{d(q-p)}}
\left(\log^2 n + \log n \cdot \log \mathsf{M}_{n,x}\right)$, $L \asymp \log n + \log \mathsf{M}_{n,x}$, and $B \asymp n^{\frac{1+(p+3)(q-p)}{d(q-p)}} \mathsf{M}_{n,x}$, where $\mathsf{M}_{n,x} = \max_{1 \le i \le n} \|X_i\|_\infty \vee 1$.

Further, there exists constants $c_m^i$ and a polynomial $\phi$
such that  if \[m_i \ge c_m^i \frac{\sigma_{t_i}^4}{h_i^2} n^{\frac{D+2 + (Dp + D + 2p)}{d(q-p)} + \frac{4(1+p(q-p))}{d(q-p)}} \phi 
(\log n , \mathsf{M}_{n,x}),\] then, 
\begin{equation}
    \E \fW_p\left(\mu, \sT_R\left(\hQ_{T-\delta_0}(\hat{s}_n^{\text{mc}})\right)\right) \lesssim n^{-1/d} \operatorname{poly-log}(n). \label{s_mc}
\end{equation}
\end{restatable}
In simpler terms, the result demonstrates that the expected Wasserstein-$p$ distance between the target measure $\mu$ and the estimates obtained by score-matching diffusion models roughly scales as $\widetilde{\cO}\left( n^{-1/d^\star_{p,q}(\mu)}\right)$ under a finite $q$-moment condition on $\mu$. It is important to note that, in contrast to existing results, our error bounds are derived in terms of the general Wasserstein-$p$ distance, rather than being limited to the special case of the Wasserstein-$1$ distance \citep{pmlr-v238-tang24a,oko2023diffusion,azangulov2024convergence} or Wasserstein-$2$ distance \citep{gao2025wasserstein,beyler2025convergence} as considered in the recent literature. This generalization provides a finer analysis of the convergence behavior, as the Wasserstein-$p$ distance for $p>1$ captures higher-order geometric discrepancies between probability measures. Furthermore, the target distribution in our framework is only assumed to satisfy a mild finite-moment condition, without requiring support constraints such as being confined to a compact subset of $\Real^D$, a key assumption in the recent works \citep{pmlr-v238-tang24a, oko2023diffusion, azangulov2024convergence}. In addition, unlike \citet{pmlr-v238-tang24a, azangulov2024convergence, chen2023score}, our analysis does not impose any subspace, manifold or smoothness assumptions on the underlying data distribution, thereby substantially broadening the scope of applicability of our results to real-data settings.

\paragraph{Inference for Data Supported on a Manifold} 
Recall that a set $\sA$ is said to be $\tilde d$-regular with respect to the $\tilde d$-dimensional Hausdorff measure $\fH^{\tilde d}$ if 
\(\fH^{\tilde d}(B_\varrho(x,r)) \asymp r^{\tilde d},  \forall\, x \in \sA, \)
as defined in \citet[Definition~6]{weed2019sharp}. A classical result \citep[Proposition~8]{weed2019sharp} shows that if $\operatorname{supp}(\mu)$ is $\tilde d$-regular and $\mu \ll \fH^{\tilde d}$, then for any $p \in [1, \tilde d/2]$ one has 
\(d_\ast(\mu) = d_p^\ast(\mu) = \tilde d.\)
Consequently, Proposition~\ref{relation_dimension} implies that \(d^\star_{p,q}(\mu) = \tilde d.\)
By Theorem~\ref{mainthm}, it therefore follows that score-matching diffusion models achieve the convergence rate
\(\E\,\fW_p\!\left(\mu,\,\sT_R\!\left(\hQ_{T-\delta}(\hat s)\right)\right)
     = \widetilde{\cO}\!\left(n^{-1/\tilde d}\right),\)
for the estimators $\hat s = \hat s_n$ and $ \hat{s}^{\mathrm{mc}}_n$ defined in \eqref{sm1}--\eqref{sm2}. Moreover, since every compact $\tilde d$-dimensional differentiable manifold is $\tilde d$-regular \citep[Proposition~9]{weed2019sharp}, the same rate $\widetilde \cO(n^{-1/\tilde d})$ holds whenever $\operatorname{supp}(\mu)$ lies on such a manifold. This recovers, up to improved constants, the rates obtained by \citet{pmlr-v238-tang24a} in the special case $\alpha = 0$, corresponding to distributions with densities supported on a $\tilde d$-dimensional manifold. Similar conclusions apply when $\operatorname{supp}(\mu)$ is a nonempty compact convex subset of an affine space of dimension $\tilde d$, the relative boundary of a compact convex set of dimension $\tilde d + 1$, or a self-similar set with similarity dimension $\tilde d$ as all these sets are $\tilde{d}$-regular.

\paragraph{Comparison with Existing GAN Error Bounds}
Our results on score-based diffusion models parallel, and in several respects exceed, the existing theory of finite-sample convergence for GANs. Classical analyses of GANs \citep[e.g.,][]{singh2018nonparametric,uppal2019nonparametric} typically establish convergence in the $\beta$-H\"older Integral Probability Metric (IPM), with rates of order $\E \|\hmu^{\text{GAN}} - \mu\|_{\sH^\beta} = \widetilde{\cO}(n^{-\beta/D})$ under smoothness and compact-support assumptions, where $D$ denotes the ambient data dimension. Later refinements, such as those of \citet{dahal2022deep,JMLR:v23:21-0732}, demonstrate that when the data distribution is supported on a compact $\tilde{d}$-dimensional manifold, the convergence rate improves to $\E \|\hmu^{\text{GAN}} - \mu\|_{\sH^\beta} = \widetilde{\cO}\left(n^{-\beta/\tilde{d}}\right)$. \citet{JMLR:v26:24-0054} established that when the data support is within the unit hypercube, $[0,1]^D$, error rates for GANs in the $\beta$-H\"older IPM scales roughly as $\cO\left(n^{-\beta/d^\ast_\beta(\mu)}\right)$. Our analysis for diffusion-based estimators achieves a similar intrinsic dimension adaptive rate but under significantly weaker regularity conditions: it requires only a finite-moment assumption on $\mu$, without the need for compact support, bounded densities, or explicit manifold structure. In addition, our bounds are established in the general Wasserstein-$p$ metric ($p \ge 1$), thereby capturing higher-order geometric discrepancies between the estimated and target distributions, whereas most GAN analyses are confined to some metric that is in the form of an IPM (e.g. $\beta$-H\"older IPM or Wasserstein-$1$ distance). The dependence of our rates on the intrinsic $(p, q)$-Wasserstein dimension of $\mu$ reveals that diffusion models adapt automatically to low-dimensional structure in the data, a property that has been observed for GANs only under restrictive settings. 

\paragraph{Choice of Stopping Times for the Forward and Backward Processes}
The choice of the diffusion horizon $T$ and the early stopping offset $\delta_0$ plays a crucial role in balancing approximation and estimation errors in score-based diffusion models. From a theoretical perspective, taking $T \to \infty$ ensures that the forward process fully converges to the Gaussian prior; however, in practice and in finite-sample analyses, excessively large values of $T$ amplify numerical instability and estimation variance. Our results, therefore, consider a finite stopping time $T$ that grows only logarithmically with the sample size ($n$), which suffices to guarantee that the residual error from incomplete diffusion remains negligible relative to the statistical estimation error.  

Similarly, the backward process is terminated at $T-\delta_0$, for a small but positive $\delta_0 > 0$, instead of stopping exactly at $T$, to prevent instability in the score approximation near the data manifold. This truncation is standard in the analysis of score-based and denoising diffusion models \citep{songscore,lai2025principlesdiffusionmodels} and ensures that the learned score $\hat{s}$ remains well-behaved in regions of low density. In our framework, $\delta_0$ is chosen as $\cO(n^{-2/d})$ to control the variance explosion for the backward process near the data support which might result in singularities.

\paragraph{Choice of Partition}
The discretization of the backward diffusion processes plays a critical role in both the theoretical analysis and the practical performance of diffusion models. In our framework, we employ a nonuniform time partition that allocates finer resolution near regions where the drift or score function exhibits high curvature or rapid variation. This design ensures that the approximation error arising from discretizing the continuous-time process remains well-controlled while avoiding unnecessary computational overhead in smoother regions. Such adaptive or variance-aware partitioning schemes have been shown to yield significant empirical improvements in sample quality and likelihood estimation across diffusion-based generative models \citep{songscore,ho2020denoising,kingma2021variational,lu2022dpm}. From a theoretical standpoint, this popular choice of partition \citep{benton2024nearly,chen2023improved} ensures that the cumulative discretization error decays at a rate that matches the statistical convergence rate of the learned score function, thereby maintaining overall optimality of the derived error bounds. The choice of $\kappa \asymp n^{-\frac{2 (1+p(q-p))}{d(q-p)}} $ ensures that the number of partitions $N \lesssim \frac{\log (1/\delta_0) + T}{\kappa} \asymp n^{\frac{2 (1+p(q-p))}{d(q-p)}} \log n$. Hence, the number of partitions grows only polynomially with the sample size, ensuring both theoretical tractability and computational scalability.

\paragraph{Scaling of the Number of Monte Carlo Samples} 
When the target measure is compactly supported, from Theorem~\ref{mainthm}, it suffices to choose $m_i$ roughly as \[m_i = \widetilde{\cO}\left(\frac{\sigma_{t_i}^4}{h_i^2} n^{\frac{D+2 + (Dp + D + 2p)}{d(q-p)} + \frac{4(1+p(q-p))}{d(q-p)}}\right),\] barring poly-log factors in $n$. Since, $\kappa \asymp n^{-\frac{2 (1+p(q-p))}{d(q-p)}} $ and $\delta_0 \asymp n^{-2/d}$, $\frac{\sigma_{t_i}^4}{h_i^2} \le \frac{1}{\min_{1 \le i \le N }h_i^2} = 1/\kappa^2 \asymp n^{\frac{4 (1+p(q-p))}{d(q-p)}}$. Thus, it suffices to choose $m_i$ roughly as $\widetilde{\cO}\left(n^{\frac{D+2 + (Dp + D + 2p)}{d(q-p)}  + \frac{8p(1+p(q-p))}{d(q-p)}}\right)$. Thus it suffices to have $\sum_{i=0}^{N-1} m_i \asymp  n^{\frac{D+2 + (Dp + D + 2p)}{d(q-p)}  + \frac{10p(1+p(q-p))}{d(q-p)}}$, which is of the form $\widetilde{\cO}\left(n^\frac{\alpha_1+D}{\alpha_2+ \alpha_3 \cdot d}\right)$, for constants $\{\alpha_i\}_{i=1}^3$. Thus, the dimension affects the sample scaling; higher $D/d$ values necessitate more generated samples to achieve equivalent accuracy. This dimension-dependent scaling of the number of Monte Carlo samples also appears in recent works by \citet{JMLR:v26:24-0054} and \citet{JMLR:v23:21-0732} for GANs.

\paragraph{Network Architecture Scaling} Note that the data-dependent term in Theorem~\ref{mainthm}, i.e., $\mathsf{M}_{n,x}$ increases at a rate of $\cO(n^{1/q})$ under Assumption \ref{a2} (see Lemma~\ref{scaling}). To achieve the generalization guarantees, it suffices to set the network depth as $L \asymp \log n$ and scale the network parameters as
$B \asymp n^{\frac{1+(p+3)(q-p)}{d(q-p)} + \frac{1}{q}}$, and $W \asymp  n^{\frac{D+2 + (Dp + D + 2p)}{d(q-p)} + \frac{D}{q}}
\log^2 n$. 
The scaling consists of two distinct components. The primary terms scale with the intrinsic and extrinsic dimensions $d$ and $D$, reflecting the geometric complexity of approximating functions on low-dimensional structures. The additional factors $1/q$ in the exponent of $B$ and $D/q$ in the exponent of $W$ account for the growth of the spatial domain $\mathsf{M}_{n,x}$ with sample size. As $n$ increases, samples are increasingly likely to appear farther from the origin, effectively expanding the region over which the network must accurately represent the score function. This domain expansion requires additional network capacity beyond that needed for the intrinsic geometric complexity alone.

\subsection{Proof Sketch of the Main Result}
\label{pf_main_results}
\subsubsection{Error Decomposition}
As a first step toward obtaining a meaningful bound on the excess risk in terms of the Wasserstein-$p$ distance, we derive an oracle inequality that bounds this risk in terms of generalization, approximation, discretization, early stopping, and truncation errors as shown in Lemma~\ref{lem_oracle_2}, with proof appearing in Appendix~\ref{ap_oracle}.

\begin{restatable}{lemma}{lemoracle}\label{lem_oracle_2}Suppose that Assumptions~\ref{a2} and \ref{a1} hold and $s \in \sS$. Then,
\begingroup
\allowdisplaybreaks
    \begin{align}
         \fW_p(\mu, \sT_R(\hQ_{T-\delta_0}(s)))  
        \le & \fW_p(\mu, \hmu_n) + R \sqrt{D} 2^{\frac{p-1}{2p}} \kl(\hP_T, \gamma_D)^{\frac{1}{2p}} 
        \nonumber\\
        & + R (2\overline{\beta}^2 )^{\frac{1}{2p}} \left(\sum_{i=0}^{N-1} h_i  \|s(\cdot, T-t_i) - \nabla \log \hp_{T-t_i}(\cdot)\|^2_{\fL_2\left(\hP_{T-t_i}\right)} \right)^{1/2p} \nonumber\\
        & + R (2\overline{\beta}^2 )^{\frac{1}{2p}} \left(\sum_{i=0}^{N-1}\int_{t_i}^{t_{i+1}} \E\|\nabla \log \hp_{T-t_i}(\hX_{T-t_i}) - \nabla \log \hp_{T-t}(\hX_{T-t})\|^2 \diff  t\right)^{1/2p} \nonumber\\
        & +   \overline{\beta}\delta_0 \sM_p(\hmu_n) + (\overline{\beta}\delta_0)^{1/2} \sM_p(\gamma_D) \nonumber\\
        & + 2^{(q-1)/p} \left( \sM_p^q(\hmu_n) +   \sM_q^q(\gamma_D) \right)^{1/p
        } R^{-(q-p)/p}. \label{eqq_oracle}
    \end{align}
    \endgroup
\end{restatable}

In Lemma~\ref{lem_oracle_2}, each term on the right-hand side admits a clear interpretation in terms of distinct sources of statistical and computational error. The first term, $\fW_p(\mu, \hmu_n)$, quantifies the \textit{generalization gap}, which measures the discrepancy between the population distribution $\mu$ and its empirical counterpart $\hmu_n$ observed from finite samples. This term captures the inherent statistical uncertainty due to sampling and vanishes as the number of training samples increases, typically at a rate depending on the intrinsic dimension of the data. The second term, $\kl(\hP_T, \gamma_D)$, accounts for the \textit{early-stopping error}. Since the forward process ideally converges to the reference distribution $\gamma_D$ only as $t \to \infty$, any finite-time truncation at $T$ introduces bias. This term thus characterizes the trade-off between computational tractability and asymptotic accuracy -- smaller $T$ leads to faster simulations but a larger discrepancy from the stationary distribution. The third term,
\begin{equation}
    \sum_{i=0}^{N-1} h_i \|s(\cdot, T-t_i) - \nabla \log \hp_{T-t_i}(\cdot)\|^2_{\fL_2(\hP_{T-t_i})}, \label{eq_approx_err}
\end{equation}
corresponds to the \textit{score approximation error}. It arises from replacing the true score function $\nabla \log \hp_t(\cdot)$ with its learned estimator $s(\cdot, t)$ at discretized time points $\{t_i\}_{i=0}^{N-1}$. This term reflects the expressive power of the chosen score model class, as well as the quality of optimization used during training.  The fourth term,
\[
\sum_{i=0}^{N-1}\int_{t_i}^{t_{i+1}} \E\|\nabla \log \hp_{T-t_i}(\hX_{T-t_i}) - \nabla \log \hp_{T-t}(\hX_{T-t})\|^2 \, \diff t,
\]
captures the \textit{discretization error} introduced when the continuous-time backward process~\eqref{eq_rp} is approximated by a discrete exponential-integrator scheme~\eqref{eq_ito}. This error quantifies the temporal approximation bias introduced by replacing the true diffusion dynamics with stepwise updates over the partition $\{t_i\}_{i=0}^N$. Notably, under the regularity of the time-scaling sequence ensured by Assumption~\ref{a1}, the discretization error diminishes as the partition becomes finer. Finally, the fifth and sixth terms represent the \textit{truncation error}, which arises from employing the $R$-truncated measure $\sT_R(\hQ_{T-\delta_0}(s))$ in lieu of the full $\hQ_{T-\delta_0}(s)$. The truncation controls the tail behavior of the generated samples and helps bound the Wasserstein-$p$ distance in an efficient way. While this truncation introduces a mild bias, the resulting error can be controlled by properly choosing the level of truncation. 

\subsubsection{Generalization Gap}
The second step to bounding the excess risk is to bound the generalization gap, i.e., the first term in Equation \eqref{eqq_oracle}, w.r.t. the Wasserstein-$p$ metric.  To do so, we employ Theorem~\ref{moment_thm}. An immediate corollary of Theorem~\ref{moment_thm} is that if $q>1$ and $1 \le p < q$, $\E \fW_p(\mu, \hmu_n)$ scales roughly as, $\cO\left(n^{-1/d^\star_{p,q}(\mu)}\right)$. 
 \begin{cor}\label{cor4}
         Suppose that there exists  $q > 1$, such that $\sM_q(\mu) < \infty$ and $1 \le p < q$. Then for any $d > d^\star_{p, q}(\mu)$, there exists constant $n_0\in \fN$ and $c^\prime >0$ that may depend on $d, \mu, p$ and $q$, such that if $n \ge n_0$,
     \[\E \fW_p(\mu, \hmu_n) \le c^\prime \ n^{-1/d}.\]
 \end{cor}

 \subsubsection{Early Stopping Error for the Forward Process}
 Under Assumptions~\ref{a2} and \ref{a1}, the forward process converges exponentially in the Kullback--Leibler (KL) divergence to the standard Gaussian distribution on $\Real^D$ as shown in Lemma~\ref{kl_bd}. This result aligns with the established literature on exponential convergence of OU processes and is closely related to Lemma~9 in~\citep{chen2023improved} and Proposition~4 of \cite{benton2024nearly}. The proof of this result is provided in Appendix~\ref{app_e}.
  \begin{restatable}{lemma}{klbd}\label{kl_bd}
  For any $t \ge \log 2/ \overline{\beta}$, $\kl(\hP_t , \gamma_D) \le \exp\left(-2\underline{\beta}t\right) (D + \sM_2^2(\hmu_n))$.
\end{restatable}
\subsubsection{Approximation Error}
To effectively bound the overall error in Lemma~\ref{lem_oracle_2}, it is essential to control the approximation error term, denoted by~\eqref{eq_approx_err}. Understanding the approximation capabilities of neural networks has been a central theme in modern learning theory over the past decade. Foundational contributions by \citet{cybenko1989approximation} and \citet{hornik1991approximation} established the universal approximation property of neural networks with sigmoid-type activations, showing that sufficiently wide single-hidden-layer networks can approximate any continuous function on compact domains with arbitrary precision. Building on these classical results, a large body of recent works has examined the approximation power of deep neural networks, highlighting the advantages of depth and width in terms of expressivity and efficiency. Important advances in this direction include \citet{yarotsky2017error, petersen2018optimal, shen2019nonlinear, schmidt2020nonparametric, lu2021deep}, among many others. These results collectively demonstrate that deep ReLU networks can approximate smooth or compositional functions with rates that scale favorably with the intrinsic dimension of the problem. In the context of diffusion models, several recent works have extended these ideas to characterize the approximation capabilities of feed-forward ReLU networks for score-based representations of target measures supported on bounded domains, such as the unit hypercube; see, e.g., \citet{pmlr-v238-tang24a, oko2023diffusion}. In contrast, our result below establishes an approximation guarantee for the score function when the target measure $\mu$ is unbounded. A proof of this result can be found in Appendix~\ref{pf_approximation}.
\begin{restatable}{theorem}{thmapprox}
\label{approx_thm}
Suppose that $\kappa \asymp  n^{-\frac{2 (1+p(q-p))}{d(q-p)}} $. Then, there exists a feed-forward ReLU network $s(\cdot, \cdot)$ satisfying 
$\cW(s) \lesssim \mathsf{M}_{n,x}^D 
n^{\frac{D+2 + (Dp + D + 2p)}{d(q-p)}} \log n
,\, 
\cL(s)  \asymp \log n + \log \mathsf{M}_{n,x}$, and  $ \cB(s) \lesssim n^{\frac{1+(p+3)(q-p)}{d(q-p)}} \mathsf{M}_{n,x},$
such that
\begin{align}
\sum_{i=0}^{N-1} h_i \,
\E_{x \sim \hat{P}_{t_i}}
\bigl\|
s(x, t_i) - \nabla \log \hat{p}_{t_i}(x)
\bigr\|^2_\infty
\;\lesssim\;
n^{-\frac{2 (1+p(q-p))}{d(q-p)}} \log n.
\label{e29}
\end{align}
\end{restatable}

\subsection{Discussions on Minimax Lower Bounds} \label{minimax bounds}
The upper bound discussed in Theorem~\ref{mainthm} matches the corresponding minimax lower bound when the support of $\mu$ is regular enough. One such regularity condition is the so-called Minkowski regularity. For simplicity, let us only consider distribution bounded within the unit hypercube $[0,1]^D$. For a bounded metric space $(S, \varrho)$, the upper Minkowski dimension of $S$ is defined as \(\overline{\text{dim}}_M(S) = \limsup_{\epsilon \downarrow 0} \frac{\log \cN(\epsilon; \, S,\, \varrho )}{\log(1/\epsilon)}.\) Similarly, 
the lower Minkowski dimension of $S$ is given by, 
\(\underline{\text{dim}}_M(S) = \liminf_{\epsilon \downarrow 0} \frac{\log \cN(\epsilon; \, S,\, \varrho )}{\log(1/\epsilon)}.\)
If $\overline{\text{dim}}_M(S) = \underline{\text{dim}}_M(S)$, we say that $S$ is Minkowski regular and a has Minkowski dimension of $\text{dim}_M(S) = \lim_{\epsilon \downarrow 0} \frac{\log \cN(\epsilon; \, S,\, \varrho )}{\log(1/\epsilon)}.$ 

Examples of Minkowski-regular sets include $\tilde{d}$-regular sets as defined in Section~\ref{main results}. Prominent examples include compact $\tilde{d}$-dimensional differentiable manifolds, a nonempty compact convex set spanned by an affine space of dimension $\tilde{d}$, the relative boundary of a nonempty, compact convex set of dimension $\tilde{d}+1$, and a self-similar set with similarity dimension $\tilde{d}$. For all these examples, $\text{dim}_M(S) = \tilde{d}$.

Suppose that $\fM \subseteq [0,1]^D$ and let $\Pi_{\fM}$ denote the set of all probability distributions on $\fM$. We assume that one has access to $n$ samples, $X_1, \dots, X_n$, generated independently from $\mu \in \Pi_{\fM}$ as in Assumption~\ref{a2}, without any moment conditions. To characterize this notion of best-performing estimator, we use the concept of minimax risk i.e., the risk of the best-performing estimator that achieves the minimum risk with respect to the least favorable members in $\Pi_{\fM}$. Formally, the minimax risk for the problem is given by,
\[\mathfrak{M}_n  = \inf_{\hat{\mu}} \sup_{\mu \in \Pi_{\fM}} \E_\mu \fW_p(\hat{\mu},\mu), \]
where the infimum is taken over all measurable estimates of $\mu$, i.e., on $\{\hat{\mu}: (X_1, \dots, X_n) \to \Pi_{\Real^D}: \hat{\mu} \text{ is measurable}\}$.  Here, we write $\E_\mu$ to denote that the expectation is taken with respect to the joint distribution of $X_1, \dots, X_n$, which are independently and identically distributed as  $\mu$. \citet{chakraborty2025minimax} showed that for any positive constant $\delta>0$, $\mathfrak{M}_n \gtrsim n^{-\frac{1}{\underline{\text{dim}}_M(S)-\delta}}$, if $n$ is large enough. Thus, if $\fM$ is Minkowski regular and $\operatorname{supp}(\mu) \subseteq \fM$, then, by Proposition~\ref{relation_dimension}, $d^\star_{p,q}(\mu) \le \text{dim}_M(\fM)$, for any $p < \frac{1}{2}\text{dim}_M(\fM)$. Hence, for any $p < \frac{1}{2}\text{dim}_M(\fM)$, the upper bound of Theorem~\ref{mainthm} scales as $\widetilde{\cO}\left(n^{-\frac{1}{\text{dim}_M(S)+\delta}}\right)$, which roughly matches the lower bound of $\Omega\left(n^{-\frac{1}{\text{dim}_M(S)-\delta}}\right)$, except a $\delta$ factor in the exponent and poly-log factors in $n$. It should be noted that the $\delta$ factor in the exponent is an artifact of the use of $\limsup$ and $\liminf$ in the definition of the upper and lower Minkowski dimensions, and can be removed with exact conditions on the growth of the covering numbers. 
\section{Conclusions}
\label{conclusions}
 This paper investigates the theoretical properties of score-matching-based diffusion models when the underlying data distribution lies on an intrinsically low-dimensional structure embedded in a high-dimensional ambient data space. To formalize this intrinsic structure, we introduce the \emph{$(p,q)$-Wasserstein dimension}, extending the classical notion of Wasserstein dimension \citep{weed2019sharp}. Under a finite $q$-moment condition on the target distribution, we prove that the empirical measure converges to the population distribution in Wasserstein-$p$ distance at a rate of \(\mathcal{O}\!\left(n^{-1/d^\star_{p,q}(\mu)}\right),\)
thereby providing a dimension-adaptive upper bound of the sample complexity under the Wasserstein-$p$ loss.

Building on these results, we derive the first Wasserstein-$p$ risk guarantee for score-based diffusion models. By combining sharp bounds on the statistical error of score matching with approximation guarantees for the empirical score functions, we establish an upper bound, on the excess risk under a Wasserstein-$p$ distance, that scales with the ambient $(p,q)$-Wasserstein dimension of the data distribution. This bypasses the curse of dimensionality ($D$) of the ambient feature space. Consequently, we show that score-based diffusion estimators can nearly achieve minimax-optimal estimation rates for distributions supported on intrinsically low-dimensional regular sets, including compact differentiable manifolds.

Beyond the population-level analysis, we also address several algorithmic considerations that arise in practical diffusion implementations---namely, \emph{discretization}, \emph{early stopping}, and \emph{truncation of score estimates}---and provide theoretically motivated prescriptions for these choices. In particular, our bounds imply an early stopping time for the forward diffusion process of order $T = \cO(\log n)$, together with an early stopping rule for the reverse process at time $T - \delta_0$, where $\delta_0 = \Theta(n^{-2/ d})$ is selected to ensure convergence of the forward process to the standard Gaussian distribution and mitigate variance explosion in the reverse process. For the numerical integration of both the backward dynamics, we recommend a partition of $[0,T]$ with exponentially decaying step sizes, which ensures that the accumulated discretization error remains of the same order as the estimation error of the score network. These choices collectively yield a natural trade-off between the aforementioned errors that is dimension-adaptive and compatible with the minimax rates established in our statistical analysis.

While our results provide detailed statistical guarantees, the \emph{optimization error} arising from training deep neural score networks remains challenging to control due to the nonconvex and coupled nature of score matching. Our bounds are optimization-agnostic and can be integrated with future advances in optimization theory for diffusion models. Our approximation and generalization analyses are carried out under the assumption that the score network is realized by a ReLU architecture, which is standard in the statistical learning literature and allows us to leverage existing approximation results for H\"older and Besov-type function classes. However, we note that practical diffusion models typically employ far more structured architectures such as U-Nets or Transformers, which incorporate multiscale convolutional blocks, attention mechanisms, and skip connections. While these architectures often exhibit strong empirical smoothness and hierarchical approximation properties, providing rigorous guarantees for such structured networks remains an open problem. 

\section*{Acknowledgements}
We gratefully acknowledge the support of the NSF and the Simons Foundation for the Collaboration on the Theoretical Foundations of Deep Learning through awards DMS-2031883 and \#814639, the NSF's support of FODSI through grant DMS-2023505, and the ONR's support through MURI award N000142112431. 
\appendix
\section*{Appendices}
\startcontents[sections]
\printcontents[sections]{l}{1}{\setcounter{tocdepth}{2}}
\section{Omitted Proofs from Section~\ref{sec_diff_intro}}
\begin{propos}\label{prop_forward_conditional}
Let $\{\hX_t\}_{t\ge 0}$ solve the SDE \eqref{forward} in $\mathbb R^D$ and
$\beta_t \ge 0$ is measurable with
$A(s,t) := \int_s^t \beta_\tau\,d\tau < \infty.$
Then for any $0 \le s \le t$,
\[
\widehat X_t \mid \widehat X_s
\sim
\mathcal N\!\left(
e^{-A(s,t)} \widehat X_s,\,
\big(1-e^{-2A(s,t)}\big) I_D
\right).
\]
\end{propos}

\begin{proof}
We define the (forward) integrating factor
\(\mathsf{m}_u:=\exp\!\Big(\int_0^u \beta_r\,\mathrm dr\Big), \, u\ge0.\) Since $\mathsf{m}_u=\exp\big(\int_0^u\beta_r\,\mathrm dr\big)$ is a deterministic $\mathscr{C}^1$ function of $u$,
\( \mathrm d\mathsf{m}_u=\beta_u \mathsf{m}_u\,\mathrm du.\) Applying the It\^o product rule to the product $\mathsf{m}_u\widehat X_u$,  because $\mathsf{m}$ has finite variation (deterministic) the quadratic covariation $[\mathsf{m},\widehat X]\equiv0$, 
\[
\mathrm d(\mathsf{m}_u\widehat X_u)
= (\mathrm d\mathsf{m}_u)\,\widehat X_u + \mathsf{m}_u\,\mathrm d\widehat X_u.
\]
Substituting $\mathrm d\mathsf{m}_u=\beta_u \mathsf{m}_u\,\mathrm du$ and the SDE for $\widehat X_u$, we get,
\[
\begin{aligned}
\mathrm d(\mathsf{m}_u\widehat X_u)
&= \beta_u \mathsf{m}_u \widehat X_u\,\mathrm du
   + \mathsf{m}_u\big(-\beta_u\widehat X_u\,\mathrm du + \sqrt{2\beta_u}\,\mathrm dW_u\big) \\
&= \big(\beta_u \mathsf{m}_u\widehat X_u - \beta_u \mathsf{m}_u\widehat X_u\big)\mathrm du
   + \mathsf{m}_u\sqrt{2\beta_u}\,\mathrm dW_u \\
&= \mathsf{m}_u\sqrt{2\beta_u}\,\mathrm dW_u.
\end{aligned}
\]
 For $0\le s\le t$, integrating this identity, we obtain,
\(\mathsf{m}_t\widehat X_t - \mathsf{m}_s\widehat X_s = \int_s^t \mathsf{m}_u\sqrt{2\beta_u}\,\mathrm dW_u.\)
Since $\mathsf{m}_u=\exp\big(\int_0^u\beta_r\,\mathrm dr\big)$,
\[
\widehat X_t
= \frac{\mathsf{m}_s}{\mathsf{m}_t}\widehat X_s + \frac{1}{M_t}\int_s^t \mathsf{m}_u\sqrt{2\beta_u}\,\mathrm dW_u
= e^{-A(s,t)}\widehat X_s + \int_s^t e^{-\int_u^t\beta_r\,\mathrm dr}\sqrt{2\beta_u}\,\mathrm dW_u,
\]
because $\mathsf{m}_s/\mathsf{m}_t=\exp(-\int_s^t\beta_r\,\mathrm dr)=e^{-A(s,t)}$ and $\mathsf{m}_u/\mathsf{m}_t=e^{-\int_u^t\beta_r\,\mathrm dr}$. Conditioned on  $\widehat X_s$, the stochastic integral,
\(Z_{s,t}:=\int_s^t e^{-\int_u^t\beta_r\,\mathrm dr}\sqrt{2\beta_u}\,\mathrm dW_u\)
is independent of $\hX_s$, Gaussian and mean zero. Its covariance matrix is scalar times the identity because $W$ is standard and the integrand is scalar:
\[
\begin{aligned}
\operatorname{Cov}(Z_{s,t})
&= \mathbb E\big[Z_{s,t}Z_{s,t}^T\big]
= 2\int_s^t \beta_u e^{-2\int_u^t\beta_r\,\mathrm dr}\,\mathrm du\; \cdot I_D = 2 e^{-2\int_0^t\beta_r\,\mathrm dr}\int_s^t \beta_u e^{2\int_0^u\beta_r\,\mathrm dr}\,\mathrm du\; \cdot I_D.
\end{aligned}
\]
Let $B(u):=\int_0^u\beta_r\,\mathrm dr$. Then,
\[
\begin{aligned}
2 e^{-2B(t)}\int_s^t \beta_u e^{2B(u)}\,\mathrm du
&= e^{-2B(t)}\Big[ e^{2B(u)}\Big]_{u=s}^{u=t}
= e^{-2B(t)}\big(e^{2B(t)}-e^{2B(s)}\big) \\
&= 1 - e^{-2(B(t)-B(s))} = 1 - e^{-2A(s,t)}.
\end{aligned}
\]
Therefore, $\operatorname{Cov}(Z_{s,t})=(1-e^{-2A(s,t)})I_D$. Combining the mean $e^{-A(s,t)}\widehat X_s$ and the covariance above yields the conditional Gaussian law,
\[
\widehat X_t\mid\widehat X_s \sim \mathcal N\!\big(e^{-A(s,t)}\widehat X_s,\; (1-e^{-2A(s,t)})I_D\big),
\]
as claimed.
\end{proof}

\subsection{Solution to the Discretized Reverse Process through the Exponential Integrator Scheme}
\begin{lemma}\label{lem_a.1}
The SDE in \eqref{e22} is solved by taking,
\begin{align*}
    \hY_{t_{i+1}} = \hY_{t_i} + \left(e^{\int_{T-t_{i+1}}^{T-t_i} \beta_\tau \diff  \tau} - 1\right)\left(\hY_{t_i} + 2 \hat{s}(\hY_{t_i}, T-t_i)\right) + Z_{i} \sqrt{e^{2\int_{T-t_{i+1}}^{T-t_i} \beta_\tau \diff  \tau} - 1},
\end{align*}    
where $Z_i$'s are i.i.d. standard Gaussian random variables on $\Real^D$.
\end{lemma}
\begin{proof}
    We define $\mathsf{m}_t = e^{-\int_{t_i}^ t \beta_{T-\tau} \diff  \tau}$ and $V_t = \mathsf{m}_t Y_t$. Thus,
    \begingroup
    \allowdisplaybreaks
    \begin{align}
        \diff  V_t = & \ \diff  (\mathsf{m}_t Y_t) \nonumber\\
        = & \ \mathsf{m}_t \diff  Y_t + Y_t \diff  \mathsf{m}_t \nonumber\\
        = & \ \mathsf{m}_t \left(\beta_{T-t} (\hat{Y}_t + 2 \hat{s}(\hY_{t_i}, T-t_i)) \diff  t + \sqrt{2 \beta_{T-t}} \diff  W_t \right) - \beta_{T-t} \mathsf{m}_t Y_t \diff  t \nonumber\\
        = & \ 2 \mathsf{m}_t \beta_{T-t}   \hat{s}(\hY_{t_i}, T-t_i) \diff  t + \mathsf{m}_t \sqrt{2 \beta_{T-t}} \diff  W_t. 
    \end{align}
    \endgroup

Thus,
\begingroup
\allowdisplaybreaks
\begin{align*}
    V_t = & V_{t_{i}} +  2 \int_{t_i}^t m_\tau \beta_{T-\tau}   \hat{s}(\hY_{t_i}, T-t_i) \diff \tau + \int_{t_i}^t m_\tau \sqrt{2 \beta_{T-\tau}} \diff  W_\tau \\
    = & Y_{t_i} + 2 \hat{s}(\hY_{t_i}, T-t_i)  \int_{t_i}^t m_\tau \beta_{T-\tau}   \diff \tau + \int_{t_i}^t m_\tau \sqrt{2 \beta_{T-\tau}} \diff  W_\tau \\
    \implies & V_{t_{i+1}} =   Y_{t_i} + 2 \hat{s}(\hY_{t_i}, T-t_i)  \int_{t_i}^{t_{i+1}} m_\tau \beta_{T-\tau}   \diff \tau + \int_{t_i}^{t_{i+1}} m_\tau \sqrt{2 \beta_{T-\tau}} \diff  W_\tau \\
   \implies & e^{-\int_{t_i}^{t_{i+1}} \beta_{T-\tau} \diff \tau} Y_{t_{i+1}} \\ = & Y_{t_i} + 2 \hat{s}(\hY_{t_i}, T-t_i) \left(1 - e^{-\int_{t_i}^{t_{i+1}} \beta_{T-\tau} \diff \tau}\right) + Z_i \sqrt{1 - e^{-2 \int_{t_i}^{t_{i+1}} \beta_{T-\tau} \diff \tau} },
\end{align*}
\endgroup
where $Z_i \sim \gamma_D$. The result follows from multiplying both sides with $e^{\int_{t_i}^{t_{i+1}} \beta_{T-\tau} \diff \tau} $.
\end{proof}
\subsection{Equivalence of the Sample Score-matching Objectives}
\begin{lemma}\label{stein}
    The score-matching objectives in \eqref{sm_1} and \eqref{sm1} admit the same minimizers. Furthermore,
    \begin{align}
        & \E  \| s(\hX_{t_i}, t_i) - \nabla \log \hp_{t_i}(\hX_{t_i})\|^2 \nonumber\\
        = & \E \left\| s\left(e^{-\int_0^{t_i} \beta_\tau \diff  \tau } \hX_0 + \sqrt{1-e^{- \int_0^{t_i} \beta_\tau \diff  \tau}} Z_{t_i} , t_i\right) + \frac{Z_{t_i}}{\sqrt{1-e^{- 2\int_0^{t_i} \beta_\tau \diff  \tau}}}  \right\|^2  + \E\| \nabla \log \hp_{t_i}(\hX_{t_i})\|^2 - \frac{D}{\sigma_{t_i}^2} \label{e30},
    \end{align}
    where $Z_{t_i} \sim \gamma_D$ and is independent of $\hX_0$. Here, $\sigma_t^2 = 1-e^{- 2\int_0^{t_i} \beta_\tau \diff  \tau}$.
\end{lemma}
\begin{proof}
We note that,
\begingroup
\allowdisplaybreaks
\begin{align}
   & \E  \| s(\hX_{t_i}, t_i) - \nabla \log \hp_{t_i}(\hX_{t_i})\|^2 \nonumber\\
    = & \E  \| s(\hX_{t_i}, t_i)\|^2  + \E\| \nabla \log \hp_{t_i}(\hX_{t_i})\|^2 - 2 \E \left\langle s(\hX_{t_i}, t_i), \nabla \log \hp_{t_i}(\hX_{t_i}) \right\rangle \nonumber \\
    = & \E  \| s(\hX_{t_i}, t_i)\|^2  + \E\| \nabla \log \hp_{t_i}(\hX_{t_i})\|^2 - 2 \int \left\langle s(x, t_i), \nabla \log \hp_{t_i}(x) \right\rangle \hat{p}_{t_i}(x) \diff x \nonumber \\
     = & \E  \| s(\hX_{t_i}, t_i)\|^2  + \E\| \nabla \log \hp_{t_i}(\hX_{t_i})\|^2 - 2 \int \left\langle s(x, t_i), \nabla  \hp_{t_i}(x) \right\rangle  \diff x \nonumber\\
     = & \E  \| s(\hX_{t_i}, t_i)\|^2  + \E\| \nabla \log \hp_{t_i}(\hX_{t_i})\|^2 + 2 \int  (\nabla \cdot s)(x, t_i)   \hp_{t_i}(x)   \diff x \label{gip}\\
     = & \E  \| s(\hX_{t_i}, t_i)\|^2  + \E\| \nabla \log \hp_{t_i}(\hX_{t_i})\|^2 \nonumber\\
     & + 2 \int  (\nabla \cdot s)\left(e^{-\int_0^{t_i} \beta_\tau \diff  \tau } \hx_0 + \sqrt{1-e^{- 2\int_0^{t_i} \beta_\tau \diff  \tau}} z_{t_i} , t_i\right)   \diff \hmu(\hx_0)    \diff \gamma_D(z_{t_i}) \nonumber\\
     = & \E  \| s(\hX_{t_i}, t_i)\|^2  + \E\| \nabla \log \hp_{t_i}(\hX_{t_i})\|^2 \nonumber\\
     & + \frac{2}{\sigma_{t_i}^2} \int   \left\langle s\left(e^{-\int_0^{t_i} \beta_\tau \diff  \tau } \hx_0 + \sigma_{t_i} z_{t_i} , t_i\right), \sigma_{t_i} z_{t_i} \right\rangle  \diff \hmu(\hx_0)    \diff \gamma_D(z_{t_i}) \nonumber\\
     = & \E  \| s(\hX_{t_i}, t_i)\|^2  + \E\| \nabla \log \hp_{t_i}(\hX_{t_i})\|^2 \nonumber\\
     & + \frac{2}{\sigma_{t_i}} \int   \left\langle s\left(e^{-\int_0^{t_i} \beta_\tau \diff  \tau } \hx_0 + \sigma_{t_i} z_{t_i} , t_i\right), z_{t_i} \right\rangle   \diff \hmu(\hx_0)    \diff \gamma_D(z_{t_i}) \nonumber\\
     = & \E \left\| s\left(e^{-\int_0^{t_i} \beta_\tau \diff  \tau } \hX_0 + \sqrt{1-e^{- 2\int_0^{t_i} \beta_\tau \diff  \tau}} Z_{t_i} , t_i\right) + \frac{Z_{t_i}}{\sqrt{1-e^{- 2\int_0^{t_i} \beta_\tau \diff  \tau}}}  \right\|^2 \nonumber \\
     & + \E\| \nabla \log \hp_{t_i}(\hX_{t_i})\|^2 - \E \left\| \frac{Z_{t_i}}{\sqrt{1-e^{-2 \int_0^{t_i} \beta_\tau \diff  \tau}}}  \right\|^2 \nonumber\\
     = & \E \left\| s\left(e^{-\int_0^{t_i} \beta_\tau \diff  \tau } \hX_0 + \sqrt{1-e^{- 2\int_0^{t_i} \beta_\tau \diff  \tau}} Z_{t_i} , t_i\right) + \frac{Z_{t_i}}{\sqrt{1-e^{- 2\int_0^{t_i} \beta_\tau \diff  \tau}}}  \right\|^2 \nonumber\\
     & + \E\| \nabla \log \hp_{t_i}(\hX_{t_i})\|^2 - \frac{D}{1-e^{- 2\int_0^{t_i} \beta_\tau \diff  \tau}}  . \nonumber
\end{align}
\endgroup
In the above calculations, \eqref{gip} follows the Gaussian integral by parts \citep[Lemma 7.2.4]{Vershynin_2026}.
\end{proof}
\section{Relations between $d^\star_{p,q}$ and Other Notions of Intrinsic Dimension}
\propone*
 \begin{proof}
     \textbf{Proof of part (a)}: Let,  \[\sA_{p,q}=  \left\{s \in (2p, \infty) : \limsup_{\epsilon \downarrow 0} \frac{\log \sN_{\epsilon}\left(\mu, \epsilon^{\frac{s p q}{(q-p)(s-2p)}}\right)}{\log(1/\epsilon)} \leq s \right\},\] 
     and \[\sB_{p}=  \left\{s \in (2p, \infty) : \limsup_{\epsilon \downarrow 0} \frac{\log \sN_{\epsilon}\left(\mu, \epsilon^{\frac{s p }{s-2p}}\right)}{\log(1/\epsilon)} \leq s \right\}.\] Fix $s \in \sA$. Then,  $\limsup_{\epsilon \downarrow 0} \frac{\log \sN_{\epsilon}\left(\mu, \epsilon^{\frac{s p q}{(q-p)(s-2p)}}\right)}{\log(1/\epsilon)} \leq s$. Since, $q/(q-p) > 1$, 
     \[\log \sN_{\epsilon}\left(\mu, \epsilon^{\frac{s p q}{(q-p)(s-2p)}}\right) \ge \sN_{\epsilon}\left(\mu, \epsilon^{\frac{s p q}{(q-p)(s-2p)}}\right).\]
     Hence, \[\limsup_{\epsilon \downarrow 0} \frac{\log \sN_{\epsilon}\left(\mu, \epsilon^{\frac{s p }{s-2p}}\right)}{\log(1/\epsilon)} \le \limsup_{\epsilon \downarrow 0} \frac{\log \sN_{\epsilon}\left(\mu, \epsilon^{\frac{s p q}{(q-p)(s-2p)}}\right)}{\log(1/\epsilon)}  \leq s. \]
     This implies that $s \in \sB_p$. Hence, $\sA_{p,q} \subseteq \sB_p$, which implies that $d^\ast_p(\mu) = \inf \sB_p \le \inf \sA_{p,q} = d^\star_{p,q}(\mu)$. 

\textbf{Proof of part (b)}: If $q_1 \le q_2$, then, $\frac{s p q_1}{(q_1-p)(s - 2 p)} \ge \frac{s p q_2}{(q_2-p)(s - 2 p)}$, which implies that $\epsilon^{\frac{s p q_1}{(q_1-p)(s - 2 p)}} \le \epsilon^{\frac{s p q_2}{(q_2-p)(s - 2 p)}}$, for all $0 < \epsilon < 1$. Hence, $\sN_{\epsilon}\left(\mu, \epsilon^{\frac{s p q_1}{(q_1-p)(s-2p)}}\right) \ge \sN_{\epsilon}\left(\mu, \epsilon^{\frac{s p q_2}{(q_2-p)(s-2p)}}\right)$. Thus, if $s \in \sA_{p,q_1}$, 
\[\limsup_{\epsilon \downarrow 0} \frac{\log \sN_{\epsilon}\left(\mu, \epsilon^{\frac{s p q_2}{(q_2-p)(s-2p)}}\right)}{\log(1/\epsilon)} \le  \limsup_{\epsilon \downarrow 0} \frac{\log \sN_{\epsilon}\left(\mu, \epsilon^{\frac{s p q_1}{(q_1-p)(s-2p)}}\right)}{\log(1/\epsilon)} \leq s, \]
which implies that $s \in \sA_{p, q_2}$. Hence, $\sA_{p, q_1} \subseteq \sA_{p, q_2} $. The result follows by taking the infimum on both sides. 

\textbf{Proof of part (c)}:
Suppose that $f(p) = \frac{spq}{\left(q-p\right)\left(s-2p\right)}.$ Clearly, $f^\prime(p) = \frac{sq(sq - 2 p^2)}{(q-p)^2 (s-2p)^2} > 0$, if $s > 2p $ and $p <q $. Thus, if $p_1 \le p_2$ and $s > 2p_2$, $\frac{s p_1 q }{(q-p_1)(s - 2 p_1)} \le \frac{s p_1 q}{(q-p_1)(s - 2 p_1)}$, which implies that $\epsilon^{\frac{s p_1 q}{(q-p_1)(s - 2 p_1)}} \ge \epsilon^{\frac{s p_2 q}{(q-p_2)(s - 2 p_2)}}$, for all $0 < \epsilon < 1$. Thus, if $s \in \sA_{p_2, q}$, 
\[\limsup_{\epsilon \downarrow 0} \frac{\log \sN_{\epsilon}\left(\mu, \epsilon^{\frac{s p_1 q}{(q-p_1)(s-2p_1)}}\right)}{\log(1/\epsilon)} \le \limsup_{\epsilon \downarrow 0} \frac{\log \sN_{\epsilon}\left(\mu, \epsilon^{\frac{s p_2 q}{(q-p_2)(s-2p_2)}}\right)}{\log(1/\epsilon)} \le s.\]
The above equation, combined with the fact that $s > 2 p_2 \ge 2 p_1$ implies that $s \in \sA_{p_1, q}$. Thus, $\sA_{p_1, q} \supseteq \sA_{p_2, q} \implies \inf \sA_{p_1, q} \le  \inf \sA_{p_2, q} \implies d^\star_{p_1, q}(\mu) \le d^\star_{p_2,q} (\mu)$.

     \textbf{Proof of part (d)}: This part of the proposition easily follows from the fact that, $\sN_{\epsilon}\left(\mu, \epsilon^{\frac{s p q}{(q-p)(s-2p)}}\right) \le \cN(\epsilon; \operatorname{supp}(\mu), \ell_\infty)$. Thus, for any $s \ge \text{dim}_M(\mu)$, 
     \[\limsup_{\epsilon \downarrow 0} \frac{\log \sN_{\epsilon}\left(\mu, \epsilon^{\frac{s p q}{(q-p)(s-2p)}}\right)}{\log(1/\epsilon)}  \le \limsup_{\epsilon \downarrow 0} \frac{\log \cN(\epsilon; \operatorname{supp}(\mu), \ell_\infty) }{\log(1/\epsilon)}  = \text{dim}_M(\mu) \le s.\]
     Hence, $[\text{dim}_M(\mu), \infty) \subseteq \sA_{p,q} \implies \text{dim}_M(\mu) \ge \inf \sA_{p,q} = d^\star_{p,q}(\mu)$.

     \textbf{Proof of part (e)}: Let $0 < \epsilon < 1$, $s > \overline{\text{dim}}_P(\mu)$  
 and $\tau = \epsilon^{\frac{s p q}{(q-p)(s - 2 p)}}$. $S$ be such that $\mu(S) \ge 1- \tau$ and $\cN(\epsilon;S, \varrho) = \sN_\epsilon(\mu, \tau)$. We let $R = \text{diam}(S) \vee 1$. Let $\{x_1, \dots, x_M\}$ be an optimal $2\epsilon$-packing of $S \cap \text{supp}(\mu)$. By the definition of the upper packing dimension, for any $s > \overline{\text{dim}}_{P}(\mu) $ we can find $r_0 < 1$, such that, 
 \begin{align*}
     & \frac{\log \mu(B(x,r))}{\log r} \le s, \, \forall r \le r_0 \text{ and } x \in \text{supp}(\mu) \\
     \implies &  \mu(B(x,r)) \ge r^s, \, \forall r \le r_0 \text{ and } x \in \text{supp}(\mu).
 \end{align*}
 
This implies that, 
Thus, if $\epsilon \le r_0$, \(1 \ge \mu\left(\cup_{i=1}^M B(x_i, \epsilon)\right) = \sum_{i=1}^M \mu\left( B(x_i, \epsilon)\right) \ge M \epsilon^s \implies M \le \epsilon^{-s}.\) By Lemma \ref{cov_pack}, we know that $\sN_\epsilon(\mu, \tau) =\cN(\epsilon; S, \varrho) \le M \le  \epsilon^{-s}$. Thus,
\[ \limsup_{\epsilon \downarrow 0} \frac{ \log \sN_\epsilon\left(\mu, \epsilon^{\frac{s pq}{(q-p)(s-2 p)}}\right)}{-\log \epsilon}  \le s \implies s \in \sA_{p,q} \implies d^\star_{p,q}(\mu) \le s.\]
Since $d^\star_{p,q}(\mu) \le s$, for all $s > \overline{\text{dim}}_P(\mu)$, we get, $d^\star_{p,q}(\mu)  \le \overline{\text{dim}}_P(\mu)$. The inequality $\overline{\text{dim}}_P(\mu) \le \overline{\text{dim}}_{\text{reg}}(\mu)$ follows from \citet[Theorem 2.1]{fraser2017upper}.

\textbf{Proof of part (f)}: The result easily follows by observing that $d^\star_{p,q}(\mu) \ge d^\ast_p(\mu) \ge d_\ast(\mu)$, where the second inequality follows from \citet[Proposition 2]{weed2019sharp} and $d_\ast(\mu) \ge \underline{dim}_{\text{reg}}(\mu)$, which follows from \citet[Proposition 8]{JMLR:v26:24-0054}.
 \end{proof}
 
\section{Omitted Proofs from Sections \ref{sec_assumptions} and \ref{main results}}
\subsection{Existence of the Reverse Process}
\proptwo*
\begin{proof}
    Note that the forward SDE \eqref{forward} is of the form, 
    \[\diff X_t = b(t, X_t) \diff t + \sigma(t, X_t) \diff W_t, \]
    with $b(t, x) = -\beta_t x$ and $\sigma(t, X_t) = \sqrt{2 \beta_t}$.  Clearly, 
    \begin{align*}
        |b(t, x) - b(t, y)| + |\sigma(t,x) - \sigma(t,y)| = \beta_t |x-y| \le \overline{\beta} |x-y|.
    \end{align*}
    Further, $|b(t,x)| + |\sigma(t,x)| = \beta_t |x| + \sqrt{2\beta_t} \le \overline{\beta} |x| + \sqrt{2\overline{\beta}} \le \max\{\overline{\beta}, \sqrt{2 \overline{\beta}} \}(|x| +1)$. Thus, (H1) of \citep{10.1214/aop/1176991505} is satisfied with $K =  \max\{\overline{\beta}, \sqrt{2 \overline{\beta}} \} $. Further both $b, \sigma \in \mathscr{C}^2$ in $x$ and $a = \sigma^2 \ge 2 \underline{\beta}$. Hence (Hiii) of \citep{10.1214/aop/1176991505} is satisfied. Thus, by Proposition 4.2 of \citep{10.1214/aop/1176991505}, Assumptions (Ai) and (Aii) of \citep{10.1214/aop/1176991505} are satisfied. Hence, by \citet[Theorem~2.3]{10.1214/aop/1176991505}, the reverse process satisfies \eqref{eq_rp}.
\end{proof}

\subsection{Scaling of $\mathsf{M}_{n,x}$}
 \begin{lemma}\label{scaling}
    With probability at least $1-\delta$, $\mathsf{M}_{n,x} \le n^{1/q} \sM_q(\mu) \delta^{-1/q} \vee 1$. Further, $\E \mathsf{M}_{n,x} \le \left(\sqrt{D} n^{1/q} \sM_q(\mu)\right) \vee 1$.
\end{lemma}
\begin{proof}
    We note that, for $t \ge 1$,
    \begin{align*}
         \prob(\mathsf{M}_{n,x} -1 > t) 
        =  \prob\left(\cup_{i=1}^n \|X_i\|_\infty > t\right)
        \le  n \prob\left(\|X_1\|_\infty > t\right)
        \le & n \prob\left(\|X_1\|_2 > t\right)\\
        \le & n \frac{\E \|X_1\|^q}{t^q}\\
        = & \frac{n \sM_q^q(\mu)}{t^q}.
    \end{align*}
    Choosing $t = n^{1/q} \sM_q(\mu) \delta^{-1/q} \vee 1$ proves the the first part of the Lemma. Moreover,
    \begin{align*}
        \E \max_{1 \le i \le n} \|X_i\|_\infty \le  \sqrt{D} \E \max_{1 \le i \le n} \|X_i\| 
        \le & \sqrt{D} \E \left(\sum_{i=1}^n \|X_i\|^q \right)^{1/q}\\
        \le & \sqrt{D} \left(\sum_{i=1}^n \E  \|X_i\|^q \right)^{1/q} \\
        = & \sqrt{D} n^{1/q} \sM_q(\mu).
    \end{align*}
\end{proof}

\section{Proof of the Main Result}
\subsection{Proof of Theorem~\ref{mainthm}}
\mainthm*
\begin{proof}
For notational simplicity, let, $M = \max_{1 \le i \le n}\|X_i\|_\infty$. Fix $d > d^\star_{p,q}(\mu)$. From Lemma~\ref{lem_oracle_2}, we note that, 
\begingroup
\allowdisplaybreaks
\begin{align}
    & \E \fW_p(\mu, \sT_R(\hQ_{T-\delta_0}(s))) \nonumber\\
    \le & \E \fW_p(\mu, \hmu_n) + R \sqrt{D} 2^{\frac{p-1}{2p}} \E \kl(\hP_T, \gamma_D)^{\frac{1}{2p}} \nonumber\\
    & + R (2\overline{\beta}^2 )^{\frac{1}{2p}} \left(\sum_{i=0}^{N-1} h_i \E \|s(\hX_{T-t_i}, T-t_i) - \nabla \log \hp_{T-t_i}(\hX_{T-t_i})\|^2 \right)^{1/2p} \nonumber\\
        & + R (2\overline{\beta}^2 )^{\frac{1}{2p}} \left(\sum_{i=0}^{N-1}\int_{t_i}^{t_{i+1}} \E\|\nabla \log \hp_{T-t_i}(\hX_{T-t_i}) - \nabla \log \hp_{T-t}(\hX_{T-t})\|^2 \diff  t\right)^{1/2p} \nonumber\\
        & +  \overline{\beta}\delta_0 \E \sM_p(\hmu_n) + (\overline{\beta}\delta_0)^{1/2} \sM_p(\gamma_D)  + 2^{(q-1)/p} \left( \E \sM_p^q(\hmu_n) +   \sM_q^q(\gamma_D) \right)^{1/p} R^{-(q-p)/p} \nonumber\\
        \le & \E \fW_p(\mu, \hmu_n) + R \sqrt{D} 2^{\frac{p-1}{2p}} \E \left(\exp\left(-2\underline{\beta}T\right) (D + \sM_2^2(\hmu_n))\right)^{\frac{1}{2p}} \label{ea42}\\
        & + R (2\overline{\beta}^2 )^{\frac{1}{2p}} \left(\sum_{i=0}^{N-1} h_i \E \|s(\hX_{T-t_i}, T-t_i) - \nabla \log \hp_{T-t_i}(\hX_{T-t_i})\|^2 \right)^{1/2p} \nonumber\\
        & + R (2\overline{\beta}^2 )^{\frac{1}{2p}} \left(\sum_{i=0}^{N-1}\int_{t_i}^{t_{i+1}} \E\|\nabla \log \hp_{T-t_i}(\hX_{T-t_i}) - \nabla \log \hp_{T-t}(\hX_{T-t})\|^2 \diff  t\right)^{1/2p} \nonumber\\
        & +  \overline{\beta}\delta_0  \sM_p(\mu) + (\overline{\beta}\delta_0)^{1/2} \sM_p(\gamma_D) + 2^{(q-1)/p} \left( \sM_q^q(\mu) +   \sM_q^q(\gamma_D) \right)^{1/p} R^{-(q-p)/p} \nonumber\\
        \le & \E \fW_p(\mu, \hmu_n) + R \sqrt{D} 2^{\frac{p-1}{2p}} \exp\left(-\underline{\beta}T/p\right) (D + \sM_2^2(\mu))^{\frac{1}{2p}} \label{ea4}\\
        & + R (2\overline{\beta}^2 )^{\frac{1}{2p}} \left(\sum_{i=0}^{N-1} h_i \E \|s(\hX_{T-t_i}, T-t_i) - \nabla \log \hp_{T-t_i}(\hX_{T-t_i})\|^2 \right)^{1/2p} \nonumber\\
        & + R (2\overline{\beta}^2 )^{\frac{1}{2p}} \left(\sum_{i=0}^{N-1}\int_{t_i}^{t_{i+1}} \E\|\nabla \log \hp_{T-t_i}(\hX_{T-t_i}) - \nabla \log \hp_{T-t}(\hX_{T-t})\|^2 \diff  t\right)^{1/2p} \nonumber\\
        & +  (\overline{\beta}\delta_0)  \sM_p^p(\mu) + (\overline{\beta}\delta_0)^{1/2} \sM_p(\gamma_D) + 2^{(q-1)/p} \left( \sM_q^q(\mu) +   \sM_q^q(\gamma_D) \right)^{1/p} R^{-(q-p)/p}. \label{e32}
\end{align}
\endgroup
In the above calculations, Equation \eqref{ea42} follows from Lemma~\ref{kl_bd} and Equation \eqref{ea4} follows from Jensen's inequality. We take $\kappa \asymp  n^{-\frac{2 (1+p(q-p))}{d(q-p)}} $. From Lemma~\ref{lem5}, this choice of $\kappa$ ensures that $N \lesssim n^{\frac{2 (1+p(q-p))}{d(q-p)}} \log n$ and \[\sum_{i=0}^{N-1} \frac{(t_{i+1} - t_i)^2}{\sigma^4_{t_i}}  \lesssim n^{-\frac{2 (1+p(q-p))}{d(q-p)}} \log n.\] Hence, from Lemma~\ref{lem4},
\begin{equation}
    \sum_{i=0}^{N-1}\int_{t_i}^{t_{i+1}} \beta^2_{T-t} \E \|\nabla \log \hp_{T-t_i}(X_{T-t_i}) - \nabla \log \hp_{T-t}(X_{T-t})\|^2 \diff  t  \lesssim n^{-\frac{2 (1+p(q-p))}{d(q-p)}} \log n. \label{e31}
\end{equation}
We first prove the inequality for $\hat{s}_n$, i.e., Equation \eqref{s_n}.
To show this, we note that, from Theorem~\ref{approx_thm}, if $\sS = \cR\cN( L, W , B)$, with $W \lesssim \mathsf{M}_{n,x}^{D} n^{\frac{D+2 + (Dp + D + 2p)}{d(q-p)}} \left(\log^2 n + \log (M \vee 1) \log n\right)$, $L \lesssim \log (1/\epsilon) \asymp \log n + \log \mathsf{M}_{n,x}$ and $B\lesssim n^{\frac{1+(p+3)(q-p)}{d(q-p)}}  (\|X_i\|_\infty + 1)$,
\begin{align}
     \sum_{i=0}^{N-1} h_i  \|\hat{s}_n(\cdot, T-t_i) - \nabla \log \hp_{T-t_i}(\cdot)\|^2_{\fL_2(\hP_{T-t_i})}
    = & \inf_{s\in \sS} \sum_{i=0}^{N-1} h_i  \|s(\cdot, T-t_i) - \nabla \log \hp_{T-t_i}(\cdot)\|^2_{\fL_2(\hP_{T-t_i})} \nonumber\\
    \lesssim & \ n^{-\frac{2(1+p(q-p))}{d(q-p)}} \log n. \label{ea5}
\end{align}
Finally, we take $\delta_0 \asymp n^{-2/d}$, and
  \(R = n^{\frac{1}{dp(q-p)}} 2^{\frac{q-1}{q-p}}\left( \sM_q^q(\mu) +   \sM_q^q(\gamma_D)\right)^{\frac{1}{q-p}}\), and 
\begin{align*}
    T \ge & \frac{p}{\underline{\beta}}\left( \log R + \frac{1}{d}\log n + \frac{1}{2} \log D + \frac{1}{2p} \log(D + \sM_2^2(\mu))\right) \\
    = & \frac{p}{\underline{\beta}}\bigg( \frac{1 + p(q-p)}{dp(q-p)}\log n + \frac{1}{2} \log D + \frac{1}{q-p} \log( \sM_q^q(\mu) + \sM_q^q(\gamma_D)) \\
    & \quad + \frac{1}{2p} \log(D + \sM_2^2(\mu)) + \frac{q-1}{q-p} \log 2 \bigg).
\end{align*}
 Plugging in the values of $T$, $R$, $\delta_0$ and the bounds \eqref{e31} and \eqref{ea5} in Equation \eqref{e32}, we observe that,
\begin{align*}
   \E \fW_p(\mu, \sT_R(\hQ_{T-\delta_0}(s))) \lesssim & \
   \E \fW_p(\mu, \hmu_n) + n^{-1/d} \operatorname{poly-log}(n)\\
   \lesssim & \ n^{-1/d} \operatorname{poly-log}(n),
\end{align*}
where the final inequality follows from Corollary~\ref{cor4}. 

To show \eqref{s_mc}, from Lemma~\ref{lem_21}, if  \[m_i \ge C \frac{\sigma_{t_i}^4}{h_i^2} n^{\frac{D+2 + (Dp + D + 2p)}{d(q-p)} + \frac{4(1+p(q-p))}{d(q-p)}} \operatorname{poly}(\log n , \max_{i \in [n]}\|X_i\|_\infty \vee 1),\] 
then,
\begin{align*}
     \sum_{i=0}^{N-1} h_i  \E \| \hat{s}^{\mathrm{mc}}_n(\cdot, t_i) - \nabla \log \hp_{t_i}(\cdot)\|^2_{\fL_2(\hP_{t_i})} 
    \lesssim & \ \inf_{s \in \sS}\sum_{i=0}^{N-1} h_i  \| s(\cdot, t_i) - \nabla \log \hp_{t_i}(\cdot)\|^2_{\fL_2(\hP_{t_i})} + n^{-\frac{2(1+p(q-p))}{d(q-p)}}\\
    \lesssim & \ n^{-\frac{2(1+p(q-p))}{d(q-p)}} \log n + n^{-\frac{2(1+p(q-p))}{d(q-p)}}\\
    \lesssim & \ n^{-\frac{2(1+p(q-p))}{d(q-p)}}.
\end{align*}
With the above choices of $R$, $\delta_0$ and $T$, \eqref{s_mc} follows from a similar calculation.
\end{proof}

\subsection{Generalization Bounds for the Sample Score-matching Objective}
This section proves generalization bounds for the sample score matching objective \eqref{sm2}.
\begin{lemma}\label{lem_21}
Suppose that $s_0 \in \sS$. There exists a constant $C$ an $n _0$, such that if  \[m_i \ge C \frac{\sigma_{t_i}^4}{h_i^2} n^{\frac{D+2 + (Dp + D + 2p)}{d(q-p)} + \frac{4(1+p(q-p))}{d(q-p)}} \operatorname{poly}(\log n , \max_{1 \le i \le n}\|X_i\|_\infty \vee 1),\] and $n \ge n_0$, then, 
    \[\sum_{i=0}^{N-1} h_i  \E \| \hat{s}^{\mathrm{mc}}_n(\cdot, t_i) - \nabla \log \hp_{t_i}(\cdot)\|^2_{\fL_2(\hP_{t_i})} \lesssim \sum_{i=0}^{N-1} h_i  \| s_0(\cdot, t_i) - \nabla \log \hp_{t_i}(\cdot)\|^2_{\fL_2(\hP_{t_i})} + n^{-\frac{2(1+p(q-p))}{d(q-p)}}.  \]
\end{lemma}
\begin{proof}
For notational simplicity, let, $M = \max_{1 \le i \le n}\|X_i\|_\infty$. Take $\epsilon = n^{-\frac{1+p(q-p)}{d(q-p)}}$ and $\delta \asymp N^{-1}(M
\vee 1)^{-1}\delta_0^{-1} \epsilon^2$. 
We choose \[m_i \ge (M \vee 1)^2
\frac{\sigma_{t_i}^4 \log(2 N \left( 1+ \cN\left(\epsilon; \sS, \ell_\infty\right) \right)/\delta)}{c h_i^2 } \cdot \frac{\epsilon^2}{N}. \]
Thus, it is enough to choose 
\begin{align}
    m_i \ge & \frac{\sigma_{t_i}^4 \mathsf{M}_{n,x}^2}{c h_i^2 } \left(\log N + \log\left(\frac{2}{\delta}\right) + W \log \left( 2LB^L(W + 1)^L n^{-\frac{(1+p(q-p))}{d(q-p)}} \right) \right) \times n^{\frac{4p(1+q-p)}{d(q-p)}} \label{ea14}\\
    \asymp & \frac{\sigma_{t_i}^4}{h_i^2} n^{\frac{D+2 + (Dp + D + 2p)}{d(q-p)} + \frac{4(1+p(q-p))}{d(q-p)}} \operatorname{poly}(\log n , \mathsf{M}_{n,x}).
\end{align}
Here, we use Lemma~\ref{lem_nakada} in Equation \eqref{ea14}. Hence from Lemma~\ref{lem_22}, with probability at least $1-\delta$, 
\begin{align}
     \sum_{i=0}^{N-1} h_i  \| \hat{s}^{\mathrm{mc}}_n(\cdot, t_i) - \nabla \log \hp_{t_i}(\cdot)\|^2_{\fL_2(\hP_{t_i})} 
    \lesssim  \sum_{i=0}^{N-1} h_i  \| s_0(\cdot, t_i) - \nabla \log \hp_{t_i}(\cdot)\|^2_{\fL_2(\hP_{t_i})} + n^{-\frac{2(1+p(q-p))}{d(q-p)}}. \label{e331}
\end{align}
Suppose that $a_1$ be the constant that honors the inequality in \eqref{e331}, i.e., 
\begin{align}
      \sum_{i=0}^{N-1} h_i  \| \hat{s}^{\mathrm{mc}}_n(\cdot, t_i) - \nabla \log \hp_{t_i}(\cdot)\|^2_{\fL_2(\hP_{t_i})}
    \le   a_1 \sum_{i=0}^{N-1} h_i  \| s_0(\cdot, t_i) - \nabla \log \hp_{t_i}(\cdot)\|^2_{\fL_2(\hP_{t_i})} + a_1 n^{-\frac{2(1+p(q-p))}{d(q-p)}}. \label{e33}
\end{align}
Let, $\cE$ denote the event that \eqref{e33} holds. This implies that 
\begingroup
\allowdisplaybreaks
\begin{align}
   & \sum_{i=0}^{N-1} h_i  \E \| \hat{s}^{\mathrm{mc}}_n(\cdot, t_i) - \nabla \log \hp_{t_i}(\cdot)\|^2_{\fL_2(\hP_{t_i})} \nonumber\\
   = & \E \Bigg[\left(\sum_{i=0}^{N-1} h_i  \| \hat{s}^{\mathrm{mc}}_n(\cdot, t_i) - \nabla \log \hp_{t_i}(\cdot)\|^2_{\fL_2(\hP_{t_i})}\right)  \cdot \one\{\cE\} \Bigg]  \nonumber\\
   & + \E \Bigg[\left(\sum_{i=0}^{N-1} h_i  \| \hat{s}^{\mathrm{mc}}_n(\cdot, t_i) - \nabla \log \hp_{t_i}(\cdot)\|^2_{\fL_2(\hP_{t_i})}\right) \cdot \one\{\cE^\mathsf{C}\} \Bigg] \nonumber\\
   \le & a_1 \sum_{i=0}^{N-1} h_i  \| s_0(\cdot, t_i) - \nabla \log \hp_{t_i}(\cdot)\|^2_{\fL_2(\hP_{t_i})} + a_1 n^{-\frac{2(1+p(q-p))}{d(q-p)}} + \sum_{i=0}^{N-1}\frac{\mathsf{m}_t M }{\sigma_{t_i^2}} \delta \nonumber\\
   \lesssim & \sum_{i=0}^{N-1} h_i  \| s_0(\cdot, t_i) - \nabla \log \hp_{t_i}(\cdot)\|^2_{\fL_2(\hP_{t_i})} + n^{-\frac{2(1+p(q-p))}{d(q-p)}}. 
\end{align}
\endgroup
\end{proof}

\begin{lemma}\label{lem_22}
    Suppose that $0 < \delta, \epsilon < 1$ and $s_0 \in \sS$. Then, if $m_i\ge \frac{1}{c}\log(2/\delta)$, with probability at least $1-\delta$, 
    \begin{align}
    & \sum_{i=0}^{N-1} h_i  \| \hat{s}^{\mathrm{mc}}_n(\cdot, t_i) - \nabla \log \hp_{t_i}(\cdot)\|^2_{\fL_2(\hP_{t_i})} \nonumber \\
    \lesssim &  \sum_{i=0}^{N-1} h_i  \| s_0(\cdot, t_i) - \nabla \log \hp_{t_i}(\cdot)\|^2_{\fL_2(\hP_{t_i})} + M \sum_{i=0}^{N-1} \frac{h_i}{\sigma_{t_i}^2}\sqrt{\frac{\log(N \left( 1+ \cN\left(\epsilon; \sS, \ell_\infty\right) \right)/\delta)}{m_i}}
     + \epsilon^2.
\end{align}
Here $c$ is the same constant as in Lemma~\ref{lemi3}.
\end{lemma}
\begin{proof}
    For notational simplicity, we write $\hX_{t_i}^{(j)} =  e^{-\int_0^{t_i} \beta_\tau \diff  \tau } X_{t_i}^{(j)} + \sqrt{1-e^{- \int_0^{t_i} \beta_\tau \diff  \tau}} Z_{t_i}^{(j)}$. 
Fix $i \in [N]$. We note that,
\begin{align}
   & \frac{1}{ m}  \sum_{i=0}^{N-1}  \sum_{j=1}^m h_i \left\| s\left( \hX_{t_i}^{(j)}, t_i\right) + \frac{Z_{t_i}^{(j)}}{\sigma_{t_i}}  \right\|^2 - \sum_{j=1}^m h_i \E \|\hat{s}^{\mathrm{mc}}_n(\hX_{t_i}) + Z_{t_i}/\sigma_{t_i}\|^2\\
   = & \frac{1}{ m}  \sum_{i=0}^{N-1}  \sum_{j=1}^m h_i \left(\left\| s\left( \hX_{t_i}^{(j)}, t_i\right) + \frac{Z_{t_i}^{(j)}}{\sigma_{t_i}}  \right\|^2 - \E \left\|\hat{s}^{\mathrm{mc}}_n(\hX_{t_i}) + \frac{Z_{t_i}}{\sigma_{t_i}}\right\|^2\right) .
\end{align}
For any fixed $i$, and $s \in \sS$, we note that with probability at least $1-\delta$, 
\begingroup
\allowdisplaybreaks
\begin{align}
    & \left|\frac{1}{ m}   \sum_{j=1}^m \left(\left\| s\left( \hX_{t_i}^{(j)}, t_i\right) + \frac{Z_{t_i}^{(j)}}{\sigma_{t_i}}  \right\|^2 - \E \left\|\hat{s}^{\mathrm{mc}}_n(\hX_{t_i}) + \frac{Z_{t_i}}{\sigma_{t_i}}\right\|^2\right)\right| \nonumber\\
    \le & \left\| \left\| s\left( \hX_{t_i}^{(j)}, t_i\right) + \frac{Z_{t_i}^{(j)}}{\sigma_{t_i}}  \right\|^2\right\|_{\psi_1} \sqrt{\frac{2\log(1/\delta)}{ c \
    m}} \label{e28}\\
    \lesssim & \left(\left\| \left\| s\left( \hX_{t_i}^{(j)}, t_i\right) \right\|^2\right\|_{\psi_1}+ \frac{1}{\sigma_{t_i}^2} \left\|\left\|Z_{t_i}^{(j)} \right\|^2\right\|_{\psi_1} \right) \sqrt{\frac{2\log(1/\delta)}{m}}\nonumber\\
    \lesssim & \frac{M}{\sigma_{t_i}^2}\sqrt{\frac{\log(1/\delta)}{m}}.
\end{align}
\endgroup
Here, \eqref{e28} follows from Lemma~\ref{lemi3}. In the above, $\|\cdot\|_{\psi_1}$ denotes the Orlicz-1 norm. By the union bound, with probability at least $1-N \delta$, 
\begin{align}
    \sum_{i=0}^{N-1} h_i\left|\frac{1}{ m}   \sum_{j=1}^m \left(\left\| s\left( \hX_{t_i}^{(j)}, t_i\right) + \frac{Z_{t_i}^{(j)}}{\sigma_{t_i}}  \right\|^2 - \E \left\|\hat{s}^{\mathrm{mc}}_n(\hX_{t_i}) + \frac{Z_{t_i}}{\sigma_{t_i}}\right\|^2\right)\right| 
    \lesssim \,  M \sum_{i=0}^{N-1} \frac{h_i}{\sigma_{t_i}^2}\sqrt{\frac{\log(1/\delta)}{m}}
\end{align}
Let $\cC\left(\epsilon; \sS, \ell_\infty\right)$ be an $\epsilon$-cover of $\sS$ in the $\ell_\infty$ norm. Further, let $s_0$ be such that \eqref{e29} is satisfied.  Thus, with probability at least $1-N \left( 1+ \cN\left(\epsilon; \sS, \ell_\infty\right) \right) \delta $, 
\begin{align}
    \sum_{i=0}^{N-1} h_i\left|\frac{1}{ m}   \sum_{j=1}^m \left(\left\| s\left( \hX_{t_i}^{(j)}, t_i\right) + \frac{Z_{t_i}^{(j)}}{\sigma_{t_i}}  \right\|^2 - \E \left\|\hat{s}^{\mathrm{mc}}_n(\hX_{t_i}) + \frac{Z_{t_i}}{\sigma_{t_i}}\right\|^2\right)\right|
    \lesssim  M \sum_{i=0}^{N-1} \frac{h_i}{\sigma_{t_i}^2}\sqrt{\frac{\log(1/\delta)}{m}},
\end{align}
for any $s \in \cC\left(\epsilon; \sS, \ell_\infty\right) \cup \{s_0\}$. Let $\tilde{s} \in \cC\left(\epsilon; \sS, \ell_\infty\right)$ be such that $\|\hat{s}^{\mathrm{mc}}_n-\tilde{s}\|_\infty \le \epsilon$. We now observe that, 
\begingroup
\allowdisplaybreaks
\begin{align}
      \sum_{i=0}^{N-1} h_i \E \left\|\hat{s}^{\mathrm{mc}}_n(\hX_{t_i}) + \frac{Z_{t_i}}{\sigma_{t_i}}\right\|^2  
     = & \sum_{i=0}^{N-1} h_i \E \left\|\tilde{s}(\hX_{t_i}) + \frac{Z_{t_i}}{\sigma_{t_i}}  + \hat{s}^{\mathrm{mc}}_n(\hX_{t_i}) - \tilde{s}(\hX_{t_i})\right\|^2 \nonumber \\
     \le & 2 \sum_{i=0}^{N-1} h_i \E \left\|\tilde{s}(\hX_{t_i}) + \frac{Z_{t_i}} {\sigma_{t_i}} \right\|^2 + 2 \left\|\hat{s}^{\mathrm{mc}}_n(\hX_{t_i}) - \tilde{s}(\hX_{t_i})\right\|^2 \nonumber \\
     \lesssim & \frac{1}{m}\sum_{i=0}^{N-1} \sum_{j=1}^m h_i  \left\|\tilde{s}(\hX_{t_i}^{(j)}) + \frac{Z_{t_i}^{(j)}} {\sigma_{t_i}} \right\|^2 + M \sum_{i=0}^{N-1} \frac{h_i}{\sigma_{t_i}^2}\sqrt{\frac{\log(1/\delta)}{m}}+ \epsilon^2 \nonumber\\
     \le & \frac{1}{m}\sum_{i=0}^{N-1} \sum_{j=1}^m h_i  \left\|s_0(\hX_{t_i}^{(j)}) + \frac{Z_{t_i}^{(j)}} {\sigma_{t_i}} \right\|^2 + M \sum_{i=0}^{N-1} \frac{h_i}{\sigma_{t_i}^2}\sqrt{\frac{\log(1/\delta)}{m}}
     + \epsilon^2 \nonumber\\
     \lesssim & \sum_{i=0}^{N-1} h_i \E \left\|s_0(\hX_{t_i}) + \frac{Z_{t_i}}{\sigma_{t_i}}\right\|^2 + M \sum_{i=0}^{N-1} \frac{h_i}{\sigma_{t_i}^2}\sqrt{\frac{\log(1/\delta)}{m}}
     + \epsilon^2.
\end{align}
\endgroup
Thus, with probability at least $1-\delta$,

\begin{align}
    & \sum_{i=0}^{N-1} h_i \E \| \hat{s}^{\mathrm{mc}}_n(\hX_{t_i}, t_i) - \nabla \log \hp_{t_i}(\hX_{t_i})\|^2 \nonumber \\
    \lesssim &  \sum_{i=0}^{N-1} h_i \E \| s_0(\hX_{t_i}, t_i) - \nabla \log \hp_{t_i}(\hX_{t_i})\|^2 + M \sum_{i=0}^{N-1} \frac{h_i}{\sigma_{t_i}^2}\sqrt{\frac{\log(N \left( 1+ \cN\left(\epsilon; \sS, \ell_\infty\right) \right)/\delta)}{m_i}}
     + \epsilon^2.
\end{align}
\end{proof}

\begin{lemma}
 \label{lemi3}
Let $\mathcal{X}$ be a Hilbert space and let $\{X_i\}_{i=1}^n$ be i.i.d. random variables in $\mathcal{X}$. Then, there exist a constant $c$, such that,  if $n \ge \frac{1}{c}\ln(2/\delta)$, then with probability at least $1 - \delta$:
\[
\left\| \frac{1}{n} \sum_{i=1}^{n} X_i - \mathbb{E}[X_1] \right\| \le \|X_1\|_{\psi_1} \sqrt{\frac{2 \ln(1/\delta)}{c \ n}}.
\]
\end{lemma}
\begin{proof}
    From Bernstein's inequality \citep[Theorem 2.9.1]{Vershynin_2026}, we note that, 
    \begin{align}
        \prob\left( \left| \sum_{i=1}^n X_i - n \E X_1\right| \ge  t \right) \le 2 \exp\left(-c \min\left\{\frac{ t^2}{ n \|X_1\|_{\psi_1}^2}, \frac{t}{\|X_1\|_{\psi_1}}\right\}\right).
    \end{align}
    Thus, if $t \le n \|X_1\|_{\psi_1}$, $\prob\left( \left| \sum_{i=1}^n X_i - n \E X_1\right| \ge  t \right) \le 2\exp\left(-c \frac{ t^2}{ n \|X_1\|_{\psi_1}^2} \right)$. Hence, with probability at least $1-\delta$, $\left| \sum_{i=1}^n X_i - n \E X_1\right| \le \|X_1\|_{\psi_1} \sqrt{\frac{n}{c}\log(2/\delta)}$, if $n \ge \frac{1}{c} \log(2/\delta)$.
\end{proof}
\section{Generalization Error}\label{ap_moment_thm}
\subsection{Supporting Results}
We first prove in Lemma~\ref{lem_c1} that the set $\sA_{p,q}$  (defined below) takes the shape of an interval (left open or closed). For notational simplicity, we will take $\alpha = q/(q-p)$ throughout this section.
\begin{lemma}\label{lem_c1}
   Suppose that \[\sA_{p,q} = \left\{s \in (2p, \infty) : \limsup_{\epsilon \downarrow 0} \frac{\log \sN_{\epsilon}\left(\mu, \epsilon^{\frac{s p q}{(q-p)(s-2p)}}\right)}{\log(1/\epsilon)} \leq s \right\}.\] Then, $\sA_{p,q} \supseteq (d^\star_{p,q}(\mu), \infty)$. 
\end{lemma}
\begin{proof}
     We begin by claiming the following: 
     \begin{equation}\label{claim:*}
        \text{\textbf{Claim}: If } s_1 \in \sA_{p,q} \text{ then } s_2 \in \sA_{p,q}, \text{ for all } s_2 \ge s_1 .
     \end{equation}
     To observe this, we note that,  if $s_2 \ge s_1>2 p$ and $\epsilon \in (0,1)$, 
 \[\frac{s_1 p \alpha}{s_1-2p } \ge  \frac{s_2 p \alpha}{s_2-2p} \implies \epsilon^{\frac{s_1 p \alpha}{s_1-2p}} \le  \epsilon^{\frac{s_2 p \alpha}{s_2-2p}} \implies \sN_{\epsilon}\left(\mu, \epsilon^{\frac{s_1 p \alpha}{s_1-2p}}\right) \ge \sN_{\epsilon}\left(\mu, \epsilon^{\frac{s_2 p \alpha}{s_2-2p}}\right).\]
Here, the last implication follows from Lemma~31 of \citet{JMLR:v26:24-0054}. Thus,
\[\limsup_{\epsilon \downarrow 0} \frac{\log \sN_{\epsilon}\left(\mu, \epsilon^{\frac{s_2 p \alpha}{s_2-2p}}\right)}{\log(1/\epsilon)} \le \limsup_{\epsilon \downarrow 0} \frac{\log \sN_{\epsilon}\left(\mu, \epsilon^{\frac{s_1 p \alpha}{s_1-2p}}\right)}{\log(1/\epsilon)} \le s_1 \le s_2.\]
Hence, $s_2 \in \sA_{p,q}$. Let $s > d^\star_{p,q}(\mu)$, then by definition of infimum, we note that
we can find $s^{\prime} \in [d^\star_{p,q}(\mu), s)$, such that, $s^{\prime} \in \sA_{p,q}$ . Since, $s>s^\prime \in \sA_{p,q}$ , by Claim~\eqref{claim:*}, $s \in \sA_{p,q}$ . Thus, for any $s > d^\ast_p(\mu)$, $s\in \sA_{p,q}$ , which proves the lemma.

\end{proof}

An immediate corollary of Lemma \ref{lem_c1} is as follows.
  \begin{cor}\label{cor_c1}
      Let $s > d^\star_{p, q}(\mu)$. Then, there exists $\epsilon^\prime \in (0,1]$, such that if $ 0< \epsilon \le \epsilon^\prime$, then, there exists a set $S$, such that $\cN(\epsilon; S, \varrho) \le \epsilon^{-s} $ and $\mu(S) \ge 1 - \epsilon^{\frac{s  p q}{(q-p)(s-2p)}}$.
  \end{cor}
 
 \begin{proof}
    Let $\delta = s - d^\star_{p,q}(\mu)$.  By Lemma~\ref{lem_c1}, we observe that  $s^\prime = d^\star_{p,q}(\mu) + \delta/2 \in \sA_{p,q}.$ Now by definition of $\limsup$, one can find a $\epsilon^\prime>0$, such that, $\frac{\log \sN_\epsilon\left(\mu, \epsilon^{\frac{s^\prime p \alpha}{s^\prime-2p}}\right)}{\log(1/\epsilon)} \le s^\prime + \delta/2 = s $, for all $\epsilon \in (0, \epsilon^\prime]$. The result now follows from observing that $\sN_\epsilon\left(\mu, \epsilon^{\frac{s \alpha}{s^-2\alpha}}\right) \le \sN_\epsilon\left(\mu, \epsilon^{\frac{s^\prime p \alpha}{s^\prime-2p}}\right) \le \epsilon^{-s}$.
 \end{proof}

 By Corollary \ref{cor_c1}, we can find an $\epsilon^\prime \in (0,1]$, such that if $\epsilon \in (0, \epsilon^\prime]$, we can find $S_\epsilon$, such that $\cN(\epsilon; S_\epsilon, \ell_\infty) \le \epsilon^{-d} $ and $\mu(S_\epsilon) \ge 1 - \epsilon^{\frac{d p \alpha}{d-2p}}$. We state and prove the following lemma that helps us create the base of a sequential partitioning of $\Real^D$. For notational simplicity, we use $\operatorname{diam}(A) = \sup_{x,y \in A} \varrho(x,y)$ to denote the diameter of a set w.r.t. the metric $\ell_\infty$-norm. 

\begin{lemma}\label{lem_c2}
    For any $r \ge \lceil \log_3(1/\epsilon^\prime) -2\rceil$, we can find disjoint sets  $S_{r,0},\dots,S_{r, m_r}$, such that, $\cup_{j=0}^{m_r} S_{r,j} = \Real^d$. Furthermore, $m_r \le 3^{d (r+2)}$, $\operatorname{diam}(S_{r,j})  \le 3^{-(r+1)}$, for all $j=1, \dots ,m_r$ and $\mu(S_{r,0}) \le 3^{-\frac{d(r+2) p \alpha }{d-2p}}$.
\end{lemma}
\begin{proof}
    We take $\epsilon = 3^{-(r+2)}$. Clearly, $0< \epsilon \le \epsilon^\prime$. We take $S_{r,0} = S_\epsilon^\mathsf{C}$. By definition of covering numbers, we can find a minimal $\epsilon$-net $\{x_1, \dots, x_{m_r}\}$, such that $S \subseteq\cup_{j=1}^{m_r} B_{\ell_\infty}(x_i, \epsilon)$ and $m_r \le \epsilon^{-d} = 3^{d(r+2)}$. We construct $S_{r,1}, \dots S_{r,m_r}$ as follows:
    \begin{itemize}
        \item Take $S_{r,1} = B_{\ell_\infty}(x_1, \epsilon) \setminus S_{r,0}$.
        \item For any $j =2,\dots, m_r$, we take $S_{r,j} = B_{\ell_\infty}(x_j, \epsilon) \setminus \left(\cup_{j^\prime=0}^{j-1} S_{r,j^\prime}\right)$.
    \end{itemize}
    By construction $\{S_{r,j}\}_{j=0}^{m_r}$ are disjoint. Moreover, $\mu(S_{r,0}) = 1 - \mu(S_\epsilon) \le \epsilon^{\frac{d p \alpha }{d - 2 p}} = 3^{-\frac{d(r+2) p \alpha }{d-2p}}$. Furthermore since, $S_{r,j} \subseteq B_{\ell_\infty}(x_j, \epsilon)$, \[\text{diam}(S_{r,j}) \le \text{diam}(B_{\ell_\infty}(x_j, \epsilon)) = 2 \epsilon = 2 \times 3^{-(r+2)} \le 3^{-(r+1)}.\]
\end{proof}

\subsection{Proof of Theorem~\ref{moment_thm}}

\begin{proof}
Let $s \le t$ with $s, t \in \fN$. We define $S_{s:t, 0} = \cup_{r=s}^t S_{r,0}$ and  $A_{t, i} = S_{t, i } \setminus S_{s:t, 0}$ for all $i \in [m_t]$. Clearly, $\operatorname{diam}(A_{t,i}) \le 3^{-(t+1)}$, form Lemma~\ref{lem_c2}. We define the following diadic partition of $\cup_{i=1}^{m_t} A_{t, i}$ as follows:

\begin{itemize}
    \item Take $\sQ^t = \{A_{t,i}\}_{i=1}^{m_t}$.
    \item Given $\ell+1$, let, \(Q^\ell_1 = \bigcup_{\substack{Q \in \sQ^{\ell+1}, \\ Q \cap S_{\ell,1} \neq \emptyset  }}Q  \) and if $2 \le j \le m_{\ell}$, we let, \[Q^\ell_j = \bigg(\bigcup_{\substack{ Q \in \sQ^{\ell+1}, \\ Q \cap S_{r,j} \neq \emptyset}}
     Q\bigg) \setminus \left(\cup_{j^\prime = 1}^{j-1} Q_{j^\prime}^\ell\right).\] Take $\sQ^\ell = \{Q^{\ell}_j\}_{j=1}^{m_\ell}$.
\end{itemize}
 Clearly for any $Q \in \sQ^\ell$, $\sup\limits_{Q \in \sQ^\ell}\text{diam}(Q) \le 2 \sup\limits_{Q^\prime \in \sQ^{\ell+1}} \text{diam}(Q^\prime) + \sup\limits_{1 \le j \le m_\ell}\text{diam}(S_{\ell,j}) \le 3 \times 3^{-(\ell+1)} = 3^{-\ell}$, by induction. Also, if $Q \in \sQ^{\ell+1}$, we can find sets $Q^\prime_1, \dots, Q^\prime_k \in \sQ^\ell$, such that $ \{Q^\prime_i\}_{i \in [k]}$ constitutes a partition of $Q$. Furthermore, $|\sQ^\ell| \le m_\ell$, by construction.

 Let,  $\mu_s = \mu$ and $\nu_s = \hmu_n$. We recursively define, 
 \begin{align*}
     \pi_r = & \sum_{j: Q^r_j \in  \sQ^r, \mu_t(Q^r_j) > 0} \left(1 - \frac{\nu_t(Q^r_j)}{\mu_t(Q^r_j)}\right)_+ \mu_r|_{Q^r_j}, \\
      \rho_r = & \sum_{j: Q^r_j \in  \sQ^r, \nu_t(Q^r_j) > 0} \left(1 - \frac{\mu_t(Q^r_j)}{\nu_t(Q^r_j)}\right)_+ \nu_r|_{Q^r_j}. 
 \end{align*}
 Clearly, $0 \le \pi_r \le \mu_r$ and $0 \le \rho_r \le \nu_r$. Further $\mu_{r+1} = \mu - \sum_{\tau = s}^r \pi_\tau$ and $\nu_{r+1} = \nu - \sum_{\tau = s}^t \rho_\tau$. It is easy to see that if $\gamma_r \in \operatorname{Couple}(\pi_r, \rho_r)$ for all $r \le t$ and  $\gamma_{t+1} \in \operatorname{Couple}(\mu_{t+1}, \rho_{t + 1})$, we note that,
 \[\sum_{r=s}^{t + 1} \gamma_t \in \operatorname{Couple}\left(\sum_{r=s}^{t} \pi_t + \mu_{t + 1}, \sum_{r=s}^{t} \rho_r + \nu_{t + 1}\right) = \operatorname{Couple}(\mu, \hmu_n).\]
Here, $\operatorname{Couple}(\mu, \nu)$ denotes the set of all measure couples between $\mu$ and $\nu$. Thus, 
 \begin{align}
    \fW_p^p(\mu, \hmu_n) 
    \le & \sum_{r=s}^{t} \fW_p^p (\pi_t, \rho_t) + \fW_p^p(\mu_{t+1}, \nu_{t + 1})\nonumber\\
     \lesssim & \sum_{r=s}^{t} \fW_p^p (\pi_t, \rho_t) + \fW_p^p(\mu_{t+1}|_{S_{s:t,0}^\mathsf{C}}, \nu_{t + 1}|_{S_{s:t,0}^\mathsf{C}})   + \fW_p^p(\mu, \mu_{t+1}|_{S_{s:t,0}}) \nonumber\\
     &+ \fW_p^p(\hmu_n, \nu_{t+1}|_{S_{s:t,0}}). \label{e13}
 \end{align}
 We define, \[\gamma_r = \sum_{ Q \in  \sQ^{r-1}, \pi_r(Q) > 0} \frac{\pi_r|_{Q} \otimes \rho_r|_{Q}}{\pi_r(Q)}.\] 
 
 First, we show that $\gamma_r$ is indeed a measure couple between $\pi_r$ and $\rho_r$. Suppose that $U$ is a measurable subset of $\Real^D$. Then,

 \[\gamma_r(\Real^D, U) = \sum_{ Q \in  \sQ^{r-1}, \pi_r(Q) > 0} \frac{\pi_r(Q)  \rho_r (Q\cap U)}{\pi_r(Q)} = \rho_r(U). \]
 Similarly,
 \[\gamma_r(U, \Real^D) = \sum_{ Q \in  \sQ^{r-1}, \pi_r(Q) > 0} \frac{\pi_r(Q \cap U)  \rho_r (Q )}{\pi_r(Q)} =  \sum_{ Q \in  \sQ^{r-1}, \pi_r(Q) > 0} \frac{\pi_r(Q \cap U)  \rho_r (Q)}{\rho_r(Q)} = \pi_r(U). \]
 The second equality in the above equation follows from Lemma~\ref{lem:8}. Thus,
 \begin{align}
     \fW_p^p(\pi_r, \rho_r) \le & \int \|x-y\|^p \diff  \gamma(x,y)\nonumber\\
     = & \sum_{ Q \in  \sQ^{r-1}, \pi_r(Q) > 0} \frac{1}{\pi_r(Q)} \int_Q \|x-y\|^p \diff \pi_r(x) \diff  \rho_r(y) \nonumber\\
     \le & \sum_{ Q \in  \sQ^{r-1}, \pi_r(Q) > 0} \frac{(\operatorname{diam}(Q))^p }{\pi_r(Q)}  \pi_r(Q) \rho_r(Q) \nonumber\\
     = & 3^{-p(r-1)} \rho_r(S_{s:t,0}^\mathsf{C}). \label{e12}
 \end{align}
 Again, since $\mu_{t+1}(Q) = \nu_{t+1}(Q)$ for all $Q \in \sQ^t$, by a similar argument,
 \begin{align}
     \fW_p^p(\mu_{t+1}|_{S_{s:t,0}^\mathsf{C}}, \nu_{t+1}|_{S_{s:t,0}^\mathsf{C}}) \le &  3^{-pt} \mu_{t+1}(S_{s:t:0}^\mathsf{C}) \le 3^{-pt}. \label{e14}
 \end{align}
 combining \eqref{e13}, \eqref{e12} and \eqref{e14}, we observe that,
 \begin{align}
     \fW_p^p(\mu, \hmu) \lesssim & 3^{-p t} + \sum_{r=s}^t 3^{-p(r-1)} \rho_r(S_{s:t,0}^\mathsf{C}) + \fW_p^p(\mu, \mu_{t+1}|_{S_{s:t,0}}) + \fW_p^p(\hmu_n, \nu_{t+1}|_{S_{s:t,0}}). \label{e15}
 \end{align}

 We now concentrate on bounding $\rho_r(S_{s:t,0}^\mathsf{C})$. Clearly, 
 \begin{align}
     \rho_r(S_{s:t,0}^\mathsf{C}) = & \sum_{Q \in \sQ^{r}} (\nu_r(Q) - \mu_r(Q))_+ = \frac{1}{2} \sum_{Q \in \sQ^{r}} |\nu_r(Q) - \mu_r(Q)| .\label{e16}
 \end{align}
 We claim that for any $s \le r \le t$, there exists scalars $c_1(Q,r)$ and $c_2(Q,r)$, such that,  $\mu_r|_Q = c_1(Q,r) \mu|_Q $ and $\nu_r|_Q = c_2(Q,r) \hmu_n|_Q$, for all $Q \in \sQ^{r-1}$. We prove this claim through induction on $r$ for $\mu$. The result for $\nu_r$ follows similarly.  Clearly the result holds for $r = 1$, for which $\mu_1 = \mu$. Let us assume that the result holds for $r-1$, i.e., $\mu_{r-1}|_Q = c_1(Q,r-1) \mu|_Q$. Thus,
 \[\mu_r = \mu_{r-1}|_Q - \pi_{r-1}|_Q =  \min\left\{ 1, \frac{\nu_{r-1}(Q)}{\mu_{r-1}(Q)}\right\} \mu_{r-1}|_Q = \min\left\{ 1, \frac{\nu_{r-1}(Q)}{\mu_{r-1}(Q)}\right\} c_1(Q,r-1) \mu|_Q .\]
 Taking $c_1(Q,r) = \min\left\{ 1, \frac{\nu_{r-1}(Q)}{\mu_{r-1}(Q)}\right\} c_1(Q,r-1)$ proves the claim.

 By construction, $c_1, c_2 \in [0,1]$. We know that $\mu_r(Q) = \nu_r(Q), $ for all $Q \in \sQ^{r-1}$. Thus, $c_1(Q,r) \mu(Q) = c_2(Q,r) \hmu(Q)$. Thus, from \eqref{e16}, 
 \begingroup
 \allowdisplaybreaks
 \begin{align}
     \rho_r(S_{s:t,0}^\mathsf{C}) = & \frac{1}{2} \sum_{Q \in \sQ^{r}} |\nu_r(Q) - \mu_r(Q)| \nonumber\\
     = & \frac{1}{2} \sum_{Q \in \sQ^{r}} |c_2(Q,r) \hmu(Q) - c_1(Q,r) \mu(Q) | \nonumber\\
     \le &  \frac{1}{2} \sum_{Q \in \sQ^{r}} \left(|c_2(Q,r) \hmu(Q) - c_2(Q,r) \mu(Q) | + | c_1(Q,r) \mu(Q) - c_2(Q,r) \mu(Q)| \right) \nonumber\\
     = & \frac{1}{2} \sum_{Q \in \sQ^{r}} \left( c_2(Q,r)| \hmu(Q) - \mu(Q) | + | c_2(Q,r) \hmu(Q) - c_2(Q,r) \mu(Q)|  \right) \nonumber\\
     \le & \sum_{Q \in \sQ^{r}} |\hmu(Q) - \mu(Q) | \label{e17}.
 \end{align}
 \endgroup

 By construction, $ \mu_{t+1}|_{S_{s:t,0}}= \mu|_{S_{s:t,0}} $. Thus,
 \begin{align}
     \fW_p^p(\mu, \mu_{t+1}|_{S_{s:t,0}}) = & \fW_p^p(\mu, \mu|_{S_{s:t,0}}) \le \E_{X \sim \mu} \|X\|^p \one\left\{ X \in {S_{s:t,0}} \right\}
     \le  (\sM_q(\mu))^p (\mu(S_{s:t,0}))^{\frac{q-p}{q}} \label{e18}.
 \end{align}
 Here \eqref{e18} follows from H\"older's inequality. We also note that, 
 \begin{align}
    \E  \fW_p^p(\hmu_n, \nu_{t+1}|_{S_{s:t,0}}) =  \E \fW_p^p(\hmu_n, \nu|_{S_{s:t,0}})
    \le  \E_{X \sim \mu} \|X\|^p \one\left\{ X \in {S_{s:t,0}} \right\} 
    \le  (\sM_q(\mu))^p (\mu(S_{s:t,0}))^{\frac{q-p}{q}} \label{e19}.
 \end{align}
 Thus, combining \eqref{e15}, \eqref{e17} and \eqref{e19}, we observe that,
 \begingroup
 \allowdisplaybreaks
 \begin{align}
     \E \fW^p_p(\mu, \hmu) \lesssim & \  3^{-p t} + \sum_{r=s}^t 3^{-p(r-1)} \sum_{Q \in \sQ^{r}} \E |\hmu(Q) - \mu(Q) |   + (\sM_q(\mu))^p (\mu(S_{s:t,0}))^{\frac{q-p}{q}} \nonumber\\
     \le & \ 3^{-p t} + \sum_{r=s}^t 3^{-p(r-1)}  \sqrt{\frac{m_r}{n}}   + (\sM_q(\mu))^p (\mu(S_{s:t,0}))^{\frac{q-p}{q}} \label{e20}\\
     \le & \ 3^{-p t} + \sum_{r=s}^t 3^{-p(r-1)}  \sqrt{\frac{3^{d (r+2)}}{n}}   + (\sM_q(\mu))^p \left(\sum_{r=s}^t\mu(S_{r,0})\right)^{\frac{q-p}{q}} \nonumber\\
     \le & \ 3^{-p t} + \frac{3^{p+d}}{\sqrt{n}}\sum_{r=s}^t 3^{-\frac{r}{2}(d - 2p)}   + (\sM_q(\mu))^p \left(\sum_{r=s}^t 3^{-\frac{d(r+2) p \alpha }{ d - 2 p}}\right)^{\frac{q-p}{q}} \nonumber \\
     \lesssim & \ 3^{-p t} + \frac{3^{p+d}}{\sqrt{n}}\dfrac{3^{-\tfrac{s}{2}(d'-2p)} \left(1 - 3^{-\tfrac{(t-s+1)}{2}(d'-2p)}\right)}{1 - 3^{-\tfrac{1}{2}(d'-2p)}}  + \left( \frac{3^{-\frac{d s p \alpha }{ d - 2 p}}}{ 1 - 3^{-\frac{d p \alpha }{ d - 2 p}}}\right)^{\frac{q-p}{q}} \nonumber\\
     \lesssim & \ 3^{-p t} + \frac{3^{\frac{d - 2 p }{2} t}}{\sqrt{n}} + (3^{-s})^{\frac{d p \alpha (q-p)}{(d - 2 p)q}}  \nonumber\\
     = & \ 3^{-p t} + \frac{3^{\frac{d - 2 p }{2} t}}{\sqrt{n}} + (3^{-s})^{\frac{d p }{(d - 2 p)}}. \label{e21}
 \end{align}
 \endgroup

 We take $\epsilon = n^{-1/d}$, $t$ to be the smallest integer such that $3^{-t} \le \epsilon$ and $s$ to be the smallest integer such that $3^{-s} \le \epsilon^{(d - 2p)/d}$. Clearly, $t \ge s$. Also by definition, $3^t \le 3/\epsilon$ and $3^s \le 3 \epsilon^{(2 p - d)/d}$. Thus, from \eqref{e21},
 \begin{align}
     \E \fW^p_p(\mu, \hmu) \lesssim & \ \epsilon^p + \frac{\epsilon^{-(d - 2p)/2}}{\sqrt{n}} + \epsilon^{p} \lesssim n^{-\frac{p}{d}}.
 \end{align}

 \end{proof}

\subsection{Proof of Corollary~\ref{cor4}}
\begin{proof}
     The corollary easily follows from Jensen's inequality by observing that,
     \begin{align*}
         \E \fW_p(\mu, \hmu_n)  =  \E \left( \inf_\gamma \E_{(X,Y) \sim \gamma} \|X-Y\|^p\right)^{1/p} 
         \le  \left( \E \inf_\gamma \E_{(X,Y) \sim \gamma} \|X-Y\|^p\right)^{1/p} 
         =  \left( \E \fW_p^p(\mu, \hmu_n)\right)^{1/p}.
     \end{align*}
     Here the infimum is taken over all measure couples between $\mu$ and $\hmu_n$.
 \end{proof}
\subsection{Additional Result}
 \begin{lemma}\label{lem:8}
      For any $r, \tau \in \fN$, such that $s \le \tau <  r \le t$, $\pi_r(Q) = \rho_r(Q)$, for all $Q \in \sQ^\tau$.
  \end{lemma}
  \begin{proof}
 Fix $ s \le r \le t $. We first observe that $\pi_r(Q^r_j) = (\mu_r(Q^r_j) - \nu_r(Q^r_j))_+$. Thus, $\mu_{r+1}(Q^r_j) = \mu_{r}(Q^r_j) - \pi_{r}(Q^r_j) = \mu_{r}(Q^r_j) \wedge \nu_{r}(Q^r_j)$. Similarly, $\nu_{r+1}(Q^r_j) = \mu_{r}(Q^r_j) \wedge \nu_{r}(Q^r_j)$. Thus, for all $s < k \le t+1$, $\mu_k(Q^{k-1}_j) = \nu_k(Q^{k-1}_j)$. Suppose that $Q \in \sQ^\tau$, where $\tau < r$. By construction, $Q$ is a disjoint union of $Q^{r-1}_{i_1}, \dots, Q_{i_m}^{r-1}$ for some indices $i_1, \dots, i_m$. Thus,
 \[\mu_r(Q) = \sum_{j = 1}^m \mu_r(Q^{r-1} ) = \sum_{j = 1}^m \nu_r(Q^{r-1} ) = \nu_r(Q). \]
 Hence, \[\pi_r(Q) = \mu_r(Q) - \mu_{r+1}(Q) = \nu_r(Q) - \nu_{r+1}(Q) = \rho(Q).\]
 \end{proof}
 \section{Proof of the Oracle Inequality}
 \label{ap_oracle}
\subsection{Supporting Results}
To prove Lemma~\ref{lem_oracle_2}, we first consider the following supporting results.
\begin{lemma}\label{wtvbd}
    Suppose that $P$ and $Q$ are two probability measures supported on $[-R,R]^D$. Then, 
    \[\fW_p(P,Q) \le 2 R \sqrt{D} \left(\operatorname{TV}(P,Q)\right)^{1/p},\]
    for any $p \ge 1$.
\end{lemma}
\begin{proof}
    Let, $X\sim P$ and $Y \sim Q$. We note that, $\|X-Y\| \le \one\{X \neq Y\} 2 R \sqrt{D}$. Hence,
    \begin{align*}
        \fW_p(P,Q) = \left(\inf_\gamma \E_{(X, Y) \sim \gamma} \|X-Y\|^p \right)^{1/p} \le  2R \sqrt{D} \left(\inf_\gamma \gamma(X\neq Y)\right)^{1/p} 
        =  2 R \sqrt{D} \left(\operatorname{TV}(P,Q)\right)^{1/p}.
    \end{align*}
\end{proof}
\begin{lemma}\label{lem_j_3}
    Suppose that $a_, \dots, a_k$ are non-negative constants. Then, $\left(\sum_{i=1}^k a_i\right)^q \le k^{q-1} \sum_{i=1}^k a_i^q $, for $q \ge 1$.
\end{lemma}
\begin{proof}
    Then, $\sum_{i=1}^k a_i \le \left(\sum_{i=1}^k a_i^q \right)^{1/q} k^{1-1/q}$, from H\"older's inequality. Taking a $q$-th power on both sides yields the result.
\end{proof}
\begin{lemma}\label{lem_oracle}
Suppose that $\nu$ is dominated by the Lebesgue measure on $\Real^D$. Then, for any $p \in (0, q)$, $\delta_0 \in (0,T)$ and $R > 0$, 
\begingroup
\allowdisplaybreaks
    \begin{align}
        & \fW_p(\mu, \sT_R(\nu)) \nonumber\\
        \le & \fW_p(\mu, \hmu_n) + 2R\sqrt{D}  \left(\tv(\hP_{\delta_0}, \nu)\right)^{1/p} \nonumber\\
        & +  \left( \overline{\beta}\delta_0 \sM_p(\hmu_n) + (\overline{\beta}\delta_0)^{1/2} \sM_p(\gamma_D) \right) + 2^{q-1} \left( \sM_p^q(\hmu_n) +   \sM_q^q(\gamma_D) \right) R^{-(q-p)}. \label{eq_oracle}
    \end{align}
    \endgroup
\end{lemma}
\begin{proof}
To begin the proof, we note that,
\begingroup
\allowdisplaybreaks
\begin{align}
    \fW_p(\mu, \sT_R(\nu)) = \fW_p(\mu, \sT_R(\nu)) 
    \le & \fW_p(\mu, \sT_R(\hP_{\delta_0})) + \fW_p(\sT_R(\hP_{\delta_0}), \sT_R(\nu)) \nonumber\\
    \le & \fW_p(\mu, \sT_R(\hP_{\delta_0})) + 2R \sqrt{D} \left(\tv(\sT_R(\hP_{\delta_0}), \sT_R(\nu))\right)^{1/p} \label{e1}\\
    \le & \fW_p(\mu, \sT_R(\hP_{\delta_0})) + 2R \sqrt{D} \left(\tv(\hP_{\delta_0}, \nu)\right)^{1/p}. \label{e2}
\end{align}
\endgroup
Here \eqref{e2} follows from the data processing inequality for the Total Variation distance. We first bound the first term in \eqref{e2}, i.e., $\fW_p(\mu, \sT_R(\hP_{\delta_0}))$. For notational simplicity, we define, \(\mathsf{m}_t := \exp\left(-\int_0^t \beta_\tau \diff  \tau\right).\) We note the following: 
\begingroup
\allowdisplaybreaks
\begin{align}
       \fW_p(\mu, \sT_R(\hP_{\delta_0})) 
    \le & \ \fW_p(\mu, \hmu_n) +  \fW_p(\hmu_n, \hP_{\delta_0}) + \fW_p(\hP_{\delta_0}, \sT_R(\hP_{\delta}) ) \nonumber\\
    \le & \ \fW_p(\mu, \hmu_n) + \left(\E_{\hat{X} \sim \hP} \|\hat{X} - \mathsf{m}_t \hat{X} - \sigma_{\delta_0} \xi\|^p \right)^{1/p}+ \left(\E \|\hat{X}_{\delta_0}\|^p \one(\|\hX_{\delta_0}\|_\infty \ge R) \right)^{1/p}\nonumber\\
    \le & \ \fW_p(\mu, \hmu_n) + \left(\E_{\hat{X} \sim \hP} \|\hat{X} - \mathsf{m}_t \hat{X} - \sigma_{\delta_0} \xi\|^p\right)^{1/p} + \left(\E \|\hat{X}_{\delta_0}\|^p \one(\|\hX_{\delta_0}\| \ge R)\right)^{1/p} \nonumber\\
    \le & \ \fW_p(\mu, \hmu_n) +  \left( (1-m_{\delta_0}) \sM_p(\hmu_n) + \sigma_{\delta_0} \sM_p(\gamma_D) \right)  + (\E \|\hat{X}_{\delta_0}\|^q)^{1/q} \prob(\|\hat{X}_{\delta_0}\|\ge R)^{\frac{q-p}{pq}} \label{e3} \\
    \le & \fW_p(\mu, \hmu_n) +  \left( (1-m_{\delta_0}) \sM_p(\hmu_n) + \sigma_{\delta_0} \sM_p(\gamma_D) \right) 
    + (\E \|\hat{X}_{\delta_0}\|^q)^{1/q} \left(\frac{\E \|\hat{X}_{\delta_0}\|^q}{R^q}\right)^{\frac{q-p}{pq}} \label{e4}\\
    = & \ \fW_p(\mu, \hmu_n) +  \left( (1-m_{\delta_0}) \sM_p(\hmu_n) + \sigma_{\delta_0} \sM_p(\gamma_D) \right) + \left(\E \|\hat{X}_{\delta_0}\|^q\right)^{1/p} R^{-\frac{q-p}{p}} \nonumber\\
    \le & \ \fW_p(\mu, \hmu_n) +  \left( (1-m_{\delta_0}) \sM_p(\hmu_n) + \sigma_{\delta_0} \sM_p(\gamma_D) \right) \nonumber\\
    & + 2^{(q-1)/p} \left( m_{\delta_0}^q \sM_p^q(\hmu_n) +  \sigma_{\delta_0}^q \sM_q^q(\gamma_D) \right)^{1/p} R^{-\frac{q-p}{p}} \label{eee6}\\
    \le & \ \fW_p(\mu, \hmu_n) + \left( (1-m_{\delta_0}) \sM_p(\hmu_n) + \sigma_{\delta_0} \sM_p(\gamma_D) \right)  \nonumber\\
    & + 2^{(q-1)/p} \left( \sM_p^q(\hmu_n) +   \sM_q^q(\gamma_D) \right)^{1/p} R^{-(q-p)/p} \label{e5}\\
    \le & \ \fW_p(\mu, \hmu_n) +  \left( (1-e^{-\overline{\beta}\delta_0}) \sM_p(\hmu_n) + (1-e^{-\overline{\beta}\delta_0})^{1/2} \sM_p(\gamma_D) \right) \nonumber\\
    & + 2^{(q-1)/p} \left( \sM_p^q(\hmu_n) +   \sM_q^q(\gamma_D) \right)^{1/p} R^{-(q-p)/p} \nonumber\\
    \le & \ \fW_p(\mu, \hmu_n) +  \left( \overline{\beta}\delta_0 \sM_p(\hmu_n) + (\overline{\beta}\delta_0)^{1/2} \sM_p(\gamma_D) \right) \nonumber\\
    & + 2^{(q-1)/p} \left( \sM_p^q(\hmu_n) +   \sM_q^q(\gamma_D) \right)^{1/p} R^{-(q-p)/p}. \label{e6}
\end{align}
\endgroup
In the above calculations, \eqref{e3} follows from H\"older's inequality and \eqref{e4} follows from Markov's inequality. Inequality \eqref{eee6} follows from Lemma~\ref{lem_j_3}.
\end{proof}
The next step in proving Lemma~\ref{lem_oracle_2} is to bound the second term of \eqref{eq_oracle} in terms of the approximation and discretization errors as follows.
\begin{lemma}
\label{lem9}Under Assumption~\ref{a1}, for any $\hat{s} \in \sS$, 
    \begin{align}
        \kl(\hP_{\delta_0}, \hQ_{T-\delta_0
        }(\hat{s})) \le & \ \kl(\hP_T, \gamma_D) + 4 \overline{\beta}^2 \sum_{i=0}^{N-1} h_i \E \|\hat{s}(\hX_{T-t_i}, T-t_i) - \nabla \log \hp_{T-t_i}(\hX_{T-t_i})\|^2   \nonumber\\
  & +  4 \overline{\beta} \sum_{i=0}^{N-1}\int_{t_i}^{t_{i+1}} \E\|\nabla \log \hp_{T-t_i}(\hX_{T-t_i}) - \nabla \log \hp_{T-t}(\hX_{T-t})\|^2 \diff  t \label{e23}.
    \end{align}
\end{lemma}
\begin{proof}
From Girsanov's theorem (for a reference, see \citep[Lemma 20]{dou2024optimalscorematchingoptimal} or \citep[Theorem 8.6.5]{oksendal2003stochastic}), we know that,
\begin{align*}
    \kl(Q_{t_{i+1}}, \hQ_{t_{i+1}} (\hat{s})) \le & \kl(Q_{t_i}, \hQ_{t_i}(\hat{s})) + 2 \int_{t_i}^{t_{i+1}} \beta^2_{T-t} \E \|\hat{s}(Y_{t_i}, T-t_i) - \nabla \log \hp_{T-t}(Y_t)\|^2 \diff  t\\
  \implies  \kl(Q_{t_N}, \hQ_{t_N}(\hat{s})) \le & \kl(Q_{t_0}, \hQ_{t_0}(\hat{s})) + 2 \sum_{i=0}^{N-1}\int_{t_i}^{t_{i+1}} \beta^2_{T-t} \E\|\hat{s}(Y_{t_i}, T-t_i) - \nabla \log \hp_{T-t}(Y_t)\|^2 \diff  t.
\end{align*}
This implies that,
\begingroup
\allowdisplaybreaks
\begin{align}
      \kl(\hP_{\delta_0}, \hQ_{T-\delta_0}(\hat{s}))
  \le & \kl(\hP_T, \gamma_D) + 2 \sum_{i=0}^{N-1}\int_{t_i}^{t_{i+1}} \beta^2_{T-t} \E \|\hat{s}(\hX_{T-t_i}, T-t_i) - \nabla \log \hp_{T-t}(\hX_{T-t})\|^2 \diff  t \nonumber\\
  \le & \kl(\hP_T, \gamma_D) \nonumber\\
  & + 4 \sum_{i=0}^{N-1}\int_{t_i}^{t_{i+1}} \beta^2_{t}\E \|\hat{s}(\hX_{T-t_i}, t_i) - \nabla \log \hp_{T-t_i}(\hX_{T-t_i})\|^2  \diff t \nonumber\\
  & + 4 \sum_{i=0}^{N-1}\int_{t_i}^{t_{i+1}} \beta^2_{t} \E \|\nabla \log \hp_{T-t_i}(\hX_{T-t_i}) - \nabla \log \hp_{T-t}(\hX_{T-t})\|^2 \diff  t \nonumber\\
  \le & \kl(\hP_T, \gamma_D) + 4 \sum_{i=0}^{N-1} \E \|\hat{s}(\hX_{T-t_i}, T-t_i) - \nabla \log \hp_{T-t_i}(\hX_{T-t_i})\|^2  \int_{T-t_i}^{T-t_{i+1}}  \beta^2_{t} \diff  t \nonumber\\
  & +  4 \sum_{i=0}^{N-1}\int_{t_i}^{t_{i+1}} \beta^2_{T-t} \E \|\nabla \log \hp_{T-t_i}(\hX_{T-t_i}) - \nabla \log \hp_{T-t}(\hX_{T-t})\|^2 \diff  t \nonumber\\
  \le &  \kl(\hP_T, \gamma_D) + 4 \overline{\beta}^2 \sum_{i=0}^{N-1} h_i \E \|\hat{s}(\hX_{T-t_i}, T-t_i) - \nabla \log \hp_{T-t_i}(\hX_{T-t_i})\|^2   \nonumber\\
  & +  4 \overline{\beta}^2 \sum_{i=0}^{N-1}\int_{t_i}^{t_{i+1}} \E \|\nabla \log \hp_{T-t_i}(\hX_{T-t_i}) - \nabla \log \hp_{T-t}(\hX_{T-t})\|^2 \diff  t . \label{e7}
\end{align}
\endgroup
\end{proof}
\subsection{Proof of Lemma~\ref{lem_oracle_2}}
Combining the bounds obtained in Lemmata~\ref{lem_oracle} and \ref{lem9}, we prove Lemma~\ref{lem_oracle_2} as follows.
\lemoracle*
\begin{proof}
The proof follows from combining Lemmata~\ref{lem_oracle} and \ref{lem9}. Formally,
\begingroup
\allowdisplaybreaks
\begin{align}
      \fW_p(\mu, \sT_R(\nu)) 
     \le & \ \fW_p(\mu, \hmu_n) + 2R\sqrt{D} \left(\tv(\hP_{\delta}, \nu)\right)^{1/p} \nonumber\\
        & +   \overline{\beta}\delta_0 \sM_p(\hmu_n) + (\overline{\beta}\delta_0)^{1/2} \sM_p(\gamma_D)  + 2^{(q-1)/p} \left( \sM_p^q(\hmu_n) +   \sM_q^q(\gamma_D) \right)^{1/p} R^{-(q-p)/p} \nonumber\\
        \le & \ \fW_p(\mu, \hmu_n) + R \sqrt{D} 2^{\frac{p-1}{2p}} \bigg(\kl(\hP_T, \gamma_D) \nonumber\\
        & + 4 \overline{\beta}^2 \sum_{i=0}^{N-1} h_i \E \|s(\hX_{T-t_i}, T-t_i) - \nabla \log \hp_{T-t_i}(\hX_{T-t_i})\|^2   \nonumber\\
  & +  4 \overline{\beta} \sum_{i=0}^{N-1}\int_{t_i}^{t_{i+1}} \E\|\nabla \log \hp_{T-t_i}(\hX_{T-t_i}) - \nabla \log \hp_{T-t}(\hX_{T-t})\|^2 \diff  t\bigg)^{\frac{1}{2p}} \nonumber\\
        & +   \overline{\beta}\delta_0 \sM_p(\hmu_n) + (\overline{\beta}\delta_0)^{1/2} \sM_p(\gamma_D)  + 2^{(q-1)/p} \left( \sM_p^q(\hmu_n) +   \sM_q^q(\gamma_D) \right)^{1/p} R^{-(q-p)/p} \label{e27}\\
        \le & \ \fW_p(\mu, \hmu_n) + R \sqrt{D} 2^{\frac{p-1}{2p}} \kl(\hP_T, \gamma_D)^{\frac{1}{2p}} \nonumber\\
        & + R (2\overline{\beta}^2 )^{\frac{1}{2p}} \left(\sum_{i=0}^{N-1} h_i \E \|s(\hX_{T-t_i}, T-t_i) - \nabla \log \hp_{T-t_i}(\hX_{T-t_i})\|^2 \right)^{1/2p} \nonumber\\
        & + R (2\overline{\beta}^2 )^{\frac{1}{2p}} \left(\sum_{i=0}^{N-1}\int_{t_i}^{t_{i+1}} \E\|\nabla \log \hp_{T-t_i}(\hX_{T-t_i}) - \nabla \log \hp_{T-t}(\hX_{T-t})\|^2 \diff  t\right)^{1/2p} \nonumber\\
        & +  \overline{\beta}\delta_0 \sM_p(\hmu_n) + (\overline{\beta}\delta_0)^{1/2} \sM_p(\gamma_D)  + 2^{(q-1)/p} \left( \sM_p^q(\hmu_n) +   \sM_q^q(\gamma_D) \right)^{1/p} R^{-(q-p)/p} .\nonumber
\end{align}
\endgroup
In the above calculations, \eqref{e27} follows from applying Pinsker's inequality and Lemma~\ref{lem9}.
\end{proof}
 \section{Early Stopping Error}\label{app_e}
 \klbd*
 \begin{proof}
Let $\hX_0 \sim \hmu_n$. Recall that, $\mathsf{m}_t = \exp\left(-\int_0^t \beta_\tau \diff  \tau\right)$ and $\sigma^2_t = 1-\mathsf{m}_t^2$. Then,
\begingroup
\allowdisplaybreaks
    \begin{align*}
        \kl(\hP_t , \gamma_D) =  \E_{\hX_0} \kl(\hP_{t|0}(\cdot | \hX_0), \gamma_D) 
        = & \frac{1}{2} \left(D \sigma_t^2 - D - D\log(\sigma_t^2)  + \frac{\mathsf{m}_t^2\E\|\hX_0\|^2}{\sigma_t^2}\right)\\
        \le & \frac{1}{2} \left( D\log(1/\sigma_t^2)  + \frac{\mathsf{m}_t^2\E\|\hX_0\|^2}{\sigma_t^2}\right)\\
        = & \frac{1}{2} \left( D\log\left(\frac{1}{1-\mathsf{m}_t^2}\right)  + \frac{\mathsf{m}_t^2\E\|\hX_0\|^2}{1-\mathsf{m}_t^2}\right)\\
        \le & \frac{1}{2} \left( D\frac{\mathsf{m}_t^2}{1-\mathsf{m}_t^2}  + \frac{\mathsf{m}_t^2\E\|\hX_0\|^2}{1-\mathsf{m}_t^2}\right)\\
        = & \frac{\mathsf{m}_t^2}{2(1-\mathsf{m}_t^2)} (D + \E \|\hX_0\|^2).
    \end{align*}
    \endgroup
    In the above calculations, the expectation is w.r.t. the forward process, i.e., conditional on the data $X_1, \dots, X_n$. We note that $\mathsf{m}_t = \exp\left(-\int_0^t \beta_\tau \diff  \tau\right) \le \exp\left(-\underline{\beta}t\right)$. Similarly, $\mathsf{m}_t \ge \exp\left(-\overline{\beta} t\right)$. Thus,
  \begin{align*}
      \kl(\hP_t , \gamma_D) \le \frac{\exp\left(-2\underline{\beta}t\right)}{2(1-\exp\left(-\overline{\beta}t\right))^2} (D + \E \|\hX_0\|^2).
  \end{align*}
  Taking $t \ge \log 2/\overline{\beta}$, we observe that, \[\kl(\hP_t , \gamma_D) \le \exp\left(-2\underline{\beta}t\right) (D + \E \|\hX_0\|^2) = \exp\left(-2\underline{\beta}t\right) (D + \sM_2^2(\hmu_n)).\]
\end{proof}
\section{Approximation Error}\label{pf_approximation}
\thmapprox*
\begin{proof}
Suppose that $M = \max_{i \in [n]} \|X_i\|_\infty$. From Lemma~\ref{lem_h5}, we note that,
\[\nabla^2 \log \hp_t(x) = \frac{\mathsf{m}_t ^2 }{\sigma_t^4} \operatorname{Var}\left(\hX_0 \big| X_t =x\right) - \frac{I_D}{\sigma_t^2}.\] Hence $\|\nabla^2 \log \hp_t(x)\|_{\infty} \le \frac{\mathsf{m}_t^2 M^2}{\sigma_t^4} + \frac{1}{\sigma_t^2}$, for all $x \in \Real^D$. Further, $\left\|\nabla \log \hp_t(x) + \frac{x}{\sigma^2}\right\|_\infty \le \frac{\mathsf{m}_t M}{\sigma_t^2}$. The function $g(x) = \nabla \log \hp_t(x) + \frac{x}{\sigma^2}$. Clearly,
\[\nabla g(x) = \frac{\mathsf{m}_t ^2 }{\sigma_t^4} \operatorname{Var}\left(\hX_0 \big| X_t =x\right). \]
Thus, $\|\nabla g(x)\|_\infty \le \frac{\mathsf{m}_t^2 M^2}{\sigma_t^4}$ and $\|g(x)\|_\infty \le \frac{\mathsf{m}_t M}{\sigma_t^2}$. Let $C_t = \max\left\{\frac{\mathsf{m}_t M}{\sigma_t^2}, \frac{\mathsf{m}_t^2 M^2}{\sigma_t^4}\right\}$. Thus, $g$ is $(1,C_t)$-Sobolev. Further, we note that for any $i \in [D]$, 
\[\|(\hX_t)_i\|_{\psi_2} =  \|(\mathsf{m}_t \hX_0 + \sigma_t Z)_i\|_{\psi_2} \le \mathsf{m}_t \|(\hX_0)_i\|_{\psi_2} + \sigma_t \|(Z)_i\|_{\psi_2} \lesssim \mathsf{m}_t M +  \sigma_t \le M+1.\] Here $Z \sim \gamma_D$. Thus, with probability at least $1-\varepsilon$, $\|\hX_t\|_\infty \le \sqrt{(M+1)\log(1/\varepsilon) \log D}$.

Let $A = \sqrt{(M+1)\log(1/\varepsilon) \log D}$. We define $\tilde{g}(x) = \frac{1}{C_t}g\left(\frac{x+A}{2A}\right)$. Clearly, $\tilde{g}$ is $(1,1)$-Sobolev. From \citet[Theorem 21]{JMLR:v26:24-0054}, we can find a neural network $\hat{\tilde{g}}$, such that, $\sup_{x \in [0,1]^D} \| \tilde{g}(x) - \hat{\tilde{g}}(x)\|_\infty \le \epsilon$ and $\cW(\hat{\tilde{g}}) \lesssim \epsilon^{-D} \log(1/\epsilon)$ and $\cL(\hat{\tilde{g}}) \lesssim \log(1/\epsilon)$. We can construct $\hat{g}$ by taking $\hat{g}(x) = C_t \hat{\tilde{g}}(A(x-1/2))$. Clearly, $\cW(\hat{g}) = \cW(\hat{\tilde{g}}) \lesssim \epsilon^{-D} \log(1/\epsilon)$, $\cL(\hat{g}) = \cL(\hat{\tilde{g}}) \lesssim \log(1/\epsilon)$ and $\cB(\tilde{g}) \le \cB(\hat{g}) \lesssim \epsilon^{-1}$. Further, $\sup_{x \in [-A,A]^D} \|g(x) - \hat{g}(x)\|_\infty \le C_t \epsilon$. Thus,

\begin{align*}
   \E\| g(\hX_t) - \hat{g}(\hX_t) \|_\infty^2 
    = & \ \E\| g(\hX_t) - \hat{g}(\hX_t) \|_\infty^2 \one\{\|X_t\|_\infty \le A\} 
    + \E\| g(\hX_t) - \hat{g}(\hX_t) \|_\infty^2 \one\{\|X_t\|_\infty > A\}\\
    \lesssim & \ C_t^2 \epsilon^2 + C_t^2 \prob(\|X_t\|_\infty > A)\\
    \le & \ C_t^2 \epsilon^2,
\end{align*}
taking $\varepsilon = \epsilon^2$. Clearly, taking $\tilde{s}_t(x)= \hat{g}(x) - x/\sigma^2_t $. Clearly, $\cW(\tilde{s}_t) \lesssim \epsilon^{-D} \log(1/\epsilon)$, $\cL(\tilde{s}_t) \lesssim \log(1/\epsilon)$, and $\cB(\tilde{s}_t) \lesssim 1/\epsilon$. We define $t_{-1} = -1$ and $t_{N+1} = T+1$. Let $\delta_i = \frac{1}{2} (t_i - t_{i-1}) \wedge (t_{i+1} - t_i)$ and 
\[\xi_{a, b}(x) = \relu\left(\frac{x+a}{a-b}\right) - \relu\left(\frac{x+b}{a-b}\right) - \relu\left(\frac{x-b}{a-b}\right) + \relu\left(\frac{x-a}{a-b}\right).\]

We define, 
\[s(x,t) = \sum_{i=0}^{N-1} \tilde{s}_{t_i}(x) \times \xi_{\delta_i/2, \delta_i/4}(t-t_i),\]
i.e., by stacking the individual networks together. Clearly, the $i$-th network spikes when $t = t_i$ and 
\begingroup
\allowdisplaybreaks
\begin{align*}
    \sum_{i=0}^{N-1} h_i \E_{x \sim \hat{P}_t}\|s(x,t_i) - \nabla \log \hat{p}_{t_i}(x)\|^2_\infty = & \sum_{i=0}^{N-1} h_i \E_{x \sim \hat{P}_t}\|\tilde{s}_{t_i}(x) - \nabla \log \hat{p}_{t_i}(x)\|^2_\infty\\
    \le & \sum_{i=0}^{N-1} h_i C_{t_i}^2 \epsilon^2\\
    \le & (M \vee 1)^2 \sum_{i=0}^{N-1} \frac{h_i}{\sigma_{t_i}^4} \epsilon^2.
\end{align*}
\endgroup
We choose $\epsilon = \frac{\min_{0 \le i \le N-1} \sqrt{h_i}}{M \vee 1}$. Thus,
\begin{align*}
    \sum_{i=0}^{N-1} h_i \E_{x \sim \hat{P}_t}\|s(x,t_i) - \nabla \log \hat{p}_{t_i}(x)\|^2_\infty 
    \le &  \sum_{i=0}^{N-1} \frac{h_i^2}{\sigma_{t_i}^4} \lesssim  n^{-\frac{2 (1+p(q-p))}{d(q-p)}} \log n.
\end{align*}
Further, 
\[\epsilon = \frac{\min_{0 \le i \le N-1} \sqrt{h_i}}{{M \vee 1}} = \frac{\sqrt{\kappa \delta_0}}{M \vee 1} \asymp (M \vee 1)^{-1} n^{-\frac{ 1+p(q-p)}{d(q-p)}} n^{-\frac{1}{d}} = (M \vee 1)^{-1}  n^{-\frac{1+(p+1) (q-p)}{d(q-p)}}. \]
By construction,
\begingroup
\allowdisplaybreaks
\begin{align*}
    & \cW(s) \lesssim N \epsilon^{-D} \log(1/\epsilon) \\
    \asymp & (M \vee 1)^{D} n^{D \times \frac{1+(p+1) (q-p)}{d(q-p)}} \times n^{\frac{ 2(1+p(q-p))}{d(q-p)}} \left(\log^2 n + \log (M \vee 1) \log n\right)\\
    \asymp & (M \vee 1)^{D} n^{\frac{D+2 + (Dp + D + 2p) (q-p)}{d(q-p)}} \left(\log^2 n + \log (M \vee 1) \log n\right),
\end{align*}
\endgroup
  $\cL(s) \lesssim \log (1/\epsilon) \asymp \log n + \log (M \vee 1)$ and $\cB(s) \lesssim (\epsilon \delta_0)^{-1}\cB(s) \lesssim (M \vee 1)  n^{\frac{1+(p+3)(q-p)}{d(q-p)}}$.
\end{proof}
\section{Discretization Error}
\label{ap_discretization}
\subsection{Proof of Lemma~\ref{lemdis}}
We prove Lemma~\ref{lemdis} through the following results.
\begin{lemma}\label{lem2}
Suppose that $\mathsf{m}_{t,s} = e^{-\int_{s}^t \beta_\tau \diff  \tau} $. Then for any $0 \le s < t  $,
    \begin{align}
        \E \| \nabla \log \hp_t(\hX_t) - \nabla \log \hp_s(\hX_s)\|^2 \le & 6 \E \left\| \nabla \log \hp_s(\hX_s) - \nabla \log \hp_s(\hX_t/\mathsf{m}_{t,s})\right\|^2 \nonumber\\
        & + 4 (1-\mathsf{m}_{t,s}^{-1})^2 \E \|\nabla \log \hp_s(\hX_s)\|^2. \label{s6}
    \end{align}
\end{lemma}
\begin{proof}
    Clearly, $\hX_t|\hX_s \sim \cN(\mathsf{m}_{t,s}\hX_s, (1-\mathsf{m}_{t,s}^2)I_D)$. Thus, 
    \begingroup
    \allowdisplaybreaks
\begin{align*}
    \nabla \log \hp_t(x) = &  -\frac{\int \frac{x-\mathsf{m}_{t,s} \hx_s}{\sigma_t^2} \exp\left(-\frac{\|x- \mathsf{m}_{t,s} \hx_s\|^2}{2 \sigma_t^2}\right) \hp_s(\hx_s) \diff  \hx_s}{\int \exp\left(-\frac{\|x- \mathsf{m}_{t,s} \hx_s\|^2}{2 \sigma_t^2}\right) \hp_s(\hx_s) \diff  \hx_s}\\
    = & - \frac{\int  \nabla_{\hx_s} \exp\left(-\frac{\|x- \mathsf{m}_{t,s} \hx_s\|^2}{2 \sigma_t^2}\right) \hp_s(\hx_s) \diff  \hx_s}{ \mathsf{m}_{t,s}\int \exp\left(-\frac{\|x- \mathsf{m}_{t,s} \hx_s\|^2}{2 \sigma_t^2}\right) \hp_s(\hx_s) \diff  \hx_s}\\
    = &  \frac{\int   \exp\left(-\frac{\|x- \mathsf{m}_{t,s} \hx_s\|^2}{2 \sigma_t^2}\right) \nabla_{\hx_s}\hp_s(\hx_s) \diff  \hx_s}{ \mathsf{m}_{t,s}\int \exp\left(-\frac{\|x- \mathsf{m}_{t,s} \hx_s\|^2}{2 \sigma_t^2}\right) \hp_s(\hx_s) \diff  \hx_s}\\
    = & \frac{1}{\mathsf{m}_{t,s}} \E_{\hx_s \sim p_{s|t}(\hx_s|x)} \nabla_{\hx_s} \log \hp_s(\hx_s).
\end{align*}
\endgroup
Thus, 
\begingroup
\allowdisplaybreaks
\begin{align}
    & \E \| \nabla \log \hp_t(\hX_t) - \nabla \log \hp_s(\hX_s)\|^2 \nonumber\\
    = &  \E \left\| \nabla \log \hp_s(\hX_s) - \frac{1}{\mathsf{m}_{t,s}} \E_{\hX_s \sim p_{s|t}(\hX_s|\hX_t)} \nabla_{\hX_s} \log \hp_s(\hX_s)\right\|^2 \nonumber\\
    \le & 2 \E \left\| \nabla \log \hp_s(\hX_s) - \nabla \log \hp_s(\hX_t/\mathsf{m}_{t,s})\right\|^2 + 2 \E \left\| \nabla \log \hp_s(\hX_t/\mathsf{m}_{t,s}) - \frac{1}{\mathsf{m}_{t,s}} \E_{\hX_s \sim p_{s|t}(\hX_s|\hX_t)} \nabla_{\hX_s} \log \hp_s(\hX_s)\right\|^2\nonumber\\
    \le & 2 \E \left\| \nabla \log \hp_s(\hX_s) - \nabla \log \hp_s(\hX_t/\mathsf{m}_{t,s})\right\|^2 + 2 \E \left\| \nabla \log \hp_s(\hX_t/\mathsf{m}_{t,s}) - \frac{1}{\mathsf{m}_{t,s}}  \nabla_{\hX_s} \log \hp_s(\hX_s)\right\|^2 \label{s1}\\
    \le & 6 \E \left\| \nabla \log \hp_s(\hX_s) - \nabla \log \hp_s(\hX_t/\mathsf{m}_{t,s})\right\|^2 + 4 (1-\mathsf{m}_{t,s}^{-1})^2 \E \|\nabla \log \hp_s(\hX_s)\|^2. \nonumber
\end{align}
\endgroup
Here, inequality \eqref{s1} follows from Jensen's inequality.
\end{proof}
The following two Lemmata bound the two terms on the RHS of \eqref{s6}, respectively.
\begin{lemma}\label{lem:3}
    For any $0 \le s < t$ and $t-s \le 1$, \[\E \|\nabla \log \hp_s(\hX_s) - \nabla \log \hp_s(\hX_t/\mathsf{m}_{t,s})\|^2 \lesssim \frac{D^2}{\sigma_s^2} \left( e^{\int_s^t \beta_\tau \diff  \tau} - 1\right) \lesssim \frac{D^2}{\sigma_s^4} \int_s^t \beta_\tau \diff  \tau.\]
\end{lemma}
\begin{proof}
We note that 
\begingroup
\allowdisplaybreaks
\begin{align}
     \E \|\nabla \log \hp_s(\hX_s) - \nabla \log \hp_s(\hX_t/\mathsf{m}_{t,s})\|^2 = & \E \left\| \int_0^1 \nabla^2 \log \hp_s(\hX_s + \alpha (\hX_t/\mathsf{m}_{t,s} - \hX_s)) (\hX_t/\mathsf{m}_{t,s} - \hX_s) \diff  \alpha \right\|^2 \nonumber\\
    \le  & \int_0^1 \E \left\| \nabla^2 \log \hp_s(\hX_s + \alpha (\hX_t/\mathsf{m}_{t,s} - \hX_s)) (\hX_t/\mathsf{m}_{t,s} - \hX_s) \diff  \alpha \right\|^2 \nonumber\\
    = & \int_0^1 \E \left\| \nabla^2 \log \hp_s(\hX_s + \alpha \epsilon_{t,s}) \epsilon_{t,s} \diff  \alpha \right\|^2, \label{s5}
\end{align}
\endgroup
where $\epsilon_{t,s} = \hX_t/\mathsf{m}_{t,s} - \hX_s \sim \cN(0, (\mathsf{m}_{t,s}^{-1}-1)I_D)$. Further, $\epsilon_{t,s}$ and $\hX_s$ are independent. Thus,
\begin{align}
     \E \left\| \nabla^2 \log \hp_s(\hX_s + \alpha \epsilon_{t,s}) \epsilon_{t,s}  \right\|^2 
    = & \E \left(\left\| \nabla^2 \log \hp_s(\hX_s) \epsilon_{t,s} \right\|^2 \frac{\diff  P_{\hX_s + \alpha \epsilon_{t,s}}}{\diff  P_{\hX_s, \epsilon_{t,s}}}(\hX_s, \epsilon_{t,s}) \right) \nonumber\\
    \le & \left(\E \left(\left\| \nabla^2 \log \hp_s(\hX_s) \epsilon_{t,s} \right\|^4 \right) \E \left( \frac{\diff  P_{\hX_s + \alpha \epsilon_{t,s}}}{\diff  P_{\hX_s, \epsilon_{t,s}}}(\hX_s, \epsilon_{t,s}) \right)^2\right)^{1/2} . \label{s4}
\end{align}
Hence, 
\begin{align}
    \E \left(\left\| \nabla^2 \log \hp_s(\hX_s) \epsilon_{t,s} \right\|^4 \right) \le & \E \left(\left\| \nabla^2 \log \hp_s(\hX_s)\|^4 \|\epsilon_{t,s} \right\|^4 \right) \nonumber\\
    = & \E \left\| \nabla^2 \log \hp_s(\hX_s)\|^4 \E \|\epsilon_{t,s} \right\|^4 \nonumber\\
    \lesssim & \left(\frac{D}{\sigma_s^2}\right)^4 (\mathsf{m}_{t,s}^{-1}-1)^2 \label{s3}.
\end{align}
Thus, from Equations \eqref{s4} and \eqref{s3},
\begingroup
\allowdisplaybreaks
\begin{align}
     \left(\E \left\| \nabla^2 \log \hp_s(\hX_s + \alpha \epsilon_{t,s}) \epsilon_{t,s}  \right\|^2 \right)^2 
    \lesssim 
    &  \left(\frac{D}{\sigma_s^2}\right)^4 (\mathsf{m}_{t,s}^{-1}-1)^2 \E \left( \frac{\diff  P_{\hX_s + \alpha \epsilon_{t,s}, \epsilon_{t,s}}}{\diff  P_{\hX_s, \epsilon_{t,s}}}(\hX_s, \epsilon_{t,s}) \right)^2 \nonumber\\
     \le  &  \left(\frac{D}{\sigma_s^2}\right)^4 (\mathsf{m}_{t,s}^{-1}-1)^2 \E \left( \frac{\diff  P_{\hX_s + \alpha \epsilon_{t,s}| \epsilon_{t,s}, x_0 } (\hX_s| \epsilon_{t,s}, x_0)}{\diff  P_{\hX_s, \epsilon_{t,s}| \epsilon_{t,s}, x_0}(\hX_s| \epsilon_{t,s}, x_0)} \right)^2 \label{s2}\\
    = & \left(\frac{D}{\sigma_s^2}\right)^4 (\mathsf{m}_{t,s}^{-1}-1)^2 \E \left( \frac{\diff  P_{\hX_s + \alpha \epsilon_{t,s}| \epsilon_{t,s}, x_0 } (\hX_s| x_0)}{\diff  P_{\hX_s, \epsilon_{t,s}| \epsilon_{t,s}, x_0}(\hX_s|  x_0)} \right)^2 \nonumber\\
    = & \left(\frac{D}{\sigma_s^2}\right)^4 (\mathsf{m}_{t,s}^{-1}-1)^2 \E \exp\left( \frac{\alpha^2 \|\epsilon_{t,s}\|^2}{\sigma_s^2}\right) \nonumber\\
    = & \left(\frac{D}{\sigma_s^2}\right)^4 (\mathsf{m}_{t,s}^{-1}-1)^2 \left(1 - \frac{2 \alpha^2 (\mathsf{m}_{t,s}^{-1}-1)}{\sigma_s^2}\right)^{-D/2} \nonumber\\
    \le & \left(\frac{D}{\sigma_s^2}\right)^4 (\mathsf{m}_{t,s}^{-1}-1)^2,\nonumber
\end{align}
\endgroup
when $t-s$ is small enough. Here, Equation \eqref{s2} follows from the data processing inequality for the  $\chi^2$-divergence. Hence, from \eqref{s5},
\begin{align*}
     \E \|\nabla \log \hp_s(\hX_s) - \nabla \log \hp_s(\hX_t/\mathsf{m}_{t,s})\|^2 \lesssim & \frac{D^2}{\sigma_s^2} \left( e^{\int_s^t \beta_\tau \diff  \tau} - 1\right) \lesssim \frac{D^2}{\sigma_s^4} \int_s^t \beta_\tau \diff  \tau.
\end{align*}
\end{proof}
Using the above results, Lemma~\ref{lem4} develops a bound on the discretization error for any partition $\{t_i\}_{i=0}^N$ as follows.
\begin{lemma}\label{lem4}
      \begin{align}
    \sum_{i=0}^{N-1}\int_{t_i}^{t_{i+1}} \beta^2_{T-t} \|\nabla \log \hp_{T-t_i}(X_{T-t_i}) - \nabla \log \hp_{T-t}(X_{T-t})\|^2 \diff  t 
    \lesssim \sum_{i=0}^{N-1}\frac{(t_{i+1} - t_i)^2}{\sigma_{t_i}^4}. \label{e24}
\end{align}
  \end{lemma}
 \begin{proof}
  Taking $t_i^\prime = T-t_{N-i}$, the third term in \eqref{e7} is thus,
 \begingroup
 \allowdisplaybreaks
\begin{align}
    & \sum_{i=0}^{N-1}\int_{t_i}^{t_{i+1}} \beta^2_{T-t} \|\nabla \log \hp_{T-t_i}(X_{T-t_i}) - \nabla \log \hp_{T-t}(X_{T-t})\|^2 \diff  t \nonumber\\
    = & \sum_{i=0}^{N-1}\int_{T-t_i}^{T-t_{i+1}} \beta^2_{t} \|\nabla \log \hp_{T-t_i}(X_{T-t_i}) - \nabla \log \hp_{t}(X_{t})\|^2 \diff  t \nonumber\\
    = & \sum_{i=0}^{N-1}\int_{t_i^\prime}^{t_{i+1}^\prime} \beta^2_{t} \|\nabla \log \hp_{t_i^\prime}(X_{t_i^\prime}) - \nabla \log \hp_{t}(X_{t})\|^2 \diff  t \nonumber\\
    \lesssim & \sum_{i=0}^{N-1}\int_{t_i^\prime}^{t_{i+1}^\prime} \|\nabla \log \hp_{t_i^\prime}(X_{t_i^\prime}) - \nabla \log \hp_{t}(X_{t})\|^2 \diff  t. \label{e8}
\end{align}
\endgroup

We will bound the individual terms in \eqref{e8}. For notational simplicity, let $\mathsf{m}_{t,s} = e^{-\int_s^t \beta_\tau \diff \tau}$. From Lemma~\ref{lem2}, we note that for any $t \in [t_i^\prime, t_{i+1}^\prime]$,
\begingroup
\allowdisplaybreaks
\begin{align}
    & \E \|\nabla \log \hp_{t_i^\prime}(X_{t_i^\prime}) - \nabla \log \hp_{t}(X_{t})\|^2 \nonumber\\
    \le & 6 \E \left\| \nabla \log \hp_{t_i^\prime}(\hX_{t_i^\prime}) - \nabla \log \hp_{t_i^\prime}(\hX_t/\mathsf{m}_{t,t_i^\prime})\right\|^2 + 4 (1-\mathsf{m}_{t,{t_i^\prime}}^{-1})^2 \E \|\nabla \log \hp_{t_i^\prime}(\hX_{t_i^\prime})\|^2 \nonumber\\
    \lesssim & \frac{D^2}{\sigma_{t_{i}^\prime}^4} \int_{{t_i^\prime}}^t \beta_\tau \diff  \tau + \frac{D(1-\mathsf{m}_{t,{t_i^\prime}}^{-1})^2}{\sigma_{t_i^\prime}^2} \label{e9}\\
    \le & \frac{D^2 \overline{\beta}}{\sigma_{t_{i}^\prime}^4} (t - {t_i^\prime}) + \frac{D (t - {t_i^\prime})}{\sigma_{t_i^\prime}^2} \label{e26}\\
    \lesssim & \frac{D^2 (t - {t_i^\prime})}{\sigma_{t_i^\prime}^4}. \label{e10}
\end{align}
\endgroup
Here, \eqref{e9} follows from Lemmata~\ref{lem:3} and \ref{lem:5}. Equation \eqref{e26} follows from  Lemma~\ref{lem6} with $c=1$. Plugging in \eqref{e10} in \eqref{e8}, we observe that,
\begin{align}
     \sum_{i=0}^{N-1}\int_{t_i}^{t_{i+1}} \beta^2_{T-t} \|\nabla \log \hp_{T-t_i}(X_{T-t_i}) - \nabla \log \hp_{T-t}(X_{T-t})\|^2 \diff  t 
    \lesssim & \sum_{i=0}^{N-1}\frac{1}{\sigma_{t_i^\prime}^4}\int_{t_i^\prime}^{t_{i+1}^\prime} (t - {t_i^\prime}) \diff  t  \nonumber\\
    \le & \sum_{i=0}^{N-1}\frac{(t_{i+1}^\prime - t_i^\prime)^2}{\sigma_{t_i^\prime}^4}. \label{e11}
\end{align}
The proof follows from observing that $t_i = T-t_{N-i}$.
\end{proof}
The only remaining part for a meaningful bound on the discretization error is to bound the RHS of \eqref{e24} that increases logarithmically with $1/\delta_0$ and linearly with $T$.  This is done by exploiting the choice of partition described in Section~\ref{sec_assumptions}.
\begin{lemma}\label{lem5}
    For the proposed partition, \(\sum_{i=0}^{N-1} \frac{(t_{i+1} - t_i)^2}{\sigma^4_{t_i}}  \le \kappa \left(\log(1/\delta_0) + T \right)\) and $N \lesssim \frac{1}{\kappa} \left(\log(1/\delta_0) + T \right)$, when $\kappa \le 1$.
\end{lemma}
\begin{proof}
We note that,
\begingroup
\allowdisplaybreaks
\begin{align*}
    \sigma^2_{t} = 1 -e^{- \int_0^t \beta_\tau \diff  \tau} \ge 1-e^{-\underline{\beta} t} \ge (t \wedge 1) \left(1-e^{-\underline{\beta}}\right).
\end{align*}
Let $i^\star = \max\{i \in \fN : t^\prime_i \le 1\} $. Then, for any $i \le i^\star$, $t^\prime_{i+1} = h^\prime_i + t^\prime_i = (\kappa+1) t^\prime_i = (1+\kappa)^{i+1} t^\prime_0 = (1+\kappa)^{i+1} \delta_0$. Thus,
\[1 \ge t^\prime_{i^\star} = (1+\kappa)^{i^\star} \delta_0 \implies i^\star \le \frac{\log(1/\delta_0)}{\log(1+\kappa)}.\]
Thus,
\begin{align}
    \sum_{i=0}^{N-1} \frac{(t_{i+1}^\prime - t_i^\prime)^2}{\sigma^4_{t_i^\prime}} 
    = & \sum_{i: t_i^\prime \le 1}\frac{(t_{i+1}^\prime - t_i)^2}{\sigma^4_{t_i^\prime}} + \sum_{i: t_i^\prime > 1}\frac{(t_{i+1}^\prime - t_i)^2}{\sigma^4_{t_i^\prime}} \nonumber\\
    \lesssim & \sum_{i: t_i^\prime \le 1}\frac{(t_{i+1}^\prime - t_i)^2}{(t_i^\prime\wedge 1)^2} + \sum_{i: t_i^\prime > 1}\frac{(t_{i+1}^\prime - t_i)^2}{(t_i^\prime\wedge 1)^2} \nonumber\\
    = & \sum_{i: t_i^\prime \le 1} \frac{(t_{i+1}^\prime - t_i)^2}{(t_i^\prime)^2} + \sum_{i: t_i^\prime > 1} (t_{i+1}^\prime - t_i)^2 \nonumber\\
    \le & \kappa^2 \frac{\log(1/\delta_0)}{\log(1+\kappa)} + \kappa^2 \times \frac{T}{\kappa} \nonumber\\
    \le & \kappa \left(\log(1/\delta_0) + T \right). \label{e25}
\end{align}
\endgroup
The bound on the number of terms in the partition, i.e., $N$ follows by counting the number of terms in the above breakdown, which totals to $\frac{\log(1/\delta_0)}{\log(1+\kappa)} + \frac{T}{\kappa} \le \log 2 \cdot  \frac{\log(1/\delta)}{\kappa} + \frac{T}{\kappa}$, for $\kappa \le 1$.
\end{proof}
\subsection{Supporting Lemmata}
\begin{lemma}\label{lem_h5}
    $\nabla^2 \log \hp_t(x) = \frac{\mathsf{m}_t ^2 }{\sigma_t^4} \operatorname{Var}\left(X_0 \big| X_t =x\right) - \frac{I_D}{\sigma_t^2}$.
\end{lemma}
\begin{proof}
     We note that,
    \begingroup
    \allowdisplaybreaks
    \begin{align*}
    & \nabla^2 \log \hp_t(x) \\
    = & -\frac{1}{{\bigg(\int \exp\left(-\frac{\|x- \mathsf{m}_t x_0\|^2}{2 \sigma_t^2}\right) \diff  \mu(x_0)\bigg)^2}} \Bigg(\int \exp\left(-\frac{\|x- \mathsf{m}_t x_0\|^2}{2 \sigma_t^2}\right) \diff  \mu(x_0) \\
    & \times \bigg( \frac{1}{\sigma_t^2}\int \exp\left(-\frac{\|x- \mathsf{m}_t x_0\|^2}{2 \sigma^2}\right) \diff  \mu(x_0) \\
    & + \int \left(\frac{x- \mathsf{m}_t x_0}{\sigma^2}\right) \left(\frac{x- \mathsf{m}_t x_0}{\sigma^2}\right)^\top \exp\left(-\frac{\|x- \mathsf{m}_t x_0\|^2}{2 \sigma^2}\right) \diff  \mu(x_0)\bigg)\Bigg)\\
    & +  \frac{\left(\int \frac{x-\mathsf{m}_t x_0}{\sigma^2}\exp\left(-\frac{\|x- \mathsf{m}_t x_0\|^2}{2 \sigma_t^2}\right) \diff  \mu(x_0)\right) \left(\int \frac{x- \mathsf{m}_t x_0}{\sigma^2}\exp\left(-\frac{\|x-\mathsf{m}_t x_0\|^2}{2 \sigma_t^2}\right) \diff  \mu(x_0)\right)^\top }{\left(\int \exp\left(-\frac{\|x-\mathsf{m}_t x_0\|^2}{2 \sigma_t^2}\right) \diff  \mu(x_0)\right)^2}\\
    = & \frac{1}{\sigma_t^4}\E \left((X_t- \mathsf{m}_t x_0)(X_t- \mathsf{m}_t x_0)^\top \big| X_t =x\right) - \frac{I_D}{\sigma^2}\\
    = & \frac{1}{\sigma_t^4} \operatorname{Var}\left((X_t- \mathsf{m}_t x_0)\big| X_t =x\right) - \frac{I_D}{\sigma_t^2}\\
    = &  \frac{\mathsf{m}_t ^2 }{\sigma_t^4} \operatorname{Var}\left(X_0 \big| X_t =x\right) - \frac{I_D}{\sigma_t^2}.
\end{align*}
\endgroup
\end{proof}
\begin{lemma}\label{lem6}
    Suppose that $t \ge s$, then \(\left(1-\mathsf{m}_{t,s}^{-1}\right)^2 \le \frac{(1-e^{\underline{\beta} c})^2}{c} \underline{\beta}^2 (t-s), \) for $t-s \le c$.
\end{lemma}
\begin{proof}
    We note that, 
    \begin{align*}
        1-\mathsf{m}_{t,s}^{-1} = 1 - \exp\left( \int_s^t \beta_\tau \diff \tau\right) \le 1 - \exp(\underline{\beta} (t-s)) \le \frac{(1-e^{\underline{\beta} c})^2}{c} (t-s).
    \end{align*}
\end{proof}
\begin{lemma}
For any $t>0$,    $\left\|\|\nabla^2 \log \hp_t(X_t)\|_{F}\right\|_{\psi_1} \lesssim \frac{D}{\sigma_t^2}$. Here, $\|\cdot\|_F$ denotes the Frobenius norm.
\end{lemma}
\begin{proof}
   We observe that,
\begingroup
\allowdisplaybreaks
\begin{align*}
    \E \left\|\frac{\mathsf{m}_t ^2 }{\sigma_t^4} \operatorname{Var}\left(X_0 \big| X_t =x_t\right) \right\|_F^p = & \E \left\|\frac{1}{\sigma_t^4}\E \left((X_t- \mathsf{m}_t x_0)(X_t- \mathsf{m}_t x_0)^\top \big| X_t\right)\right\|_F^p\\
    \le & \frac{1}{\sigma^{2p}} \E_{\epsilon \sim \cN(0,I_D)} \|\epsilon \epsilon^\top\|_F^p\\
    = & \frac{1}{\sigma^{2p}} \E_{\epsilon \sim \cN(0,I_D)} \|\epsilon \|^{2p}\\
    = & \frac{2^p \Gamma(p + D/2)}{\sigma_t^{2p}\Gamma(D/2)}.
\end{align*}
\endgroup
Thus,
\begin{align*}
    \frac{1}{p} \left(\E \left\|\frac{\mathsf{m}_t ^2 }{\sigma_t^4} \operatorname{Var}\left(X_0 \big| X_t =x_t\right) \right\|_F^p\right)^{1/p} = \frac{1}{p \sigma_t^2} \left(\prod_{m=1}^p (D + 2 m - 2)\right)^{1/p} \le \frac{D+ 2p - 2}{p \sigma_t^2} \lesssim \frac{D}{\sigma_t^2}, 
\end{align*}
for all $p \in \fN$. Thus, $\left\|\E \left\|\frac{\mathsf{m}_t ^2 }{\sigma_t^4} \operatorname{Var}\left(X_0 \big| X_t =x_t\right) \right\|_F\right\|_{\psi_1} \lesssim D/\sigma_t^2$. The result now follows from triangle inequality. 
\end{proof}

\begin{lemma}\label{lem:5}
For any $t > 0$,    \(\|\nabla \log \hp_t(x_t)\|_{\psi_2} \lesssim \frac{D^2}{\sigma_t}. \)
\end{lemma}
\begin{proof}
    We note that,
    \begingroup
    \allowdisplaybreaks
    \begin{align*}
         \E \|\nabla \log \hp_t(X_t)\|^p 
        \le & \frac{1}{\sigma_t^{2p}} \E \|X_t - \mathsf{m}_t X_0\|^p\\
         = & \frac{1}{\sigma_t^p} \E_{\epsilon \sim \cN(0, I_D)} \E \|\epsilon\|^p\\
         = & \frac{1}{\sigma_t^p} \frac{2^{p/2} \Gamma\left(\frac{D+p}{2}\right)}{\Gamma(D/2)}\\
         = & \frac{2^{p/2}}{\sigma_t^p} \prod_{i=1}^{\lfloor p/2 \rfloor} \left(\frac{D}{2} + \frac{p}{2} - i\right) \cdot \frac{\Gamma(D/2+1/2) \one\{p \text{ is odd}\} + \Gamma(D/2) \one\{p \text{ is even}\}}{\Gamma(D)}\\
         \le & \frac{2^{p/2}}{\sigma_t^p} \left(\frac{D}{2} + \frac{p}{2} - 1\right)^{\lfloor p/2 \rfloor}\frac{\Gamma(D/2+1/2)}{\Gamma(D/2)}.
    \end{align*}
    \endgroup
    Thus,
\begin{align*}
   \frac{1}{\sqrt{p}} \left( \E \|\nabla \log \hp_t(X_t)\|^p\right)^{1/p} \le \frac{2^{1/2}}{\sqrt{p }\sigma_t} \left(\frac{D}{2} + \frac{p}{2} - 1\right)^{\frac{\lfloor p/2 \rfloor}{p}} \left(\frac{\Gamma(D/2+1/2)}{\Gamma(D/2)}\right)^{1/p} \lesssim \frac{D^2}{\sigma_t}, 
\end{align*}
for all $ p \in \fN.$
\end{proof}

\section{Supporting Results from the Literature}
\begin{lemma}[Lemma 4.2.8 of \citet{Vershynin_2026}]
     \label{cov_pack}For any metric space, $(S,\varrho)$ and $\epsilon>0$,
     $M(2\epsilon; S, \varrho) \le\cN(\epsilon; S, \varrho) \le M(\epsilon; S, \varrho)$.
 \end{lemma}
\begin{lemma}[Lemma 21 of \citet{JMLR:v21:20-002}]\label{lem_nakada} 
     Let $\cF = \cR\cN(W, L, B)$ be a space of ReLU networks with the number of weights, the number of layers, and the maximum absolute value of weights bounded by $W$, $L$, and $B$, respectively. Then,
\[
\log\cN(\epsilon; \cF, \ell_\infty) \leq W \log \left( 2LB^L(W + 1)^L \frac{1}{\epsilon} \right).
\]
 \end{lemma}
\bibliographystyle{apalike}

\end{document}